%% file: main.tex
\documentclass{article}

\usepackage{graphicx}
\usepackage[style=numeric,sorting=none]{biblatex}%
\usepackage{multirow}%
\usepackage{authblk}%
\usepackage{array}%
\usepackage{amsmath,amssymb,amsfonts}%
\usepackage{amsthm}%
\usepackage{empheq}%
\usepackage{hyperref}%
\usepackage{float}%
\usepackage{geometry}%
\usepackage{mathrsfs}%
\usepackage{mathtools}%
\usepackage{makecell}%
\usepackage[title]{appendix}%
\usepackage{xcolor}%
\usepackage{soul}%
\usepackage{textcomp}%
\usepackage{tabularx}%
\usepackage{bbm}%
\usepackage{thmtools}%
\usepackage{cases}%
\usepackage[acronym]{glossaries}%
\usepackage{thm-restate}%
\usepackage{manyfoot}%
\usepackage{multirow}%
\usepackage{threeparttable}%
\usepackage{lipsum}%
\usepackage{tikz-cd}%
\usepackage{lineno}%
\usepackage{stmaryrd}%
\usepackage{setspace}%
\usepackage{siunitx}%
\usepackage{nomencl}%
\usepackage{etoolbox}%
\usepackage{booktabs}%
\usepackage[ruled,linesnumbered]{algorithm2e}%
\usepackage{algorithmicx}%
\usepackage{listings}%
\usepackage{caption}%
\usepackage{subcaption}%
\usepackage{rotating}%

\AtEveryBibitem{%
  \clearfield{urlyear}%
  \clearfield{urlmonth}%
  \clearfield{urlday}%
  \clearfield{urldate}%
  \clearfield{url}%
  \clearfield{issn}%
  \clearfield{month}%
  \clearfield{note}%
}

\graphicspath{ {figures/} }
\hypersetup{
    colorlinks = true,
    allcolors = blue
}

\newtheorem{theorem}{Theorem}%
\newtheorem{lemma}{Lemma}
\newtheorem{corollary}{Corollary}
\newtheorem*{definition*}{Definition}

\makeglossaries

\newacronym{aco}{ACO}{Ant Colony Optimisation}
\newacronym{api}{API}{Application Programming Interface}
\newacronym{aps}{APS}{Advanced Planning \& Scheduling}
\newacronym{b2b}{B2B}{Business-to-Business}
\newacronym{b2c}{B2C}{Business-to-Customer}
\newacronym{bb}{B\&B}{Branch \& Bound}
\newacronym{bds}{BDS}{Bayesian Dynamic Scheduling}
\newacronym{bn}{BN}{Bayesian Network}
\newacronym{cdf}{CDF}{Cumulative Distribution Function}
\newacronym{cp}{CP}{Constraint Programming}
\newacronym{cpd}{CPD}{Conditional Probability Distribution}
\newacronym{cpt}{CPT}{Conditional Probability Table}
\newacronym{dag}{DAG}{Directed Acyclic Graph}
\newacronym{des}{DES}{Discrete Event Simulation}
\newacronym{erp}{ERP}{Enterprise Resource Planning}
\newacronym{fis}{FIS}{Finite Intermediate Storage}
\newacronym{ga}{GA}{Genetic Algorithm}
\newacronym{gep}{GEP}{Gene Expression Programming}
\newacronym{iot}{IoT}{Internet of Things}
\newacronym{lp}{LP}{Linear Programming}
\newacronym{lstm}{LSTM}{Long Short-Term Memory}
\newacronym{mc}{MC}{Monte Carlo}
\newacronym{mes}{MES}{Manufacturing Execution System}
\newacronym{mdp}{MDP}{Markov Decision Process}
\newacronym{mip}{MIP}{Mixed Integer Programming}
\newacronym{milp}{MILP}{Mixed-Integer Linear Programming}
\newacronym{minlp}{MINLP}{Mixed-Integer Nonlinear Programming}
\newacronym{mle}{MLE}{Maximum Likelihood Estimation}
\newacronym{mpc}{MPC}{Model Predictive Control}
\newacronym{nis}{NIS}{No Intermediate Storage}
\newacronym{nn}{NN}{Neural Network}
\newacronym{nphard}{NP-hard}{Nondeterministic Polynomial-time hard}
\newacronym{or}{OR}{Operations Research}
\newacronym{pdf}{PDF}{Probability Density Function}
\newacronym{pid}{P\&ID}{Piping and Instrumentation Diagram}
\newacronym{pmf}{PMF}{Probability Mass Function}
\newacronym{pse}{PSE}{Process Systems Engineering}
\newacronym{pso}{PSO}{Particle Swarm Optimisation}
\newacronym{rd}{R\&D}{Research and Development}
\newacronym{rl}{RL}{Reinforcement Learning}
\newacronym{rtn}{RTN}{Resource-Task Network}
\newacronym{sa}{SA}{Simulated Annealing}
\newacronym{stn}{STN}{State-Task Network}
\newacronym{tsp}{TSP}{Travelling Salesman Problem}
\newacronym{uis}{UIS}{Unlimited Intermediate Storage}
\newacronym{zw}{ZW}{Zero-Wait}

\newcommand{\independent}{\perp\!\!\!\perp}

\title{Bayesian dynamic scheduling of multipurpose batch processes under incomplete look-ahead information}
\author[1]{Taicheng Zheng}
\author[1]{Dan Li}
\author[1,*]{Jie Li}
\affil[]{Centre for Process Integration, the Department of Chemical Engineering, The University of Manchester, Oxford Road, Manchester, M13 9SS, United Kingdom}

\begin{document}

\maketitle

\begin{abstract}
    Multipurpose batch processes become increasingly popular in manufacturing industries because they can adapt to low-volume, high-value products and shifting demands. These processes often operate in a dynamic environment, which faces disturbances such as processing delays and demand fluctuations. To minimise long-term cost and system nervousness (that is, disruptive changes to schedules), schedulers must design rescheduling strategies to address such disturbances effectively. Existing methods often assume complete look-ahead information over the scheduling horizon. This assumption contrasts with realistic situations where schedulers can only access incomplete look-ahead information. Insisting on existing methods may lead to suboptimal long-term cost and high system nervousness. To address this gap, we propose a Bayesian dynamic scheduling method. This method relies on learning a Bayesian Network from the probability distribution of disturbances. Specifically, the Bayesian Network represents how likely each operation will be impacted by disturbances. During the online execution, when new disturbances become observed, this method updates the posterior distribution and therefore guides the rescheduling strategy. We compare our method with the existing periodically-completely rescheduling strategy (which generates new schedules from scratch at fixed intervals) on four benchmark problems. Computational results show that our method achieves statistically better long-term costs and system nervousness. In the theoretical aspect, we prove that if disturbances are mutually independent, the impact-quantifying variables inherently satisfy the independence assumptions required by Bayesian Networks. As an implication, practitioners can extend the method to other scheduling problems (such as job shop scheduling and continuous processes), provided that they define the problem-specific dependencies between operations.
\end{abstract}


\input{sections/section1}
\input{sections/section2}
\input{sections/section3}
\input{sections/section4}
\input{sections/section5}
\input{sections/section6}
\input{sections/section7}
\input{sections/section8}
\input{sections/section9}
\input{sections/acknowledgments}

\printglossary[type=\acronymtype]

\printbibliography

\appendix
\input{sections/appendixA}
\input{sections/appendixB}
\input{sections/appendixC}
\input{sections/appendixD}
\input{sections/appendixE}
\input{sections/appendixF}

\end{document}

%% file: sections/section1.tex
\section{Introduction} \label{sec:introduction}

Modern chemical production becomes increasingly driven by niche markets that demand low-volume, high-value products \cite{wong_critical_2019}. Examples include pharmaceuticals \cite{awad_constraint_2022}, agrochemicals, and advanced materials \cite{prata_integrated_2008}. Such markets often experience fast-changing customer preferences, regulatory updates, and demand fluctuations \cite{ouelhadj_survey_2009, harjunkoski_scope_2014}. To remain competitive, manufacturers must adapt rapidly to these changes and deliver diverse, customisable, and specialised products. Multipurpose batch processes allow enterprises to address these challenges by setting up flexible, reconfigurable production workflows within a single facility \cite{li_integrated_2016, li_novel_2022, zheng_rescheduling_2025}. Unlike continuous processes, which are designed for large-scale, standardised outputs, multipurpose batch processes share equipment across multiple products and production stages. This flexibility enables manufacturers to switch between recipes and adjust batch sizes without incurring costly downtime. As a result, in chemical industries where flexibility and specialisation are desired, multipurpose batch processes are attracting increasing attention \cite{lee_mixed-integer_2017, woolway_application_2019}.

Multipurpose batch processes often operate in a dynamic environment where real-time disturbances (such as equipment malfunctions, processing delays, and raw material shortages) exist. Such disturbances pose significant challenges to production scheduling. On the one hand, if schedulers ignore disturbances and stick to the original schedule, it may lead to suboptimal long-term costs \cite{gupta_rescheduling_2016}. More seriously, the original schedule may become infeasible due to violations of operational constraints or resource availability. On the other hand, over-responding to disturbances (such as constantly revising production sequences, or delaying operations) often leads to \textit{system nervousness}, which refers to frequent disruptive changes to schedules \cite{hwangbo_production_2024}. Such changes may disrupt pre-arranged workflow, cause chaos and confusion, and damage business credibility with third party suppliers \cite{morita_proactive_2023}. For example, a sudden cancellation of chemical batches may leave a supplier with unused inventory, which damages business trust and collaboration. As a result, \textit{dynamic scheduling}, the research area that aims to address scheduling in the presence of real-time disturbances, is attracting growing research interests \cite{oyebolu_dynamic_2019, di_pretoro_demand_2022, avadiappan_state_2021}.

Dynamic scheduling methods can be broadly grouped into two categories: \textit{proactive scheduling} and \textit{reactive scheduling}. Proactive scheduling anticipates potential disturbances to generate a robust baseline schedule before execution begins \cite{wittmannhohlbein_proactive_2013, lim_proactive_2009, morita_proactive_2023}. In contrast, reactive scheduling dynamically revises or regenerates schedules in real-time as disruptions occur during the execution \cite{li_reactive_2008, rahal_proactive_2020, zhang_data-driven_2021}. For a comprehensive review on these methods, we refer readers to review articles such as \cite{vieira_rescheduling_2003, ouelhadj_survey_2009, li_process_2008, sabuncuoglu_hedging_2009}. A critical concept that implicitly distinguishes proactive and reactive scheduling is the \textit{certainty horizon}, that is, the horizon within which the outcome of disturbances becomes confirmed. Proactive scheduling studies, such as \cite{bonfill_proactive_2008, lim_proactive_2009, han_proactive_2010, cui_proactive_2018}, often assume a certainty horizon of zero length. This assumption implies that schedulers possess no information about future disturbance outcomes. Conversely, reactive scheduling studies, such as \cite{janak_production_2006, kopanos_reactive_2014, nie_extended_2014, rahmani_stable_2016, kotidis_digiglyc_2021}, often assume a certainty horizon that equals the scheduling horizon, which implies a complete look-ahead across the scheduling horizon.

However, in practice, multipurpose batch systems rarely operate with certainty horizons that equal either a length of zero or the scheduling horizon. Instead, the certainty horizon often presents a length in between. Consider a multipurpose batch process within a chemical plant, where schedulers are required to generate schedules for the next rolling week. In realistic situations, the certainty horizon may range from a few hours to a few days due to the real-time nature of disturbances. For example, a supplier may notify the plant of a delay in raw materials only 24 hours before its scheduled use in a high-priority batch. Also, \acrfull{iot}-enabled predictive maintenance systems may flag a high-probability machine failure only multiple hours prior to its next operation. In fact, such scenarios frequently occur in day-to-day production operations. These facts reflect the reality of \textit{incomplete look-ahead information}, where schedulers have only partial visibility into future disturbances over the scheduling horizon \cite{dwibedy_semi-online_2022, lee_semi-online_2013}.

This incomplete look-ahead information poses additional challenges to dynamic scheduling. When disturbances are partially observed over the scheduling horizon, the performance of a dynamic scheduling algorithm will inevitably depend on its ability to effectively use this myopic, yet nonzero look-ahead information. However, to the best of our knowledge, this problem has received very limited attention from the area of proactive scheduling (which assumes zero look-ahead), or reactive scheduling (which assumes complete look-ahead), or proactive-reactive scheduling (which alternates between the other two). The only discussions to our awareness, within the \acrfull{pse} community, appear in \cite{wittmann-hohlbein_proactive_2013}, \cite{kopanos_reactive_2014} and \cite{gupta_design_2019}.

Wittmann-Hohlbein and Pistikopoulos \cite{wittmann-hohlbein_proactive_2013} formulated the proactive scheduling problem of batch processes as a multi-parametric \acrshort*{milp} model. In their formulation, uncertain parameters are divided into two groups, namely those revealed at the decision stage and those remaining uncertain throughout the scheduling horizon. However, their work does not involve the closed-loop pattern of dynamic scheduling. Kopanos et al.\ \cite{kopanos_reactive_2014} discussed the difference between the prediction horizon (comparable to the certainty horizon in our study) and the control horizon (the same as the scheduling horizon). However, their work assumes that the prediction horizon exceeds the control horizon in length. This assumption implies complete look-ahead information, which falls beyond our context. Gupta et al.\ \cite{gupta_design_2019} explored strategies for tuning the scheduling horizon length and the rescheduling frequency. Their strategies were evaluated on various lengths of certainty horizon for uncertain demands. Although their work involves the incomplete look-ahead information, it was limited to a purely periodic rescheduling strategy. Such a strategy neither leverages the probabilistic property of system disturbances, nor reuses information from previous schedules, which may lead to impractical computational time and high system nervousness for industrial-sized problems.

In addition to the above limitations, to the best of our knowledge, few works have systematically addressed the problem of dynamic scheduling under incomplete look-ahead information. Moreover, existing works rarely point out that, despite the superficial difference in the length of certainty horizon, from a probabilistic decision-making perspective, the incomplete look-ahead implies a fundamentally different challenge than the zero or complete look-ahead. That is, how to use observed disturbances to update the posterior distribution based on physical knowledge. To explain, let us first consider the proactive scheduling, where zero look-ahead about disturbances is assumed. In this case, rescheduling decisions often rely on \textit{unconditional probability queries} \cite{friedman_probabilistic_2009}. For example, designing a proactive schedule may implicitly involve answering queries such as:
\begin{itemize}
    \item ``What is the probability that, Task X will not complete on time due to a machine breakdown during its scheduled execution?''
    \item ``What is the probability that, Order P will be cancelled before its delivery date?''
\end{itemize}
In practice, such unconditional probability queries can be addressed using statistical models (such as Weibull analysis, autoregressive models) or machine learning models (such as random forests, LSTM) trained directly on historical data. On the other hand, in the reactive scheduling, where complete look-ahead is assumed, the rescheduling decision is often based on \textit{deterministic logic queries}. For example:
\begin{itemize}
    \item ``If Task X fails, can Task Y initiate as scheduled without delaying other tasks?''
    \item ``If an unexpected order requiring 10 units of product M is due at time $t$, can the current schedule fulfil this demand without inserting new tasks?''
\end{itemize}
In principle, these queries can be addressed through a single episode of \acrfull{des} (with all disturbance variables fixed to their observed values) \cite{burns_discrete-event_2022}, or through logic-based reasoning methods such as propositional calculus \cite{rawlings_incorporating_2019, soares_real-time_2008}. However, in the context of incomplete look-ahead information, the underlying problem shifts to answering \textit{conditional probability queries}, such as:
\begin{itemize}
    \item ``Given a confirmed two-hour delay in Task X, what is the probability that Task Y cannot initiate as scheduled?''
    \item ``Given that an unexpected order requiring 10 units of product M is due at time $t$, what is the probability that the current schedule can still fulfil this demand under a certain disturbance level?''
\end{itemize}
These conditional probability queries are more difficult to address because essentially, they combine the previous two types of queries. On the one hand, obtaining the conditional distribution requires representing the joint prior of disturbances. On the other hand, conditioning on the event (that is, updating the posterior) requires modelling both the observed disturbances and the physical knowledge (such as the production workflow, the relationships between tasks, and the characteristics of disturbances). These challenges suggest the need for a novel dynamic scheduling framework that can systematically incorporate probabilistic reasoning with logical constraints to (1) model observed disturbances, (2) formalise the physical knowledge, and (3) update the posterior using both the observed disturbances and the physical knowledge. To the best of our knowledge, this area remains unexplored within the \acrshort*{pse} community.

In this work, we propose a Bayesian dynamic scheduling framework to systematically address these challenges. We first formalise the dynamic scheduling problem under incomplete look-ahead information using a dynamic system formalism, where the planned schedules are modelled as action sequences and the look-ahead information is modelled as realisations of disturbance variables. For multipurpose batch processes, we then specialise this formalism into a transition function that incorporates various commonly encountered disturbances, including machine breakdowns, processing time variations, yield losses, casual orders, and urgent demands. To generate schedules during online execution, we further develop an \acrshort*{milp} formulation that is specifically tailored to the dynamic system formulation. This \acrshort*{milp} formulation incorporates the current system states through a collection of initial state constraints, and involves hierarchical objective functions to avoid multi-solution issues and improve schedule robustness. The full transition functions and the full \acrshort*{milp} formulation are presented in the appendices, while the main body summarises only what is needed to follow the framework. Before entering the detailed methodology, we provide a high-level framework overview that consists of the motivation and key ideas, a closed-loop workflow with a one-page flowchart, and a small illustrative example, so that readers can grasp the core idea without committing to the heavy mathematics. We then develop the detailed methodology. The framework generally follows a closed-loop scheduling pattern, while the main difference appears in updating a \acrfull{bn} to represent the potential impact of disturbances on each batch (that is, an operation) in the schedule. During the online execution, when new disturbances are observed, the framework performs inference algorithms in real time to update the posterior impacts. If the posterior impact distribution is probabilistically unacceptable, a warm-start rescheduling is triggered. One of the main novelties of the framework is the introduction of impact variables, which are artificially constructed random variables that encode the ``local impact'' of disturbances. These variables allow us to conveniently construct \acrshort*{bn}s and perform inference algorithms to reason about posterior impacts. In the theoretical aspect, we prove that, if the inherent disturbances in the system are mutually independent, then the impact variables automatically satisfy the conditional independence semantics required by \acrshort*{bn}s. This theoretical fact validates our framework, and demonstrates the applicability of our framework to other scheduling problems, such as job shop scheduling and continuous processes. To test the effectiveness of our framework, we perform extensive numerical experiments on four benchmark problems. The computational results show that, under the incomplete look-ahead information, increasing the rescheduling frequency does not necessarily reduce long-term cost. Furthermore, in some disturbance scenarios, our algorithm achieves a lower long-term cost than the periodically-completely strategy across all rescheduling frequencies.

We organise the remainder of this paper as follows. In Section \ref{sec:problem_description}, we describe the dynamic scheduling problem of a multipurpose batch process under incomplete look-ahead information. In Section \ref{sec:problem_formulation}, we summarise the dynamic system formulation and the \acrshort*{milp} formulation, and refer readers to appendices for full details. In Section \ref{sec:framework_overview}, we present a high-level overview of our Bayesian dynamic scheduling framework. This overview consists of the motivation and key ideas, a closed-loop workflow, and a small illustrative example, so that readers can grasp the core idea before the technical details. In Section \ref{sec:methodology}, we present the detailed methodology. In Section \ref{sec:theoretical_foundation}, we present the theoretical foundation of our framework. In Sections \ref{sec:experimental_design}, \ref{sec:results}, and \ref{sec:conclusions}, we present the experimental design, computational results, and conclusions, respectively.

%% file: sections/section2.tex
\section{Problem description} \label{sec:problem_description}

In this section, we describe the dynamic scheduling problem of multipurpose batch processes in detail. We partition the problem description into two parts: the static part and the dynamic part. The static part includes the structure, mechanism, and scheduling decisions of typical multipurpose batch processes. The dynamic part includes how the production system evolves through time, how disturbance information is dynamically revealed to the scheduler, and how to measure the long-term performance of a dynamic scheduling algorithm.

Before proceeding, we briefly set up the notational conventions and the mathematical background used throughout the paper. For notation, we adopt the following typographical conventions: lowercase italic symbols (such as $x$) denote scalars; lowercase boldface italic symbols (such as $\boldsymbol{x}$) denote column vectors; uppercase italic symbols (such as $X$) denote events or random variables; uppercase boldface italic symbols (such as $\boldsymbol{X}$) denote random column vectors; blackboard bold symbols (such as $\mathbb{X}$) denote sets; uppercase calligraphic symbols (such as $\mathcal{X}$) denote structured collections such as tuples; and small cap symbols (such as $\textsc{SomeFunc}$) denote functions or algorithms. For the mathematical background, we will use probability, independence, \acrshort*{dag}s, and \acrshort*{bn}s in later sections. A formal review of these topics is provided in Section S1 of the Supplementary Information.

\subsection{Static aspect} \label{subsec:static_aspect}

\begin{figure}[h]
\includegraphics[width=0.65\textwidth]{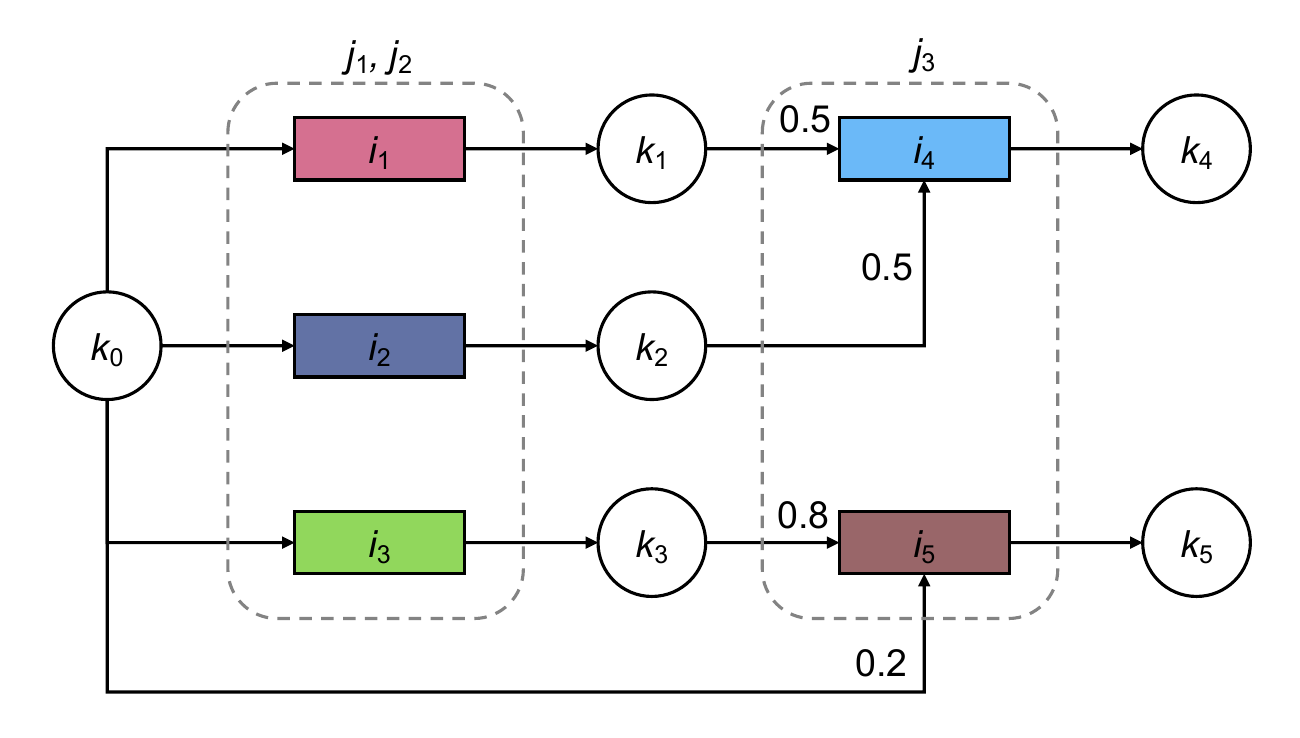}
\centering
\caption{
    \acrshort*{stn} representation of a sample multipurpose batch process.
}
\label{fig:stn_sample}
\end{figure}

Generally speaking, multipurpose batch processes are a class of production process, where tasks are executed in variable-sized batches by multipurpose machines. We use the well-established \acrfull{stn} to describe multipurpose batch processes. Figure \ref{fig:stn_sample} presents a sample \acrshort*{stn}. In the figure, circle nodes represent \textit{materials} (also called \textit{states}). Circle nodes with no parents, with no children, and all others represent raw materials, products, and intermediate materials, respectively. We use indices and sets $k \in \mathbb{K}$ to denote materials. Also, we denote the collection of raw materials, intermediate materials, and products by $\mathbb{K}^{\mathrm{R}}$, $\mathbb{K}^{\mathrm{I}}$, and $\mathbb{K}^{\mathrm{P}}$, respectively. Each material has a dedicated storage tank (although not explicitly denoted in the figure) with a finite capacity $s_{k}^{\max}$. Square nodes represent \textit{tasks}, that is, processing steps that transform materials. We use indices and sets $i \in \mathbb{I}$ to denote tasks. Arcs that connect materials and tasks represent processing recipes. Decimal numbers next to arcs represent the conversion coefficient between tasks and materials. For example, in Figure \ref{fig:stn_sample}, when initiating a batch of task $i_{5}$, the input materials should comprise $80 \%$ of material $k_{3}$ and $20 \%$ of material $k_{0}$. We use $\rho_{i,k}$ to denote the conversion coefficient between $i$ and $k$. If the coefficient $\rho_{i,k} > 0$, task $i$ produces material $k$. Otherwise $\rho_{i,k} < 0$, task $i$ consumes material $k$. We use $\mathbb{I}_{k}^{+}$ and $\mathbb{I}_{k}^{-}$ to denote the collections of tasks that produces and consumes material $k$, respectively.

In addition to tasks and materials, dashed round-corner boxes represent \textit{machines}, that is, processing units used to execute specific tasks. Tasks placed within the box of a machine means that these tasks are compatible with that machine. For example, in Figure \ref{fig:stn_sample}, tasks $i_{4}$ and $i_{5}$ can be executed on machine $j_{3}$, but cannot be executed on $j_{1}$ and $j_{2}$. We use indices and sets $j \in \mathbb{J}$ to denote machines and use the set $\mathbb{I}_{j}$ to denote tasks that can be executed on machine $j$. Each task-machine pair $(i, j)$ has a minimum and a maximum batch size requirement, denoted by $y_{i,j}^{\min}$ and $y_{i,j}^{\max}$, respectively. We denote the \textit{nominal processing time} of pair $(i, j)$ by $\tau_{i,j}$, that is, the time required to complete a task $i$ on machine $j$ when no disturbances occur.

A task, as introduced above, is a general category of activity rather than an actual execution of that activity on the shop floor. To distinguish the two, we further introduce the notion of an \textit{operation}. An operation refers to a specific instance of task execution in a schedule. In Gantt charts, an operation is typically denoted by a horizontal bar, and in the context of batch scheduling, an operation is equivalent to a batch. We denote operations using tuples and sets $\mathcal{O} \in \mathbb{O}$. In multipurpose batch scheduling settings, we identify each operation by a 5-tuple $\mathcal{O} = (i, j, b, t_{s}, t_{e})$, where the indices $i$, $j$, $b$, $t_{s}$, and $t_{e}$ represent the task, the machine, the batch size, the start time, and the completion time of the operation, respectively. Given an operation $\mathcal{O}$, we use $i(\mathcal{O})$, $j(\mathcal{O})$, $b(\mathcal{O})$, $t_{s}(\mathcal{O})$, and $t_{e}(\mathcal{O})$ to denote the corresponding components. We also use $\mathbb{K}^{+}(\mathcal{O})$ and $\mathbb{K}^{-}(\mathcal{O})$ to denote the collections of materials that $\mathcal{O}$ produces and consumes, respectively. A schedule is then a collection of operations, and we use the slice notation $\mathbb{O}_{t: t+h \mid t}$ to denote the operations that fall within the time interval $[t, t+h]$ in the schedule at time $t$. That is,
\begin{align*}
    \mathbb{O}_{t: t+h \mid t} = \{ \mathcal{O} \text{ in the schedule at $t$}: t_{s} (\mathcal{O}) \geq t \land  t_{e} (\mathcal{O}) \leq t + h \}.
\end{align*}
This slice notation will be used throughout the paper whenever we refer to the operations in a sub-window of a schedule.

Since each operation is represented as a tuple and a schedule is a collection of operations, we will frequently encounter variables indexed by tuples and by collections of tuples in later sections. To keep these expressions compact, we introduce the following notational shorthand. Let $\mathcal{A} = (a_{1}, a_{2}, \cdots, a_{m})$ be a tuple. We write $x_{\mathcal{A}}$ to mean $x_{a_1, a_2, \cdots, a_m}$. For a collection of tuples $\mathbb{A} = \{ \mathcal{A}_{1}, \cdots, \mathcal{A}_{n} \}$, we further write $\boldsymbol{x}_{\mathbb{A}}$ to mean the column vector
\begin{align*}
    \boldsymbol{x}_{\mathbb{A}} = (x_{\mathcal{A}_{1}}, x_{\mathcal{A}_{2}}, \cdots, x_{\mathcal{A}_{n}} )^{\intercal},
\end{align*}
where $\mathcal{A}_{i} \in \mathbb{A}$ for each $i \in \{1, 2, \cdots, n \}$. This shorthand appears later, for example, when we collect impact variables over all operations in a sub-window of a schedule.

The objective of scheduling is to determine a schedule that encodes four basic types of production decisions, which include timing (when to initiate a batch), batching (how much size to initiate a batch), assignment (which machine to execute a batch), and sequencing (what are the precedence relationships between batches).

\subsection{Dynamic aspect} \label{subsec:dynamic_aspect}

We describe the dynamics of a multipurpose batch process using a discrete-time dynamic system
\begin{align} \label{eq:system_dynamics_general}
    \boldsymbol{s}_{t+1} = \textsc{TransFunc} \left( \boldsymbol{s}_t, \boldsymbol{a}_t, \boldsymbol{i}_t \right).
\end{align}
Here, the index $t$ represents time steps, which starts from $t_{0}$ and ends at $t_{N}$. We denote $\mathbb{T} = \{t_{0}, t_{1}, \cdots, t_{N}\}$ as the collection of all time steps. The state vector $\boldsymbol{s}_t$ represents physical states (such as machine statuses and inventory levels) during the half-open time interval $[t-1, t)$. The action vector $\boldsymbol{a}_t$ represents system operations (such as batch initiations and machine assignments) at time point $t$. The information vector $\boldsymbol{i}_{t}$ represents realised disturbances (such as processing delays and machine failures) during $[t-1, t)$. The transition function $\textsc{TransFunc}$ maps the current state, action, and information variables to the next-step state $\boldsymbol{s}_{t+1}$.

We model disturbances using random variables. Specifically, we view the information vector $\boldsymbol{i}_{t}$ as a possible realisation of a random vector $\boldsymbol{I}_{t}$. Denote $\boldsymbol{I}_{t_0: t_N} = \left( \boldsymbol{I}_{t_0}, \boldsymbol{I}_{t_1}, \cdots, \boldsymbol{I}_{t_N} \right)$, and let $\mathsf{P}$ be the probability distribution of $\boldsymbol{I}_{t_0: t_N}$. We assume that random variables in $\boldsymbol{I}_{t_0: t_N}$ are mutually independent under $\mathsf{P}$. Also, the probability distribution $\mathsf{P}$ is known a priori, but not necessarily time-invariant. In Section \ref{sec:methodology}, we will use this probabilistic view to construct a \acrshort*{bn} that updates the posterior distributions of operation impacts as new disturbances are observed.

We group disturbances into two categories: \textit{endogenous disturbances} and \textit{exogenous disturbances}. While the endogenous disturbances refer to those that directly affect the physical state of the system, such as machine failures, the exogenous disturbances refer to those that originate outside the system, such as demand fluctuations. We use $\boldsymbol{I}^{\mathrm{En}}_{t_{0}: t_{N}}$ and $\boldsymbol{I}^{\mathrm{Ex}}_{t_{0}: t_{N}}$ to denote vectors whose elements represent endogenous and exogenous disturbances, respectively. Because different endogenous disturbances may propagate through the schedule in different patterns, we index disturbance types by $\ell \in \mathbb{L}^{\mathrm{En}}$ and write $\boldsymbol{I}_{t}^{\mathrm{En}, \ell}$ for the sub-vector whose elements correspond to the $\ell$-th type of endogenous disturbance at time $t$.

\begin{figure}[h]
\includegraphics[width=0.75\textwidth]{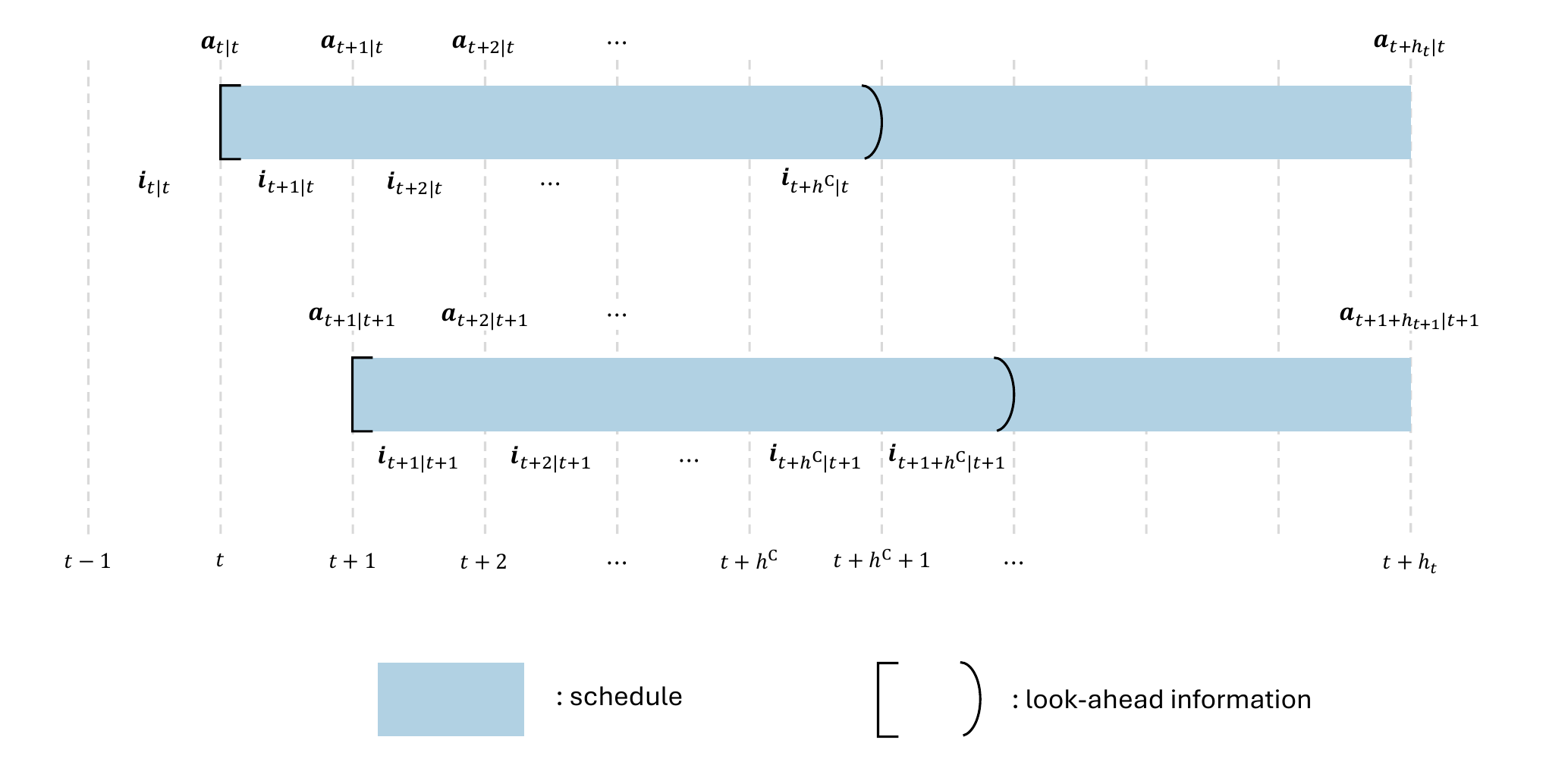}
\centering
\caption{
    Illustration of look-ahead information and action sequences.
}
\label{fig:illustration_dynamic_system}
\end{figure}

Recall that the certainty horizon refers to the horizon within which the outcome of future disturbances becomes confirmed. As illustrated in Figure \ref{fig:illustration_dynamic_system}, we describe this concept using a rolling window of look-ahead information. That is, at each time step $t$, the scheduler observes the look-ahead information
\begin{align} \label{eq:general_lookahead_info}
    \boldsymbol{i}_{t+1: t+h^{\mathrm{C}} \mid t} = \left( \boldsymbol{i}_{t+1 \mid t}, \boldsymbol{i}_{t+2 \mid t}, \cdots, \boldsymbol{i}_{t+h^{\mathrm{C}} \mid t} \right)
\end{align}
in advance. The notation $\boldsymbol{i}_{t+1: t+h^{\mathrm{C}} \mid t}$ denotes the sequence of information vectors spanning from $t+1$ to $t + h^{\mathrm{C}}$ given the current time step $t$. The integer $h^{\mathrm{C}}$ represents the length of the certainty horizon. We assume that the look-ahead information is \textit{perfect}: once an information vector $\boldsymbol{i}_{t}$ has been realised, its value will no longer change. Formally, for every $t$ such that $t + h^{\mathrm{C}} \leq t_{N}$, we assume
\begin{align*}
    \boldsymbol{i}_{t+h^{\mathrm{C}} \mid t} = \boldsymbol{i}_{t+h^{\mathrm{C}} \mid t + 1} = \cdots = \boldsymbol{i}_{t+h^{\mathrm{C}} \mid t + h^{\mathrm{C}}}.
\end{align*}
Given an integer $a > 0$, when the current time step is clear from the context, we simply write $\boldsymbol{i}_{t + a \mid t}$ as $\boldsymbol{i}_{t+a}$, which is consistent with the unconditional notation used in Equation \ref{eq:system_dynamics_general}.

We describe a planned schedule using a sequence of action vectors. As illustrated in Figure \ref{fig:illustration_dynamic_system}, if at time step $t$ we have a schedule spanning from $t$ to $t + h_{t}$, we denote this schedule by
\begin{align*}
    \boldsymbol{a}_{t: t+h_{t} \mid t} = \left( \boldsymbol{a}_{t \mid t}, \boldsymbol{a}_{t + 1 \mid  t}, \cdots, \boldsymbol{a}_{t + h_{t} \mid t} \right),
\end{align*}
where the integer $h_{t}$ represents the scheduling horizon at $t$. For two distinct time steps $t + a$ and $t + b$, where $0 < a < b$, it is not necessary that $\boldsymbol{a}_{t + b \mid t} = \boldsymbol{a}_{t + b \mid t + a}$, because a rescheduling between $t$ and $t+a$ may change the action vector at $t + b$. Also, the scheduling horizon $h_{t}$ may vary from one time step to another. For example, if no rescheduling occurs at $t+1$, the horizon shrinks as $h_{t+1} = h_{t} - 1$ because the schedule between $t$ and $t+1$ has been executed. However, the horizon $h_{t}$ must be greater than a minimum required length $h^{\min}$.

We measure the performance of a dynamic scheduling algorithm using two objectives. The first objective is to minimise the long-term cost
\begin{align} \label{eq:cost_func}
    c = \sum_{t=t_0}^{t_{N}} c_t = \sum_{t=t_0}^{t_{N}} \textsc{CostFunc}(\boldsymbol{s}_t, \boldsymbol{a}_t, \boldsymbol{i}_t),
\end{align}
where the cost function $\textsc{CostFunc}$ evaluates the cost incurred during $[t-1, t)$. The second objective is to minimise the cumulative system nervousness
\begin{align} \label{eq:nerv_func}
    d = \sum_{t=t_1}^{t_N} d_t = \sum_{t=t_1}^{t_{N}} \textsc{NervFunc}( \boldsymbol{a}_{t: t+h_t \mid t}, \boldsymbol{a}_{t-1: t-1+h_{t-1} \mid t-1}),
\end{align}
where the nervousness function $\textsc{NervFunc}$ quantifies how much the planned action sequence at $t$ differs from that at $t-1$.

Beyond the two objectives above, we introduce two theoretical reference costs that help interpret the performance of dynamic scheduling algorithms. The \textit{nominal cost} $c^{*}$ is the long-term cost obtained by solving the problem a single time over the entire time span $[t_{0}, t_{N}]$ under the assumption that all disturbance variables take their nominal values. It represents the best long-term cost achievable when no uncertainty is present in the system. The \textit{oracle cost} $c^{\infty}$ is the long-term cost obtained by solving the problem a single time over $[t_{0}, t_{N}]$ when all realisations of the disturbance variables are revealed in advance. It represents a lower bound on the long-term cost that any online algorithm can achieve. We will use these two costs as reference points when reporting computational results in Section \ref{sec:experimental_design}.

%% file: sections/section3.tex
\section{Problem formulation} \label{sec:problem_formulation}

In this section, we concretise the dynamic scheduling problem described in Section \ref{sec:problem_description} via developing two mathematical formulations: a \textit{dynamic system formulation} that describes how the system states evolve over time, and a \textit{\acrshort*{milp} formulation} that generates planned schedules during online execution. We summarise the role of each formulation below. Both formulations are presented in full in Appendices \ref{appendix:transition_functions} and \ref{appendix:milp_formulation}.

The dynamic system formulation concretises the transition function in Equation \ref{eq:system_dynamics_general} for the multipurpose batch process. This formulation involves three groups of variables: (1) the state variables (elapsed time and batch size of ongoing batches, inventory levels, backlog levels), (2) the action variables (batch initiations, batch input sizes, raw material purchases, product shipments), and (3) the information variables (machine breakdowns, processing time variations, yield losses, urgent demands). Given these variables, we derive a collection of logic-based transition rules that describe how the system state evolves from one time step to the next under the combined influence of actions and disturbances. Our formulation models four types of commonly encountered disturbances: demand fluctuations (exogenous), machine breakdowns, processing time variations, and yield losses (all endogenous). We also assume that no changeover times are required between different families of tasks and that storage tanks have no capacity limit. This formulation can be extended to other disturbances and assumptions. We discuss how to make such extensions alongside the methodology in Section \ref{sec:methodology}. The complete state-transition equations are provided in Appendix \ref{appendix:transition_functions}.

Extending the dynamic system formulation to other operational settings (such as different operational conditions or additional disturbance types) requires reformulating its transition rules, which can be non-trivial. The challenge typically arises when the logical interactions between variables become complex: Modellers can easily overlook some corner cases, that is, specific combinations of current-step variable values that yield non-obvious next-step states. Missing such cases produces transition rules that no longer correctly represent the actual system behaviour. To address this, we also propose an empirical procedure that derives transition rules while guaranteeing that every corner case is inspected. The procedure represents all current-step variables as a set of \textit{decision trees}, where each path from root to leaf corresponds to one combination of variable values, and the next-step state is assigned at the leaf. We present the full procedure in Appendix \ref{appendix:decision_tree}.

The second formulation is an \acrshort*{milp} formulation that serves as a \textit{schedule generator}. For readers familiar with static scheduling of multipurpose batch processes, this \acrshort*{milp} can be viewed as a variant of the standard discrete-time formulation. At each rescheduling time point $t$, the \acrshort*{milp} takes the current system state and the look-ahead information over $[t, t+h^{\mathrm{C}}]$ as input, and returns a new action sequence $\boldsymbol{a}_{t: t+h_t \mid t}$ as output. The lifting, material balance, machine occupancy, and bounding constraints follow the standard discrete-time form. Disturbance parameter values are taken from the look-ahead information within the certainty horizon, and from their nominal values beyond it. In addition to these standard ingredients, two features are specific to our formulation. First, a set of initial state constraints aligns the decision variables at $t' = t$ with the observed system state, so that the generated schedule continues consistently from the current state. Second, the formulation uses three hierarchical objectives: minimum total cost, maximum consistency of batch initiations with the previous schedule, and minimum batch-initiation lateness. The complete formulation, including the decision variables, parameters, constraints, bar notations, and virtual-time indexing, is presented in Appendix \ref{appendix:milp_formulation}.

The three objectives are solved hierarchically rather than aggregated into a single weighted objective. That is, we first solve the \acrshort*{milp} to minimise the total cost. We then add the resulting minimum as an upper bound on the total cost, and re-solve to maximise consistency. We finally add the resulting maximum as a lower bound on the consistency, and re-solve to minimise batch-initiation lateness. This hierarchy matters in multipurpose batch scheduling, where many schedules can share the same minimum cost: Schedules that preserve batch initiations from the previous schedule are strongly preferred in practice, both to reduce system nervousness and to improve operational robustness. The three objective functions are stated explicitly in Appendix \ref{appendix:milp_formulation}.

%% file: sections/section4.tex
\section{Framework overview} \label{sec:framework_overview}

With the problem and its formulations in place, we turn to our Bayesian dynamic scheduling method. Before entering the technical details of the methodology, this section first presents the motivation and key ideas behind our Bayesian dynamic scheduling method, then introduces the overall framework at a high level, and finally walks the reader through a small illustrative example. The goal is to convey the core idea of the framework before committing to the mathematical details that follow in Section \ref{sec:methodology}.

\subsection{Motivation and key ideas} \label{subsec:motivation}

We first motivate our framework by examining a generic rescheduling decision problem. Consider the scenario where at a rescheduling time point $t$, we generate a schedule spanning from $t$ to $t + h_{t}$. Then, the system advances to $t + h$ (where $0 < h < h^{\mathrm{C}}$) without revising the schedule. Now, looking forward from the time point $t$, we have obtained the new look-ahead information over $[t + h^{\mathrm{C}}, t + h + h^{\mathrm{C}})$, while the information for $[t + h + h^{\mathrm{C}}, t + h_{t})$ remains unobserved. As a scheduler, the problem is to determine whether to reschedule at $t + h$, and if so, how to revise the schedule.

To address this problem, we first notice that, the information over $[t + h^{\mathrm{C}}, t + h + h^{\mathrm{C}})$ has been observed. Naturally, we can decide based on this part of information. For example, assume that from information variables $\boldsymbol{i}_{t + h^{\mathrm{C}}: t + h + h^{\mathrm{C}}}$, we know an upstream operation will lose a certain percentage of yield. Then, to ensure a sufficient inventory level for initiating downstream operations, one reasonable option is to insert additional operations. This class of approaches is exactly what reactive rescheduling methods do: They decide based on the observed information. On the other hand, the information over $[t + h + h^{\mathrm{C}}, t + h_{t})$ remains unobserved. However, we also notice that, due to the inherent structure of the production system (such as the \acrshort*{stn} layout, the changeover policies), we can update the uncertainty within this interval by reasoning from the observed information $\boldsymbol{i}_{t + h^{\mathrm{C}}: t + h + h^{\mathrm{C}}}$. For example, assume that during $[t + h^{\mathrm{C}}, t + h + h^{\mathrm{C}})$, we observe that a midstream operation on a bottleneck machine will require more than its nominal processing time to complete. Then, downstream operations will be more likely (or even certainly) not able to initiate as scheduled, because the replenishment of input materials is delayed.

From a decision-making perspective, rescheduling decisions should leverage available information as much as possible. That is, using both the first, determined part of information (from $t + h$ to $t + h + h^{\mathrm{C}}$) and the second, updated, but unobserved part of information (from $t + h + h^{\mathrm{C}}$ to $t + h_{t}$). For the first part, designing some deterministic logic-based heuristics is not difficult. Often, challenges come from the second part. That is, the industry often has diverse prior knowledge about production system structures, operational conditions, and disturbance characteristics. Difficulties exist in developing an approach that simultaneously encodes such prior knowledge, unifies with the first-part observation, and reasons about the second-part uncertainty. Furthermore, it is even not trivial to decide how to quantify the second-part uncertainty. Possible quantification options may include discrete-valued indicators, formulations that encode expert knowledge, event-based metrics, and so on.

These challenges motivate our Bayesian dynamic scheduling method. Specifically, at the core of our method, we propose a probabilistic framework that (1) encodes the prior knowledge using a graph structure, which is intuitive and visually explainable, (2) unifies the representations of the first- and second-part information, and therefore (3) allows efficient reasoning about the second-part uncertainty. The probabilistic framework relies on constructing a random variable $Z_{\mathcal{O}}^{\ell}$ (referred as the \textit{impact variable}), which, loosely speaking, describes how severe the operation $\mathcal{O}$ is impacted by the $\ell$-th type of endogenous disturbances. However, as one of the main novelties of our method, the impact variable $Z_{\mathcal{O}}^{\ell}$, by itself, does not encode \textit{global semantics}, such as:
\begin{align*}
\begin{minipage}{0.85\textwidth}
    \noindent
    ``the impact on the operation $\mathcal{O}$ due to system disturbances''.
    \end{minipage}
\end{align*}
Instead, the impact variable $Z_{\mathcal{O}}^{\ell}$ encodes \textit{disentangled local semantics}, such as:
\begin{align*} 
\begin{minipage}{0.85\textwidth}
    \noindent
    ``regarding a specific disturbance type, the impact on the operation $\mathcal{O}$ due to (1) the impacts on parents of $\mathcal{O}$ and (2) the disturbance over the operation $\mathcal{O}$ itself''.
    \end{minipage}
\end{align*}

Although at first glance, this local semantic may appear counter-intuitive and not suffice to describe the impact on the entire schedule, it can greatly simplify the reasoning about the second-part uncertainty. To explain, let a \acrfull*{dag} describe the dependencies between operations in a schedule, where nodes represent operations and arcs represent dependencies of an operation on another. Then, if we were to infer the global impact on an operation $\mathcal{O}$, essentially, we are facing a query conditioned by an evidence chain. For example:
\begin{align*}
\begin{minipage}{0.85\textwidth}
    \noindent
    ``Given the disturbance information on the operation $\mathcal{O}$ itself, the disturbance information on the parents of $\mathcal{O}$, the disturbance information on the parents of parents of $\mathcal{O}$, and so on up to the earliest ancestors of $\mathcal{O}$, what is the impact on the operation $\mathcal{O}$?''
\end{minipage}
\end{align*}
Answering this type of queries is often difficult, because they are complexly entangled by not only the propagation effects from ancestral operations, but also the inherent production system structures and various disturbance characteristics.

However, if we instead work on the disentangled local impact, we actually shift to addressing impact variables with the \textit{local Markov property}. The formal reason for this property will be explained in Section \ref{sec:theoretical_foundation}. Intuitively, the impact variables associated with the parents of $\mathcal{O}$ suffice to encode the information propagated from all earlier ancestor operations. As a result, when inferring the impact on the operation $\mathcal{O}$, we need only answer local queries such as:
\begin{align*} 
\begin{minipage}{0.85\textwidth}
    \noindent
    ``Given the impacts on the parents of $\mathcal{O}$ and the disturbance information on $\mathcal{O}$ itself, what is the impact on the operation $\mathcal{O}$?''
    \end{minipage}
\end{align*}
In other words, we no longer need to examine all the disturbances over the ancestors of $\mathcal{O}$, but can focus on only the local operation family consisting of $\mathcal{O}$ and its parents. This greatly simplifies the reasoning about the second-part uncertainty.

This global-to-local idea coincides with the principle of \acrshort*{bn}s: decomposing a high-dimensional joint probability distribution into a set of local conditional probability distributions. Theoretically, this principle leads to the fact (as we will prove in Section \ref{sec:theoretical_foundation}) that, the impact variables admit a \acrshort*{bn} whose structure follows the inter-operation dependencies in the prior knowledge. Specifically, the local Markov property guarantees that this \acrshort*{bn} will exactly encode the joint probability distribution of these impact variables. This theoretical soundness allows convenient modelling in industrial settings. Beyond modelling, once the \acrshort*{bn} has been developed, we can apply the rich library of inference algorithms established by the \acrshort*{bn} community, such as variable elimination, belief propagation, and particle-based algorithms. Given the first-part information as evidence, these algorithms can efficiently update the posterior distribution of every impact variable beyond the certainty horizon. For these reasons, the Bayesian dynamic scheduling framework developed in this paper is built on this modelling choice.

\subsection{Overall framework} \label{subsec:framework}

\begin{figure}[h]
\includegraphics[width=\textwidth]{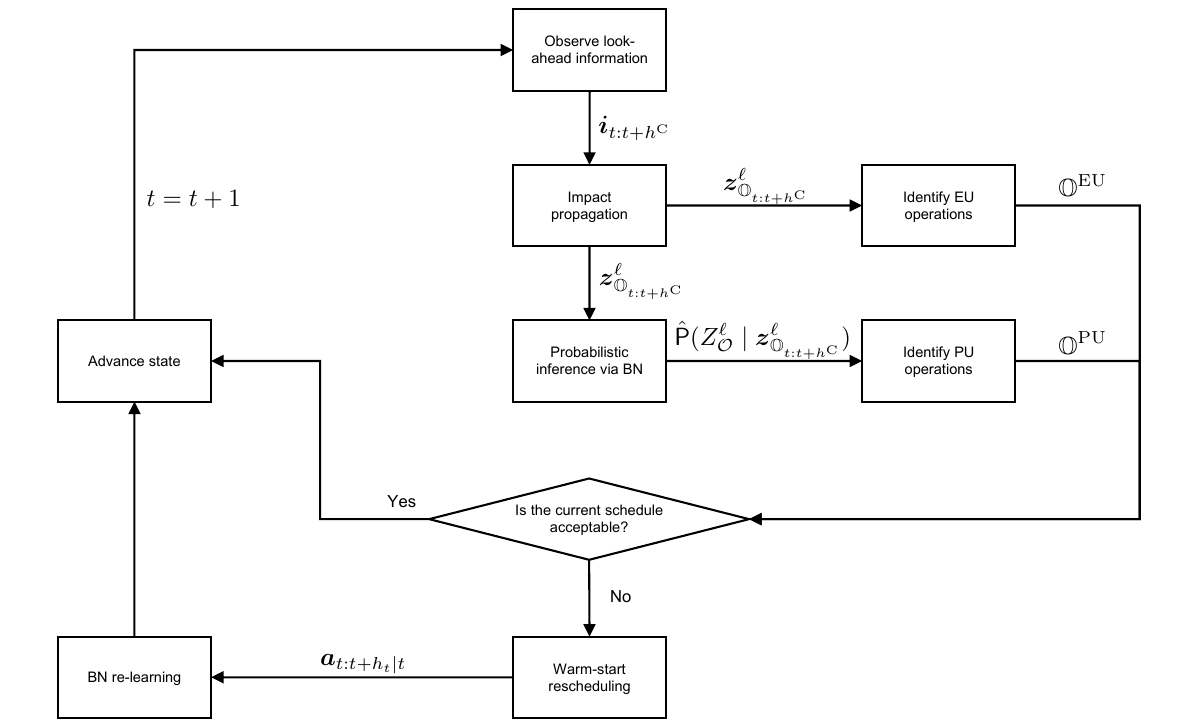}
\centering
\caption{
    Flowchart of the overall procedure of the Bayesian dynamic scheduling method. BN: \acrshort*{bn}; EU: evidentially unrecoverable; PU: probabilistically unrecoverable.
}
\label{fig:framework_scheme}
\end{figure}

Figure \ref{fig:framework_scheme} shows the overall procedure of our method. At each time step $t$, the scheduler observes the look-ahead information $\boldsymbol{i}_{t:t+h^{\mathrm{C}}}$ (the first-part information) within the next certainty horizon. This observed information is then used by an impact propagation procedure (which will be detailed in Section \ref{sec:methodology}) to compute the realised values of impact variables $\boldsymbol{z}^{\ell}_{\mathbb{O}_{t:t+h^{\mathrm{C}}}}$ for the operations within the certainty horizon.

The realised impact variables then feed two parallel branches. The first branch flags the \textit{evidentially unrecoverable} operations $\mathbb{O}^{\mathrm{EU}}$: An operation $\mathcal{O}$ is evidentially unrecoverable when the realised impact on $\mathcal{O}$ exceeds a deterministic severity threshold. The second branch performs probabilistic inference on the current \acrshort*{bn} and flags the \textit{probabilistically unrecoverable} operations $\mathbb{O}^{\mathrm{PU}}$: An operation is probabilistically unrecoverable when the posterior probability that $\mathcal{O}$ is unrecoverable exceeds a probabilistic threshold.

Given $\mathbb{O}^{\mathrm{EU}}$ and $\mathbb{O}^{\mathrm{PU}}$, the framework decides whether the current schedule is acceptable. If the percentage of unrecoverable operations is below a threshold, the schedule is regarded acceptable and the framework simply advances the state to the next time step. Otherwise, the framework triggers a warm-start rescheduling that fixes the start times of recoverable operations and re-optimises only the unrecoverable ones. This warm-start rescheduling generates a new action sequence $\boldsymbol{a}_{t:t+h_{t} \mid t}$. From this new schedule, the framework learns both the structure and the parameters of a new \acrshort*{bn} online, and then advances the state to the next time step.

The detailed algorithm modules that realise each step of this procedure are presented in Section \ref{sec:methodology}.

\subsection{Illustrative example} \label{subsec:illustrative_example}

To make the workflow concrete, we illustrate the framework through a small example. Consider a multipurpose batch process that produces a single product $k_{3}$ from a raw material $k_{1}$ through one intermediate material $k_{2}$. As Figure \ref{fig:illustrative_stn} shows, the process involves two tasks. Task $i_{1}$ converts $k_{1}$ into $k_{2}$, and can run on either of two machines: On machine $j_{1}$ a single batch has size $1$ and processing time one time period, and on machine $j_{2}$ a single batch has size $2$ and processing time two time periods. Task $i_{2}$ converts $k_{2}$ into $k_{3}$, and runs on machine $j_{3}$ with batch size $1$ and processing time one time period. We consider a single type of endogenous disturbance in this example: In every time period, each running batch is subject to a one-period delay with prior probability $0.05$.

\begin{figure}[h]
\includegraphics[width=0.75\textwidth]{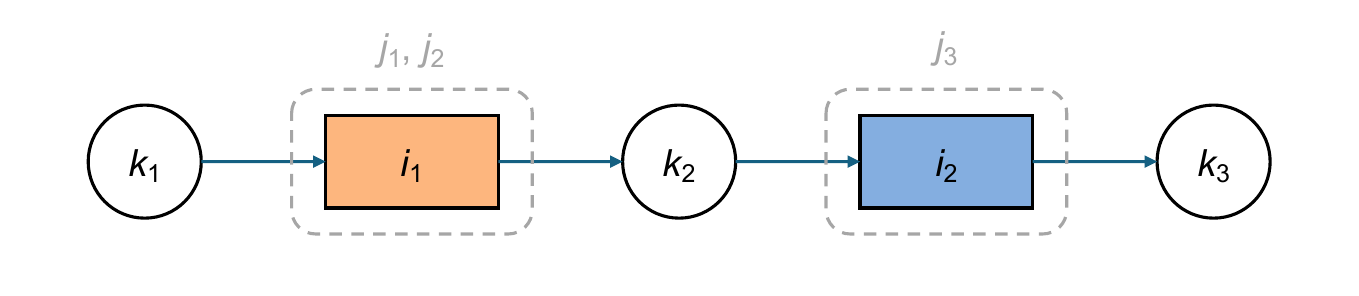}
\centering
\caption{
    \acrshort*{stn} of the illustrative example.
}
\label{fig:illustrative_stn}
\end{figure}

Suppose that at time $t = 0$, the scheduler has the planned schedule shown in Figure \ref{fig:illustrative_gantt}, which contains seven operations $\mathcal{O}_{1}, \mathcal{O}_{2}, \ldots, \mathcal{O}_{7}$ over the horizon $[t_{0}, t_{6}]$. Among them, operations $\mathcal{O}_{1}, \mathcal{O}_{2}, \mathcal{O}_{3}$ instantiate task $i_{1}$ and operations $\mathcal{O}_{4}, \mathcal{O}_{5}, \mathcal{O}_{6}, \mathcal{O}_{7}$ instantiate task $i_{2}$. Figure \ref{fig:illustrative_bn} shows the corresponding operation-dependency graph and the \acrshort*{bn} constructed over the impact variables. In this graph, the arcs encode dependencies of two types. First, material dependencies follow the production-consumption relation between tasks $i_{1}$ and $i_{2}$. For example, $\mathcal{O}_{1}$ produces the batch of $k_{2}$ consumed by $\mathcal{O}_{4}$, and $\mathcal{O}_{3}$ produces the two batches consumed by $\mathcal{O}_{6}$ and $\mathcal{O}_{7}$. Second, machine-occupancy dependencies follow the operation sequence on each machine. For example, $\mathcal{O}_{1}$ and $\mathcal{O}_{2}$ run back-to-back on $j_{1}$, and $\mathcal{O}_{4}, \mathcal{O}_{5}, \mathcal{O}_{6}, \mathcal{O}_{7}$ run sequentially on $j_{3}$.

\begin{figure}[h]
\includegraphics[width=\textwidth]{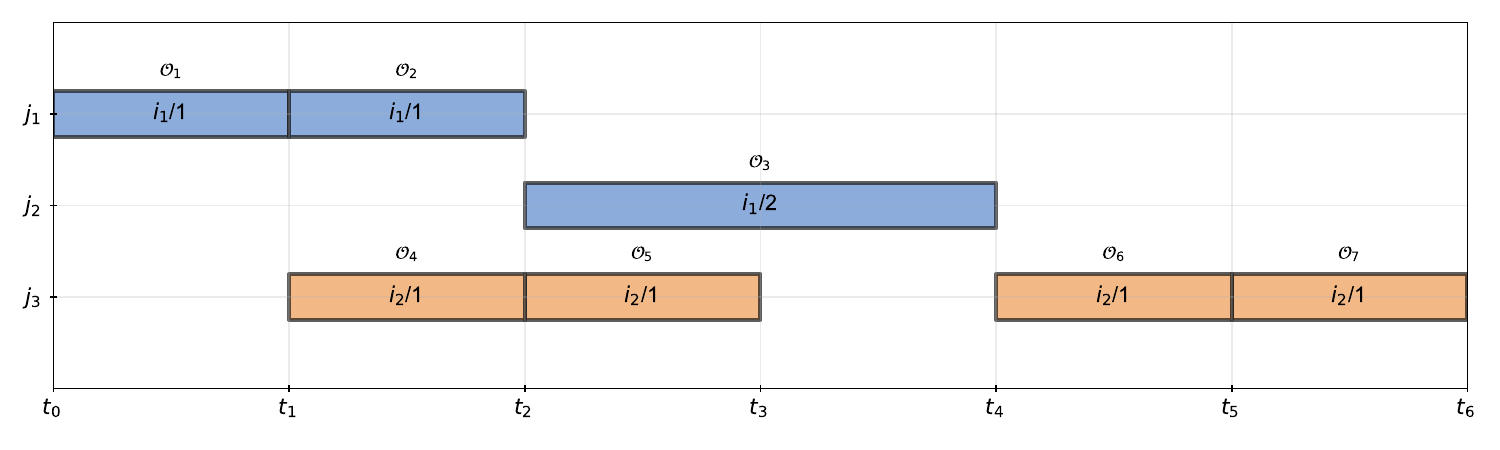}
\centering
\caption{
    Planned schedule of the illustrative example. The label inside each bar shows the task name and the batch size. The operation index is annotated above the bar. The idle gap of one time period on machine $j_{3}$ between $\mathcal{O}_{5}$ and $\mathcal{O}_{6}$ is induced by the material-availability constraint, since $\mathcal{O}_{6}$ cannot start before $\mathcal{O}_{3}$ completes at $t_{4}$.
}
\label{fig:illustrative_gantt}
\end{figure}

\begin{figure}[h]
\includegraphics[width=0.9\textwidth]{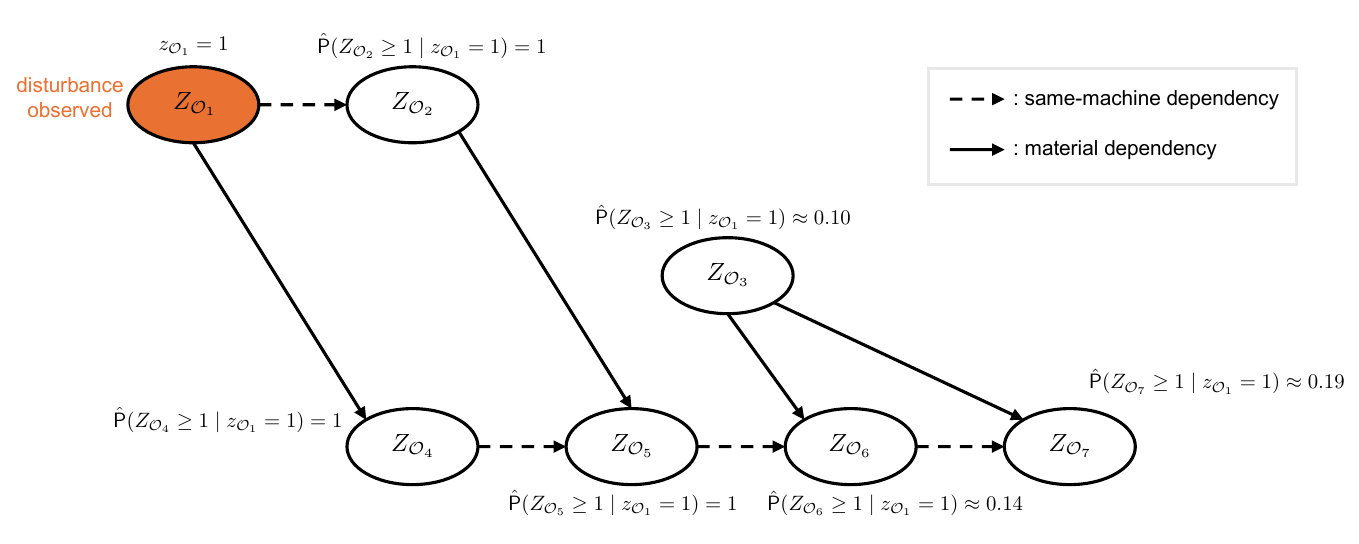}
\centering
\caption{
    \acrshort*{bn} over the impact variables of the illustrative example. The shaded node $Z_{\mathcal{O}_{1}}$ represents the observation at $t = 1$. Posterior marginals $\hat{\mathsf{P}}(Z_{\mathcal{O}_{i}} = 1 \mid Z_{\mathcal{O}_{1}} = 1)$ obtained via inference algorithms are shown next to each unobserved node.
}
\label{fig:illustrative_bn}
\end{figure}

For each operation $\mathcal{O}$, the framework introduces an integer-valued impact variable $Z_{\mathcal{O}}$ that counts the number of time periods by which $\mathcal{O}$ is delayed relative to the planned start time of $\mathcal{O}$. Rigorously, the full notation $Z_{\mathcal{O}}^{\ell}$ should carry a superscript $\ell$ for the disturbance type. In this example, we simply drop $\ell$ because only one type is considered. The lowercase impact $z_{\mathcal{O}}$ (that is, the realised value of $Z_{\mathcal{O}}$) is computed via the impact propagation procedure from the observed disturbance information. By the local semantics introduced in Subsection \ref{subsec:motivation}, the realised impact $z_{\mathcal{O}}$ depends only on the disturbance information on $\mathcal{O}$ itself and the impacts on the parents of $\mathcal{O}$. When the schedule contains slack between $\mathcal{O}$ and a parent of $\mathcal{O}$, the slack may absorb the impact from that parent before this impact reaches $\mathcal{O}$.

We now apply the framework to this example at time $t = 0$. With assumed certainty horizon $h^{\mathrm{C}} = 1$, the framework observes the disturbance information on every operation ongoing in $[t_{0}, t_{1}]$. In our schedule, only $\mathcal{O}_{1}$ is ongoing in this interval, and the observation shows that $\mathcal{O}_{1}$ is delayed by one time period. Impact propagation over the operation-dependency graph then yields the realised impact $z_{\mathcal{O}_{1}} = 1$. Since $z_{\mathcal{O}_{1}}$ reaches the evidence threshold of $1$, the framework flags $\mathcal{O}_{1}$ as evidentially unrecoverable, that is, $\mathbb{O}^{\mathrm{EU}} = \{\mathcal{O}_{1}\}$.

For the operations $\mathcal{O}_{2}, \ldots, \mathcal{O}_{7}$ that are beyond the certainty horizon, the disturbance information on each operation is unobserved, so impact propagation alone cannot determine the realised impacts. Instead, the framework performs probabilistic inference on the \acrshort*{bn} of Figure \ref{fig:illustrative_bn} to obtain the posterior probability $\hat{\mathsf{P}}(Z_{\mathcal{O}} \geq 1 \mid z_{\mathcal{O}_{1}} = 1)$ that each of these operations is delayed. Although the inference is probabilistic, the resulting posteriors fall into three qualitatively different groups.
\begin{itemize}
    \item For $\mathcal{O}_{2}$, $\mathcal{O}_{4}$ and $\mathcal{O}_{5}$, the schedule provides no slack between these operations and the delayed $\mathcal{O}_{1}$. Specifically, operation $\mathcal{O}_{2}$ runs on $j_{1}$ immediately after $\mathcal{O}_{1}$, so the delay on $\mathcal{O}_{1}$ propagates directly through machine occupancy. Operation $\mathcal{O}_{4}$ starts at the planned completion of $\mathcal{O}_{1}$ and consumes the $k_{2}$ batch from $\mathcal{O}_{1}$, so the delay propagates through both timing and material. Operation $\mathcal{O}_{5}$ does not consume material from $\mathcal{O}_{1}$ directly, but follows $\mathcal{O}_{4}$ on $j_{3}$, so the delay reaches $\mathcal{O}_{5}$ through the machine-occupancy dependencies $\mathcal{O}_{1} \to \mathcal{O}_{4} \to \mathcal{O}_{5}$. As a result, the \acrshort*{bn} outputs posterior probability $1$ to a delay on each of these operations regardless of any further disturbance. Therefore, we have $\hat{\mathsf{P}}(Z_{\mathcal{O}} \geq 1 \mid z_{\mathcal{O}_{1}} = 1) = 1$ for $\mathcal{O} \in \{\mathcal{O}_{2}, \mathcal{O}_{4}, \mathcal{O}_{5}\}$.
    \item For $\mathcal{O}_{6}$ and $\mathcal{O}_{7}$, the slack on $j_{3}$ between $\mathcal{O}_{5}$ and $\mathcal{O}_{6}$ absorbs the impact propagated along the machine-occupancy channel. However, two sources of impact remain: the disturbance information on $\mathcal{O}_{6}$ and $\mathcal{O}_{7}$ themselves, and the impact propagated through the unobserved ancestor $\mathcal{O}_{3}$. The posteriors therefore increases over their priors only moderately: $\hat{\mathsf{P}}(Z_{\mathcal{O}_{6}} \geq 1 \mid z_{\mathcal{O}_{1}} = 1) \approx 0.14$ and $\hat{\mathsf{P}}(Z_{\mathcal{O}_{7}} \geq 1 \mid z_{\mathcal{O}_{1}} = 1) \approx 0.19$. Note that the numerical posterior values reported here are illustrative only. We will report concrete simulation-based estimates in Section \ref{sec:results}.
    \item For $\mathcal{O}_{3}$, which has no ancestor affected by the observation, the posterior coincides with the prior: $\hat{\mathsf{P}}(Z_{\mathcal{O}_{3}} \geq 1 \mid z_{\mathcal{O}_{1}} = 1) \approx 1 - 0.95^{2} \approx 0.10$.
\end{itemize}
With the probability threshold set to $0.5$, the framework flags $\mathcal{O}_{2}$, $\mathcal{O}_{4}$, and $\mathcal{O}_{5}$ as probabilistically unrecoverable, that is, $\mathbb{O}^{\mathrm{PU}} = \{\mathcal{O}_{2}, \mathcal{O}_{4}, \mathcal{O}_{5}\}$. The fraction of unrecoverable operations within the scheduling horizon is $|\mathbb{O}^{\mathrm{EU}} \cup \mathbb{O}^{\mathrm{PU}}| / 7 = 4/7 \approx 0.57$, which exceeds any reasonable threshold. The framework therefore triggers a warm-start rescheduling at $t = 0$. Specifically, the warm-start fixes the start times of the recoverable operations $\{\mathcal{O}_{3}, \mathcal{O}_{6}, \mathcal{O}_{7}\}$ and re-optimises the start times and machine assignments of the four unrecoverable ones $\{\mathcal{O}_{1}, \mathcal{O}_{2}, \mathcal{O}_{4}, \mathcal{O}_{5}\}$. A feasible rescheduling, for example, moves $\mathcal{O}_{2}$ from $j_{1}$ onto the still-idle machine $j_{2}$ and thereby recovers the impact propagated along the machine-occupancy channel. From this new schedule, the framework constructs an updated \acrshort*{bn} online, after which the system advances to $t = 1$.

For contrast, consider a periodic rescheduling strategy with rescheduling period $2$ that ignores the observation at $t = 0$ and only reschedules at $t = 2$. By the time step $t=2$, $\mathcal{O}_{2}$, $\mathcal{O}_{4}$ and $\mathcal{O}_{5}$ have already missed their planned start times, the backlog cost has accumulated for two time periods, and the scheduler faces a more constrained \acrshort*{milp} in which fewer repairs (such as moving $\mathcal{O}_{2}$ onto $j_{2}$) remain feasible. Reducing the rescheduling period to $1$ would reclaim some of this lost ground but at the cost of increased nervousness, since most rescheduling procedures would still trigger when no disturbance has been observed. Our framework avoids this trade-off because a single probabilistic inference on the \acrshort*{bn} answers both questions at once: when to reschedule, and how to reschedule (that is, which operations to fix and which to re-optimise).


%% file: sections/section5.tex
\section{Methodology} \label{sec:methodology}

In this section, we detail the algorithm modules that realise the Bayesian dynamic scheduling framework introduced in Section \ref{sec:framework_overview}. The framework consists of several modules. While each module by itself incorporates the ideas of Bayesian decision-making, different modules interact dynamically with each other to automate the rescheduling loop of the framework. The framework is also customisable. A user can replace any module with a problem-specific variant, for example, to adapt to a different production workflow or a different disturbance setting. To persuade readers that our framework can be applied to other dynamic scheduling problems, such as job shop scheduling and continuous processes, we describe every module at a high level and wrap the problem-specific details as function arguments. We close each module with a brief discussion of whether and how the module can be customised. With this organisation, a reader who wishes to implement our framework on a different scheduling problem only needs to follow the overall procedure and substitute the problem-specific function arguments.

Algorithm \ref{algo:bds} presents the framework as pseudo code. While the framework generally follows a closed-loop dynamic scheduling workflow, the main differences appear in updating the \acrshort*{bn} $\mathcal{B}_{\ell}$ and computing the posterior distribution $\mathsf{P} (Z_{\mathcal{O}}^{\ell} \mid \boldsymbol{z}_{\mathbb{O}_{t + h: t + h + h^{\mathrm{C}} }}^{\ell} )$. Specifically, at first (Line \ref{algo_bds_line_initialise_system}), we initialise the system by setting up the initial time step $t_{0}$, the initial state $\boldsymbol{s}_{t_{0}}$, and an empty action sequence. Also, we sample the initial look-ahead information from the disturbance distribution $\mathsf{P}(\boldsymbol{I}_{t: t + h^{\mathrm{C}}})$. Then, we initialise with a group of empty \acrshort*{bn}s (Line \ref{algo_bds_line_initialise_bn}) and set the cumulative cost and nervousness values to zero (Line \ref{algo_bds_line_initialise_c_and_d}).

\vspace{0.5cm}
\RestyleAlgo{ruled}
\SetKwComment{Comment}{/*}{*/}
\SetKw{Or}{or}
\begin{algorithm}[H]
\setstretch{1.25}
\caption{Bayesian dynamic scheduling under incomplete look-ahead information} \label{algo:bds}
\Begin{
    Initialise $t \coloneqq t_0$, $\boldsymbol{s}_{t} \coloneqq \boldsymbol{s}_{t_0}$, $\boldsymbol{a}_{t: t+h_t} \coloneqq \varnothing$, and $\boldsymbol{i}_{t: t+h^{\mathrm{C}} } \sim \mathsf{P} (\boldsymbol{I}_{t: t+h^{\mathrm{C}}}) $\; \label{algo_bds_line_initialise_system}

    Initialise $\mathcal{B}_{\ell} = (\mathcal{G}^{\mathrm{BN}}_{\ell}, \hat{\mathsf{P}}) \coloneqq \varnothing$ for each $\ell \in \mathbb{L}^{\mathrm{En}}$\; \label{algo_bds_line_initialise_bn}

    Initialise $c \coloneqq 0$, $d \coloneqq 0$\; \label{algo_bds_line_initialise_c_and_d}

    \While{$t \leq t_N$\label{algo_bds_main_loop}}{
        Compute the realised values of impact variables within the next certainty horizon $\boldsymbol{z}_{\mathbb{O}_{t: t+h^{\mathrm{C}}}}^{\ell} \coloneqq \textsc{ImpactPropagation} \bigl(\ell, \mathcal{G}_{\ell}^{\mathrm{Op}}, \{ \boldsymbol{i}^{\mathrm{En}, \ell}_{t_{s} (\mathcal{O}): t_{e} (\mathcal{O}) } \in \boldsymbol{i}_{t: t+h^{\mathrm{C}}}^{\mathrm{En}, \ell}: \mathcal{O} \in \mathbb{O}_{t: t + h^{\mathrm{C}}} \}, \textsc{IsoFunc}, \textsc{PropFunc} \bigr)$, where the input \acrshort*{dag} $\mathcal{G}^{\mathrm{Op}}_{\ell}$ mirrors the structure of $\mathcal{G}^{\mathrm{BN}}_{\ell}$ while replacing each impact variable node $Z_{\mathcal{O}}^{\ell}$ with its associated operation $\mathcal{O}$ \Comment*[r]{See Algorithm \ref{algo:impact_propagation}} \label{algo_bds_line_compute_evidence}

        Identify evidentially unrecoverable operations $\mathbb{O}^{\mathrm{EU}}_{t: t+h^{\mathrm{C}} }$ \Comment*[r]{See Equation \ref{eq:identify_eu_operations}} \label{algo_bds_line_get_eu_operations}

        Compute the posterior distribution $\hat{\mathsf{P}} (Z_{\mathcal{O}}^{\ell} \mid \boldsymbol{Z}^{\ell}_{\mathbb{O}_{t: t + h^{\mathrm{C}}}} = \boldsymbol{z}^{\ell}_{\mathbb{O}_{t: t + h^{\mathrm{C}}}} ) \coloneqq \textsc{VariableElimination} (\mathcal{B}_{\ell}, Z_{\mathcal{O}}^{\ell}, \boldsymbol{z}_{\mathbb{O}_{t: t + h^{\mathrm{C}}}}^{\ell}) $ for each $\mathcal{O} \in \mathbb{O}_{t + h^{\mathrm{C}}: t + h_{t}}$ \Comment*[r]{See Equation \ref{eq:variable_elimination}} \label{algo_bds_line_compute_posterior_distribution}

        Identify probabilistically unrecoverable operations $\mathbb{O}^{\mathrm{PU}}_{t+h^{\mathrm{C}}: t+h_{t} }$ \Comment*[r]{See Equation \ref{eq:identify_pu_operations}} \label{algo_bds_line_get_pu_operations}

        Determine whether to reschedule $\varphi \coloneqq \textsc{WhenToReschedule} (\mathbb{O}_{t: t+h_t}, \mathbb{O}^{\mathrm{EU}}_{t: t+h^{\mathrm{C}} }, \mathbb{O}^{\mathrm{PU}}_{t+h^{\mathrm{C}}: t+h_{t}} , \gamma_3)$ \Comment*[r]{See Algorithm \ref{algo:when_to_reschedule}} \label{algo_bds_line_reschedule_triggering}

        \If{$\varphi = \mathsf{True}$ \Or $h_t \leq h^{\mathrm{min}}$\label{algo_bds_line_reschedule_condition}}{

            Generate a new action sequence $\boldsymbol{a}_{t: t + h } \coloneqq \textsc{HowToReschedule} \bigl(
            (\boldsymbol{s}_{t}, \boldsymbol{a}_{t: t + h_{t} }, \boldsymbol{i}_{t: t+h^{\mathrm{C}} }), \mathbb{O}_{t: t+h_t }, \mathbb{O}_{t: t+h^{\mathrm{C}} }^{\mathrm{EU}}, \mathbb{O}_{t + h^{\mathrm{C}}: t+h_t }^{\mathrm{PU}}, h^{\mathrm{C}} + h_{t} \bigr)$ \Comment*[r]{See Algorithm \ref{algo:how_to_reschedule}} \label{algo_bds_line_schedule_generation}

            For each $\ell \in \mathbb{L}^{\mathrm{En}}$, learn a new \acrshort*{bn} $\mathcal{B}_{\ell} \coloneqq (\mathcal{G}_{\ell}^{\mathrm{BN}}, \hat{\mathsf{P}})$, where $\mathcal{G}_{\ell}^{\mathrm{BN}} \coloneqq \textsc{StructureLearning} (\boldsymbol{Z}_{\mathbb{O}_{t: t+h_t }}^{\ell}, \chi_{\ell}^{\mathrm{Temp}}, \chi_{\ell}^{\mathrm{Spat}} )$ and $\hat{\mathsf{P}} \coloneqq \textsc{ParameterLearning} (n, \mathcal{G}_{\ell}^{\mathrm{BN}})$ \Comment*[r]{See Algorithms \ref{algo:structure_learning} and \ref{algo:parameter_learning}} \label{algo_bds_line_bn_learning}
        }\label{algo_bds_line_reschedule_procedure_end}

        Update $c \coloneqq c + \textsc{CostFunc}(\boldsymbol{s}_t, \boldsymbol{a}_{t}, \boldsymbol{i}_{t})$, $d \coloneqq d + \textsc{NervFunc}(\boldsymbol{a}_{t: t+h_t \mid t}, \boldsymbol{a}_{t+1: t+h_{t+1} \mid t+1 })$ \Comment*[r]{See Equations \ref{eq:cost_func} and \ref{eq:nerv_func}} \label{algo_bds_line_update_states}

        Update $\boldsymbol{s}_{t+1} \coloneqq \textsc{TransFunc} \left( \boldsymbol{s}_t, \boldsymbol{a}_{t}, \boldsymbol{i}_{t} \right)$, $\boldsymbol{a}_{t+1 : t+1 + h_{t+1} \mid t+1} \coloneqq \boldsymbol{a}_{t+1 : t + h_{t} \mid t} $, and $\boldsymbol{i}_{t + h^{\mathrm{C}} + 1 } \sim \mathsf{P}( \boldsymbol{I}_{t + h^{\mathrm{C}} + 1} )$ \Comment*[r]{See Equation \ref{eq:system_dynamics_general}} \label{algo_bds_line_update_c_and_d}

        Update $t \coloneqq t+1$\; \label{algo_bds_line_proceed_next_step}
    }
}
\end{algorithm}
\vspace{0.5cm}

Next, we enter the main procedure (Line \ref{algo_bds_main_loop}). At each time step $t$, for the first-part information, we use the \textsc{ImpactPropagation} algorithm (which will be presented in Subsection \ref{subsec:impact_variables}) to compute the realised values of impact variables within the next certainty horizon (Line \ref{algo_bds_line_compute_evidence}). Here, the argument $\mathcal{G}^{\mathrm{Op}}_{\ell}$ represents the graph structure that encodes the inter-operation dependencies implied by the prior knowledge. The set of vectors $\{ \boldsymbol{i}_{t_s (\mathcal{O}): t_e (\mathcal{O})}^{\mathrm{En}, \ell} \in \boldsymbol{i}_{t: t + h^{\mathrm{C}}}: \mathcal{O} \in \mathbb{O}_{t: t + h^{\mathrm{C}}} \}$ contains realised values of information variables regarding the $\ell$-th type of endogenous disturbances over operations within the next certainty horizon. The functions \textsc{IsoFunc} and \textsc{PropFunc} represent \textit{isolation function} and \textit{propagation function}, respectively. These two functions, which can be customised by users to adapt to their problem-specific contexts, are used to describe the propagation effects of specific endogenous disturbances. We will discuss the design of these two functions in Subsection \ref{subsec:impact_variables} and Section \ref{sec:theoretical_foundation}.

Having obtained the realised values of impact variables, we then identify \textit{evidentially unrecoverable} operations in the current schedule (Line \ref{algo_bds_line_get_eu_operations}). Here, the term ``evidentially unrecoverable'' means that, confirmed by the observed evidence, these operations are severely impacted such that they cannot be ``recovered'' from disturbances. As a result, these operations should be rescheduled. Specifically, the evidentially unrecoverable operations are identified by
\begin{align} \label{eq:identify_eu_operations}
    \mathbb{O}^{\mathrm{EU}}_{t: t+h^{\mathrm{C}} } \coloneqq \bigl\{ \mathcal{O} \in \mathbb{O}_{t: t+h^{\mathrm{C}} } : \bigvee_{\ell \in \mathbb{L}^{\mathrm{En}}}  z_{\mathcal{O}}^{\ell} \geq \gamma_{1}^{\ell}  \bigr\},
\end{align}
where the parameter $\gamma_{1}^{\ell}$ defines the threshold above which, an operation becomes unrecoverable due to the disturbance type $\ell$. Intuitively, Equation \ref{eq:identify_eu_operations} states that, an operation is evidentially unrecoverable if, regarding any type of endogenous disturbances, the realised value of its impact variable $z_{\mathcal{O}}^{\ell}$ exceeds the threshold $\gamma_{1}^{\ell}$.

Then, for the second-part information, we use the \acrshort*{bn} $\mathcal{B}_{\ell}$ to update the posterior distribution over the impact variables beyond the certainty horizon (Line \ref{algo_bds_line_compute_posterior_distribution}). By having the posterior distributions, we then identify \textit{probabilistically unrecoverable} operations (Line \ref{algo_bds_line_get_pu_operations}) as
\begin{align} \label{eq:identify_pu_operations}
    \mathbb{O}^{\mathrm{PU}}_{t+h^{\mathrm{C}}: t+h_{t} } \coloneqq \Bigl\{ \mathcal{O} \in \mathbb{O}_{t+h^{\mathrm{C}}: t+h_{t} } :  1 - \prod_{\ell \in \mathbb{L}^{\mathrm{En}}} \bigl(1 - \hat{\mathsf{P}} (Z_{\mathcal{O}}^{\ell} < \gamma_{1}^{\ell} \mid \boldsymbol{z}_{\mathbb{O}_{t: t + h^{\mathrm{C}} }}^{\ell} ) \bigr) \geq \gamma_{2}  \Bigl\},
\end{align}
where the parameter $\gamma_{2}$ represents the probabilistic threshold above which, an operation is regarded as unrecoverable due to the updated posterior distribution. Here, the left-hand side of the inequality is derived by a sequence of set operations as follows.
\begin{subequations} \label{eq:identify_pu_operations_derivation}
\begin{align}
    &\hat{\mathsf{P}} \left(
            \text{$\mathcal{O}$ is unrecoverable given the observed evidence}
    \right) \\
    = \; &\hat{\mathsf{P}} \left(
        \parbox{9.8cm}{\noindent regarding at least one type of endogenous disturbances, $\mathcal{O}$ is unrecoverable given the observed evidence}
    \right) \\
    = \; & 1 - \hat{\mathsf{P}} \left(
        \parbox{8.5cm}{\noindent for all types of endogenous disturbances, $\mathcal{O}$ is not unrecoverable given the observed evidence}
    \right) \\
    = \; & 1 - \prod_{\ell \in \mathbb{L}^{\mathrm{En}}} \bigl(1 - \hat{\mathsf{P}} (Z_{\mathcal{O}}^{\ell} < \gamma_{1}^{\ell} \mid \boldsymbol{z}_{\mathbb{O}_{t: t + h^{\mathrm{C}} }}^{\ell} ) \bigr),
\end{align}
\end{subequations}
where the last equality follows from the assumption of mutual independence in disturbances.

Having the evidentially and probabilistically unrecoverable operations identified, we then decide whether to reschedule (Line \ref{algo_bds_line_reschedule_triggering}). That is, if the function \textsc{WhenToReschedule} returns $\mathsf{True}$, or the span of the current schedule is less than the minimum required length $h^{\min}$, we enter the rescheduling procedure (Lines \ref{algo_bds_line_reschedule_condition}--\ref{algo_bds_line_reschedule_procedure_end}). Otherwise the module \textsc{WhenToReschedule} returns $\mathsf{False}$, we update the system and proceed to the next time step (Lines \ref{algo_bds_line_update_states}--\ref{algo_bds_line_proceed_next_step}). In Line \ref{algo_bds_line_reschedule_triggering}, the argument $\gamma_{3}$ represents the percentage threshold such that, if the number of unrecoverable operations exceeds $\gamma_{3}$ (such as $0.3$) times the total number of operations in the schedule, we trigger a rescheduling.

Regarding the rescheduling procedure (Lines \ref{algo_bds_line_reschedule_condition}--\ref{algo_bds_line_reschedule_procedure_end}), we first generate a new schedule using the \textsc{HowToReschedule} algorithm, which essentially solves a warm-start version of the \acrshort*{milp} presented in Section \ref{sec:problem_formulation}. That is, we unfix the start times of unrecoverable operations, while fixing the start times of the remaining operations. After that, we construct a new \acrshort*{bn} $\mathcal{B}_{\ell} = (\mathcal{G}^{\mathrm{BN}}_{\ell}, \hat{\mathsf{P}})$ for each $\ell \in \mathbb{L}^{\mathrm{En}}$ according to the new schedule, where the \acrshort*{bn} structure $\mathcal{G}^{\mathrm{BN}}_{\ell}$ is generated by the \textsc{StructureLearning} algorithm, and the estimated distribution $\hat{\mathsf{P}}$ is returned by the \textsc{ParameterLearning} algorithm. Specifically, we construct $\mathcal{G}^{\mathrm{BN}}_{\ell}$ directly from the prior knowledge about multipurpose batch processes, while we estimate $\hat{\mathsf{P}}$ using Monte Carlo simulations. We present these two algorithms in Subsection \ref{subsec:online_bn}.

Next, we present the details of each algorithm integrated in the overall framework. In Subsection \ref{subsec:impact_variables}, we present the \textsc{ImpactPropagation} algorithm. In Subsection \ref{subsec:online_bn}, we present algorithms related to \acrshort*{bn}s, including \textsc{StructureLearning}, \textsc{ParameterLearning}, and \textsc{VariableElimination}. In Subsection \ref{subsec:rescheduling_strategy}, we present \textsc{WhenToReschedule} and \textsc{HowToReschedule}.

\subsection{Impact propagation} \label{subsec:impact_variables}

Before presenting the pseudo code, it is helpful to explain the functionality of \textsc{ImpactPropagation} from an input-output perspective. Given (1) a \acrshort*{dag} $\mathcal{G}^{\mathrm{Op}} = (\mathbb{O}, \mathbb{A})$, which describes the inter-operation dependencies over the operation set $\mathbb{O}$ regarding the $\ell$-th type of endogenous disturbances, (2) an isolation function \textsc{IsoFunc}, and (3) a propagation function \textsc{PropFunc}, the \textsc{ImpactPropagation} algorithm maps the information variables $\{ \boldsymbol{I}^{\mathrm{En}, \ell}_{t_{s} (\mathcal{O}): t_{e} (\mathcal{O}) } : \mathcal{O} \in \mathbb{O} \}$ to the impact variables $\boldsymbol{Z}_{\mathbb{O}}^{\ell}$. When the inputs are replaced by their realised values $\{ \boldsymbol{i}^{\mathrm{En}, \ell}_{t_{s} (\mathcal{O}): t_{e} (\mathcal{O}) } : \mathcal{O} \in \mathbb{O} \}$, the algorithm returns the realised values of impact variables $\boldsymbol{z}_{\mathbb{O}}^{\ell}$. Note that the construction over random variables is valid, because random variables are essentially measurable functions, and therefore a function of a measurable function remains a measurable function. Then, from this input-output perspective, the functionality of \textsc{ImpactPropagation} can be interpreted as transforming the $I$-\textit{description} (that is, using information variables) of disturbances into the $Z$-\textit{description} (that is, using impact variables) of disturbances.

Basically, \textsc{ImpactPropagation} achieves this transformation by traversing the operation set. When arriving at a single operation $\mathcal{O}$, we require two steps. The first step is \textit{isolation}, which uses the isolation function \textsc{IsoFunc} to map the $I$-description of disturbances over $\mathcal{O}$ into the isolated impact level of $\mathcal{O}$. Here, the adjective ``isolated'' means that this impact level is only related to the operation $\mathcal{O}$ itself, while is irrelevant of any other operations in the set. Mathematically, the isolation function is denoted by
\begin{align} \label{eq:isofunc_definition}
    \underline{Z}^{\ell}_{\mathcal{O}} = \textsc{IsoFunc} ( \boldsymbol{I}_{t_s(\mathcal{O}): t_e(\mathcal{O})}^{\mathrm{En}, \ell} ; \ell),
\end{align}
where the input $\boldsymbol{I}_{t_s(\mathcal{O}): t_e(\mathcal{O})}^{\mathrm{En}, \ell}$ represents information variables over the duration of $\mathcal{O}$. The parameter $\ell$, which is placed after the semicolon, means that this mapping is parameterised by the disturbance type $\ell$. The output variable $\underline{Z}_{\mathcal{O}}^{\ell}$ represents the \textit{isolation variable} regarding the disturbance type $\ell$, where the underline notation means that this variable is isolated.

After isolation, the second step is \textit{local propagation}, which quantifies the propagation effect within the local operation family $\{ \mathcal{O} \} \cup \mathrm{Pa}^{\mathcal{G}^{\mathrm{Op}}} (\mathcal{O})$. Here, the operation $\mathcal{O}$ represents the child, and the collection $\mathrm{Pa}^{\mathcal{G}^{\mathrm{Op}}} (\mathcal{O})$ represents parents. Within this local family, we view the impact on $\mathcal{O}$ as consisting of two parts. The first part comes from the isolated impact on the child $\mathcal{O}$ itself. The second part comes from the impact propagated from parents. The local propagation effect is captured by the propagation function \textsc{PropFunc}, which is denoted by
\begin{align} \label{eq:propfunc_definition}
    Z_{\mathcal{O}}^{\ell} = \textsc{PropFunc} (\underline{Z}_{\mathcal{O}}^{\ell}, \boldsymbol{Z}_{\mathrm{Pa}^{\mathcal{G}^{\mathrm{Op}}} (\mathcal{O}) }^{\ell}; \ell),
\end{align}
where the first argument $\underline{Z}_{\mathcal{O}}^{\ell}$ represents the isolated impact on the child, and the second argument $\boldsymbol{Z}_{\mathrm{Pa}^{\mathcal{G}^{\mathrm{Op}}} (\mathcal{O}) }^{\ell}$ represents impact variables corresponding to parent operations.

The \textsc{ImpactPropagation} algorithm visits every operation in the schedule by repeatedly applying this isolation-and-propagation procedure. At every iteration, the algorithm selects an operation $\mathcal{O}$ whose parents in the \acrshort*{dag} $\mathcal{G}^{\mathrm{Op}}$ have already been visied, applies the isolation function on the information variables of $\mathcal{O}$, and then applies the propagation function on the isolation variable of $\mathcal{O}$ and the impact variables of the parents of $\mathcal{O}$. The algorithm terminates once every operation has received its impact variable. Figure \ref{fig:illustration_impact_propagation} illustrates the procedure of \textsc{ImpactPropagation} from a scheduling perspective. The overall layout of Figure \ref{fig:illustration_impact_propagation} resembles a Gantt chart, where the $x$-axis represents time steps and the $y$-axis represents machines. Distributed on each machine are operations (denoted by coloured boxes) and mutually independent information variables (denoted by round circles). The inter-operation dependencies are represented by grey arrows. For an operation $\mathcal{O}$, the isolation function (denoted by dashed arrows) accepts the information variables over $\mathcal{O}$ as input and returns the isolation variable as output. After that, the propagation function (denoted by solid arrows) accepts the impact variables of parents of $\mathcal{O}$ and the isolation variable of $\mathcal{O}$, and returns the impact variable for $\mathcal{O}$. The full pseudo code of \textsc{ImpactPropagation} is provided in Appendix \ref{appendix:algorithms}.

\begin{figure}[H]
\includegraphics[width=0.85\textwidth]{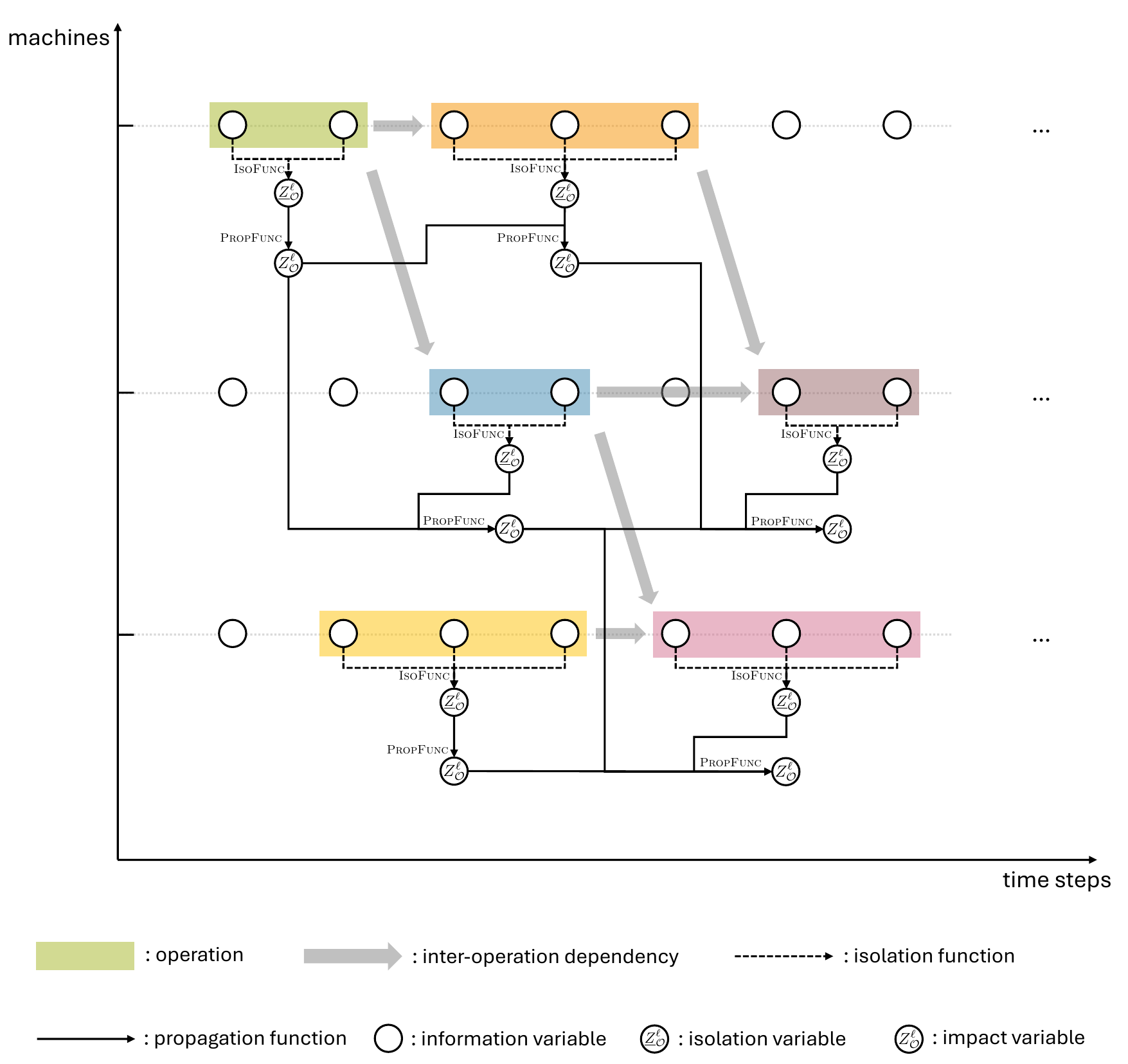}
\centering
\caption{
    Illustration of the \textsc{ImpactPropagation} algorithm.
}
\label{fig:illustration_impact_propagation}
\end{figure}

The \textsc{ImpactPropagation} algorithm can extend beyond multipurpose batch scheduling. The algorithm adapts in two directions. First, to adapt to a new production workflow, one needs to adjust the structure of the input \acrshort*{dag}. Second, to adapt to a new disturbance setting, one needs to redesign \textsc{IsoFunc} and \textsc{PropFunc} so that they capture the propagation behaviour of the new disturbances. In both directions, two rules must apply. First, \textsc{IsoFunc} can depend only on the disturbance variables related to $\mathcal{O}$ itself. Second, \textsc{PropFunc} can depend only on the isolated impact of $\mathcal{O}$ and the impact variables of the parents of $\mathcal{O}$. As we will prove in Section \ref{sec:theoretical_foundation}, any pair of \textsc{IsoFunc} and \textsc{PropFunc} that follow these two rules yields a \acrshort*{bn} with structure $\mathcal{G}_{\ell}^{\mathrm{BN}} = (\boldsymbol{Z}_{\mathbb{O}}^{\ell}, \mathbb{A})$ that exactly encodes the disturbance distribution $\mathsf{P}(\boldsymbol{I}_{t_{0}: t_{N}})$.

For our specific context, the design of the \acrshort*{dag} structure is presented later through the \textsc{StructureLearning} algorithm in Subsection \ref{subsec:online_bn}. The specific designs of \textsc{IsoFunc} and \textsc{PropFunc} for the three endogenous disturbance types of this paper, and a discussion of how to design \textsc{IsoFunc} and \textsc{PropFunc} for disturbances beyond our setting, are presented in Appendix \ref{appendix:iso_prop_funcs}.

\subsection{Bayesian Network algorithms} \label{subsec:online_bn}

In the previous subsection, we present how to construct the impact variable $Z_{\mathcal{O}}^{\ell}$ using the \textsc{ImpactPropagation} algorithm. From a modelling perspective, if we were to develop a \acrshort*{bn} $\mathcal{B}_{\ell}$ with the structure $\mathcal{G}^{\mathrm{BN}}_{\ell} = (\boldsymbol{Z}_{\mathbb{O}}^{\ell}, \mathbb{A})$, then essentially, the previous subsection establishes the semantics of each single node in $\mathcal{G}_{\ell}^{\mathrm{BN}}$. That said, we remain to develop a complete \acrshort*{bn} that encodes the impacts on the entire set of operations in the schedule. This step can be broken down into two steps: \textit{structure learning} and \textit{parameter learning}. While the structure learning aims to develop the graph structure of the \acrshort*{bn}, the parameter learning determines the local \acrfull{cpd} of each impact variable. In this subsection, we present these two learning procedures. Also, to guide the rescheduling strategy, we need to perform inference algorithms on the learned \acrshort*{bn} to update the posterior distribution of impact variables. We will also briefly introduce the inference algorithm used in our method.

\subsubsection*{Structure learning}
The \textsc{StructureLearning} algorithm aims to develop the \acrshort*{bn} structure $\mathcal{G}^{\mathrm{BN}}_{\ell}$ for disturbance type $\ell$. Recall that within our Bayesian dynamic scheduling framework (Algorithm \ref{algo:bds}), each \acrshort*{bn} structure $\mathcal{G}^{\mathrm{BN}}_{\ell}$ is paired with a \acrshort*{dag} $\mathcal{G}^{\mathrm{Op}}_{\ell}$, which mirrors the structure of $\mathcal{G}^{\mathrm{BN}}_{\ell}$ while replacing each impact variable $Z_{\mathcal{O}}^{\ell}$ with its corresponding operation $\mathcal{O}$. In other words, the graphs $\mathcal{G}^{\mathrm{Op}}_{\ell}$ and $\mathcal{G}^{\mathrm{BN}}_{\ell}$ share the same structure but differ only in their node sets. Nodes in $\mathcal{G}^{\mathrm{Op}}_{\ell}$ represent operations, while nodes in $\mathcal{G}^{\mathrm{BN}}_{\ell}$ represent impact variables. Therefore, developing the structure of $\mathcal{G}^{\mathrm{BN}}_{\ell}$ also implies developing the structure of $\mathcal{G}^{\mathrm{Op}}_{\ell}$. As a result, \textsc{StructureLearning} can be viewed as a procedure that develops a graph structure to describe how each operation depends on another.

In the context of multipurpose batch scheduling, we identify two groups of inter-operation dependencies. The first group is called \textit{temporal dependence}, which means that for two consecutive operations sharing the same machine, the latter operation depends on the previous one. In practice, this dependence can occur when operations delay. For example, consider two operations $\mathcal{O}_{1}$ and $\mathcal{O}_{2}$. Suppose that $\mathcal{O}_{1}$ and $\mathcal{O}_{2}$ share the same machine and $t_{e} (\mathcal{O}_{1}) = t_{s} (\mathcal{O}_{2})$. Then, if $\mathcal{O}_{1}$ delays by one time period, $\mathcal{O}_{2}$ will have to be postponed or revised. As a result, we assert that the operation $\mathcal{O}_{2}$ depends on $\mathcal{O}_{1}$. Generally, for an operation set $\mathbb{O}$, we add the following arcs to represent temporal dependencies.
\begin{align} \label{eq:temporal_arcs}
    \mathcal{O}' \rightarrow \mathcal{O}, \quad \forall \mathcal{O} \in \mathbb{O}, \mathcal{O}' \in \underset{\substack{ j(\mathcal{O}') = j(\mathcal{O}) \\ t_{e} (\mathcal{O}') \leq t_{s} (\mathcal{O}) } }{\mathrm{arg\,max}} t_{e} (\mathcal{O}')
\end{align}
In Equation \ref{eq:temporal_arcs}, operations $\mathcal{O}$ and $\mathcal{O}'$ represent the child and the parent, respectively. The parent $\mathcal{O}'$ is selected such that (1) it shares the same machine with $\mathcal{O}$, (2) it completes before the start time of $\mathcal{O}$, and (3) among all the operations satisfying conditions (1) and (2), $\mathcal{O}'$ has the latest completion time.

The second group is \textit{spatial dependence}, which means that for two consecutive operations within the same processing sequence, the downstream operation depends on the upstream operation. In multipurpose batch scheduling, this dependence means that initiating the downstream operation requires at least one output material from the upstream operation. In industrial settings, the spatial dependence may correspond to yield-related disturbances, such as machine breakdowns, quality defects, and contaminations. Mathematically, we represent the spatial dependence using the following arcs.
\begin{align} \label{eq:spatial_arcs}
    \mathcal{O}' \rightarrow \mathcal{O}, \quad \forall \mathcal{O} \in \mathbb{O}, \mathcal{O}' \in \underset{\substack{ \mathbb{K}^{+} (\mathcal{O}') \cap \mathbb{K}^{-} (\mathcal{O}) \neq \varnothing \\ t_{e} (\mathcal{O}') \leq t_{s} (\mathcal{O}) } }{\mathrm{arg\,max}} t_{e} (\mathcal{O}').
\end{align}
The spatial parent $\mathcal{O}'$ is selected such that (1) it produces at least one input material for $\mathcal{O}$, (2) it completes before the start time of $\mathcal{O}$, and (3) among all the operations satisfying conditions (1) and (2), $\mathcal{O}'$ has the latest completion time.

The \textsc{StructureLearning} algorithm constructs these two groups of dependencies over the impact variables $\boldsymbol{Z}_{\mathbb{O}}^{\ell}$. The binary indicators $\chi_{\ell}^{\mathrm{Temp}}$ and $\chi_{\ell}^{\mathrm{Spat}}$, which are provided as arguments, represent whether the disturbance type $\ell$ can propagate across the temporal and spatial dimensions, respectively. The full pseudo code of \textsc{StructureLearning} is provided in Appendix \ref{appendix:algorithms}.

\subsubsection*{Parameter learning}
The \textsc{ParameterLearning} algorithm aims to determine the parameters of each \acrshort*{cpd} in the \acrshort*{bn} $\mathcal{B}_{\ell}$. In principle, these parameters can be either provided by experts or estimated from data. In our problem, when the size of a schedule becomes large (for example, containing hundreds of operations), relying on expert input becomes difficult. This is because the \acrshort*{bn} will be associated with too many parameter entries for experts to provide. In addition, our framework involves dynamic learning of \acrshort*{bn}s. That is, each time a new schedule is generated, the procedure of learning a \acrshort*{bn} is automatically triggered. Therefore, it is also impractical to require human experts to dynamically provide parameters. Therefore, we estimate the parameters through Monte Carlo simulations, which is straightforward to implement and sufficient in our application.

The parameter learning procedure consists of two phases: data generation and estimation. In the data generation phase, we repeatedly sample realisations of disturbances $\boldsymbol{i}_{t_s(\mathcal{O}) : t_e(\mathcal{O})}^{\mathrm{En}, \ell, (k)} \sim \mathsf{P} (\boldsymbol{I}_{t_s(\mathcal{O}) : t_e(\mathcal{O})}^{\mathrm{En}, \ell})$ for each operation $\mathcal{O} \in \mathbb{O}$, and perform \textsc{ImpactPropagation} to obtain the realised values of impact variables $\boldsymbol{z}_{\mathbb{O}}^{\ell, (k)}$. Here, the superscript $(k)$ represents the index of episode. We collect these realised values into a dataset $\mathbb{D}$. In the estimation phase, for each impact variable $Z_{\mathcal{O}}^{\ell}$ and every possible parent-child realisation $(z_{\mathcal{O}}^{\ell}, \mathrm{pa}^{\mathcal{G}} (Z_{\mathcal{O}}^{\ell}))$, we estimate the parameter through Maximum Likelihood Estimation as
\begin{align} \label{eq:parameter_mle}
    \hat{\mathsf{P}} \bigl(z_{\mathcal{O}}^{\ell} \mid \mathrm{pa}^{\mathcal{G}}(Z_{\mathcal{O}}^{\ell})  \bigr) = \frac{
        \sum_{k=1}^{n} \mathbbm{1} \bigl\{ Z_{\mathcal{O}}^{\ell, (k)} = z_{\mathcal{O}}^{\ell} \land \mathrm{Pa}^{\mathcal{G}} (Z_{\mathcal{O}}^{\ell})^{(k)} = \mathrm{pa}^{\mathcal{G}}(Z_{\mathcal{O}}^{\ell}) \bigr\}
    }{
        \sum_{k=1}^{n} \mathbbm{1} \bigl\{ \mathrm{Pa}^{\mathcal{G}}(Z_{\mathcal{O}}^{\ell})^{(k)} = \mathrm{pa}^{\mathcal{G}}(Z_{\mathcal{O}}^{\ell}) \bigr\}
    },
\end{align}
where $\mathbbm{1}\{\cdot\}$ is the indicator function. The sample size $n$ scales with the schedule size. The full pseudo code of \textsc{ParameterLearning} is provided in Appendix \ref{appendix:algorithms}

\subsubsection*{Inference}
From an input-output perspective, an inference algorithm accepts a \acrshort*{bn} $\mathcal{B}$, a vector $\boldsymbol{X}$ containing variables for which the posterior distributions are desired, and the observed evidence $\{ \boldsymbol{E} = \boldsymbol{e} \}$, while returns the posterior distribution $\mathsf{P} (\boldsymbol{X} \mid \boldsymbol{E} = \boldsymbol{e})$. To perform inference, the probabilistic graphical model community has developed various classes of inference algorithms \cite{friedman_probabilistic_2009, darwiche_modeling_2009}. These algorithms can be broadly grouped into two categories: \textit{exact inference} and \textit{approximate inference}. The exact inference algorithms, such as variable elimination and junction tree algorithms, generate exact posterior distributions, but are sometimes computationally expensive. In contrast, the approximate inference algorithms, such as particle-based inference and variational inference, generate approximate posterior distributions but are less sensitive to problem sizes.

In this work, we apply the variable elimination (VE) algorithm to perform inference tasks. Generally speaking, VE is a type of dynamic-programming algorithm that repeatedly (1) selects a variable to eliminate, (2) multiplies all \acrshort*{cpd}s containing that variable, and (3) sums out the variable from the resulting product. This process continues until only the desired variables remain, while those irrelevant variables are eliminated. As a result, the algorithm yields the desired \acrshort*{cpd}. Because (1) the VE algorithm does not contribute to the main novelties of our work, and (2) describing VE requires formalising the concept of \textit{factors}, which is beyond the scope of this paper, we do not present a detailed pseudo code here. Instead, we denote the VE algorithm as an equation:
\begin{align} \label{eq:variable_elimination}
    \mathsf{P} ( \boldsymbol{X} \mid \boldsymbol{E} = \boldsymbol{e}) \coloneqq \textsc{VariableElimination} \left( \mathcal{B},  \boldsymbol{X}, \boldsymbol{E} = \boldsymbol{e} \right).
\end{align}
For a comprehensive walkthrough of the VE algorithm, we refer readers to \cite{friedman_probabilistic_2009}. From an implementation perspective, readers can regard VE as an algorithm that is integrated in most mainstream \acrshort*{bn} libraries, like the Branch \& Bound algorithm integrated in most \acrshort*{milp} solvers.

\subsection{Rescheduling strategies} \label{subsec:rescheduling_strategy}

As mentioned earlier, the core idea of our framework is to use the \acrshort*{bn}s $\{\mathcal{B}_{\ell}: \ell \in \mathbb{L}^{\mathrm{En}} \}$ to guide the rescheduling strategies. Essentially, a rescheduling strategy consists of two aspects: when-to-reschedule, which determines whether to trigger a rescheduling procedure at a time step $t$; and how-to-reschedule, which determines how to generate the new schedule if a rescheduling is triggered. In this subsection, we present these two procedures.

\subsubsection*{When-to-reschedule}
Our when-to-reschedule strategy is simple. If the total number of both evidentially and probabilistically unrecoverable operations exceeds a certain percentage (that is, $\gamma_{3}$) of the total number of operations in the original schedule, we trigger a rescheduling. Otherwise, we continue to execute the original schedule. The full pseudo code of \textsc{WhenToReschedule} is provided in Appendix \ref{appendix:algorithms}.

Note that in principle, one can definitely design a more finely tuned when-to-reschedule strategy based on the available information, which may involve logical conditions, surrogates, and so on. However, in this work, we intentionally design a simple strategy, while leaving more sophisticated strategies to future work. This is because, one of the main purposes of this work is to convince readers that the Bayesian dynamic scheduling framework is effective. From a statistical perspective, if we involve too many factors that may affect the final performance of a dynamic scheduling method, it would be difficult to distinguish whether the effectiveness originates from the finely designed when-to-reschedule strategy, or the Bayesian approach itself. In fact, as we will present in Section \ref{sec:results}, even using such simple strategy, the Bayesian dynamic scheduling method already performs satisfactorily in many scenarios. We will explore more finely designed strategies in future work.

\subsubsection*{How-to-reschedule}
Our how-to-reschedule strategy, \textsc{HowToReschedule}, generally follows a hierarchical approach with warm-start. First, we develop the \acrshort*{milp} constraints for the next scheduling horizon $[t, t + h'_{t}]$ according to the \acrshort*{milp} formulation presented in Section \ref{sec:problem_formulation}. Then, for operations that are neither evidentially nor probabilistically unrecoverable, we fix their start times as a warm-start. After that, we solve the \acrshort*{milp} hierarchically according to the three objective functions presented in Section \ref{sec:problem_formulation}. Finally, we decode the new action sequence $\boldsymbol{a}'_{t: t+h'_t \mid t}$ from the \acrshort*{milp} solution to advance the dynamic system. The full pseudo code of \textsc{HowToReschedule} is provided in Appendix \ref{appendix:algorithms}.

%% file: sections/section6.tex
\section{Theoretical foundation} \label{sec:theoretical_foundation}

In this section, we present the theoretical foundation of our framework. Specifically, we focus on the \textsc{ImpactPropagation} algorithm and its associated impact variables. Readers that are not interested in theories may wish to skip this section on first reading.

\subsection{Properties of impact variables} \label{subsec:properties_of_impact_variables}

Before presenting the theorem, we first introduce the aim of this subsection by reviewing our Bayesian dynamic scheduling framework from a statistical perspective. Recall that in Line \ref{algo_bds_line_bn_learning} of Algorithm \ref{algo:bds}, given an operation set $\mathbb{O}$ (in Algorithm \ref{algo:bds}, this operation set equals $\mathbb{O}_{t: t+h_t}$), we learn the \acrshort*{bn} $\mathcal{B}_{\ell}$ using the \textsc{StructureLearning} and \textsc{ParameterLearning} algorithms. Statistically, this step can be viewed as estimating the \textit{true} \acrshort*{bn} through Monte Carlo simulations. The true \acrshort*{bn}, which is denoted by $\mathcal{B}^{*}_{\ell} = (\mathcal{G}^{\mathrm{BN}}_{\ell}, \mathsf{P})$, encodes the probability distribution of impact variables implied by the inherent system disturbances. More specifically, the true \acrshort*{bn} $\mathcal{B}^{*}_{\ell}$ can be constructed through the following procedure:
\begin{itemize}
    \item Call the \textsc{StructureLearning} algorithm to produce a graph structure $\mathcal{G}_{\ell}^{\mathrm{Op}} = (\mathbb{O}, \mathbb{A})$ that represents the inter-operation dependencies over the operation set $\mathbb{O}$ (recall that $\mathcal{G}^{\mathrm{Op}}_{\ell}$ and $\mathcal{G}^{\mathrm{BN}}_{\ell}$ share the same structure).
    \item Call the stochastic version of \textsc{ImpactPropagation} to construct the impact variables $\boldsymbol{Z}_{\mathbb{O}}^{\ell}$ over $\mathcal{G}^{\mathrm{Op}}_{\ell}$.
    \item Replace each operation $\mathcal{O}$ in the graph $\mathcal{G}^{\mathrm{Op}}_{\ell}$ with its corresponding impact variable $Z_{\mathcal{O}}^{\ell}$.
\end{itemize}
At this point, readers may wish to pause and compare the procedure of constructing $\mathcal{B}_{\ell}^{*}$ with that of estimating $\mathcal{B}_{\ell}$ to understand why these three steps yield $\mathcal{B}^{*}_{\ell}$. Once this is understood, it becomes clear that the repeated execution of \textsc{ImpactPropagation} (as presented in Algorithm \ref{algo:parameter_learning} in Appendix \ref{appendix:algorithms}) is equivalent to sampling from the true \acrshort*{bn} $\mathcal{B}_{\ell}^{*}$. Ideally, as the number of samples goes to infinity, the estimated \acrshort*{bn} $\mathcal{B}_{\ell}$ converges to the true \acrshort*{bn} $\mathcal{B}_{\ell}^{*}$. However, a fundamental question remains: Is the true \acrshort*{bn} $\mathcal{B}^{*}_{\ell}$ correct? More precisely, do the impact variables $\boldsymbol{Z}_{\mathbb{O}}^{\ell}$ satisfy the conditional independence semantics encoded by the structure $\mathcal{G}^{\mathrm{BN}}_{\ell}$?

Answering this question is fundamental for establishing the theoretical correctness of our framework. This is because, if the impact variables do not satisfy the conditional independence semantics encoded by the \acrshort*{bn} structure $\mathcal{G}^{\mathrm{BN}}_{\ell}$, then the probability distribution $\mathsf{P}(\boldsymbol{Z}_{\mathbb{O}}^{\ell})$ cannot be correctly represented by $\mathcal{B}_{\ell}^{*}$. As a result, the estimated \acrshort*{bn} $\mathcal{B}_{\ell}$ will eventually converge to another probability distribution that is not consistent with the inherent system disturbances. Consequently, the posterior impacts inferred from $\mathcal{B}_{\ell}$ will be incorrect.

\begin{equation} \label{eq:mapping_i_and_z}
    \begin{tikzcd}[row sep = huge, column sep = huge]
    \boldsymbol{i}_{t_0: t_N} \arrow[mapsto]{r}{\mathsf{P}_{\boldsymbol{I}_{t_0: t_N}} } \arrow[mapsto]{d}{\textsc{ImpactPropagation}} & \mathsf{P}(\boldsymbol{i}_{t_0: t_N})  \\
    \boldsymbol{z}_{\mathbb{O}}^{\ell} \arrow[mapsto]{r}{\mathsf{P}_{\boldsymbol{Z}_{\mathbb{O}}^{\ell} } } & \mathsf{P} (\boldsymbol{z}_{\mathbb{O}}^{\ell})
    \end{tikzcd}
\end{equation}

Equation \ref{eq:mapping_i_and_z} restates this question through a mapping diagram. The top-left element $\boldsymbol{i}_{t_{0}: t_N}$ represents a possible realisation of disturbance scenarios. The probability distribution $\mathsf{P}_{\boldsymbol{I}_{t_0: t_N}}$, which is positioned over the top arrow, is induced by the probability measure $\mathsf{P}$ which represents the inherent system disturbances. The \textsc{ImpactPropagation} algorithm maps a given realisation $\boldsymbol{i}_{t_0: t_N}$ to its corresponding realisation of impact variables $\boldsymbol{z}_{\mathbb{O}}^{\ell}$. The probability distribution $\mathsf{P}_{\boldsymbol{Z}_{\mathbb{O}}^{\ell}}$, which is positioned over the bottom arrow, represents the impact distribution also induced by $\mathsf{P}$. The question is whether the probability distribution $\mathsf{P} (\boldsymbol{z}_{\mathbb{O}}^{\ell})$ satisfies the conditional independence semantics encoded by $\mathcal{G}^{\mathrm{BN}}_{\ell}$.

We answer this question through Theorem \ref{thm:impact_independence}. The complete proofs are provided in Appendix \ref{appendix:proofs}.

\begin{theorem} \label{thm:impact_independence}
    Let $\mathcal{G} = (\mathbb{O}_{t+h^{\mathrm{C}}: t+h_{t}}, \mathbb{A})$ be a \acrshort*{dag} that represents the dependencies between operations in $\mathbb{O}_{t+h^{\mathrm{C}}: t+h_{t}}$. For each operation $\mathcal{O} \in \mathbb{O}_{t + h^{\mathrm{C}}: t + h_{t} \mid t}$ and endogenous disturbance type $\ell \in \mathbb{L}^{\mathrm{En}}$, construct the random variable $Z_{\mathcal{O}}^{\ell} = \textsc{PropFunc} (\underline{Z}_{\mathcal{O}}^{\ell}, \boldsymbol{Z}_{\mathrm{Pa}^{\mathcal{G}} (\mathcal{O})}^{\ell} ; \ell )$ according to the \textsc{ImpactPropagation} algorithm (that is, Algorithm \ref{algo:impact_propagation} in Appendix \ref{appendix:algorithms}). Then, for each operation $\mathcal{O} \in \mathbb{O}_{t + h^{\mathrm{C}}: t + h_{t}}$, we have
    \begin{align*}
        \mathsf{P} \models Z_{\mathcal{O}}^{\ell} \independent \boldsymbol{Z}_{\overline{\mathrm{De}}^{\mathcal{G}} (\mathcal{O})}^{\ell} \mid \boldsymbol{Z}_{\mathrm{Pa}^{\mathcal{G}} (\mathcal{O})}^{\ell}.
    \end{align*}
\end{theorem}

Theorem \ref{thm:impact_independence} actually forms the theoretical foundation for the entire Bayesian dynamic scheduling method. Intuitively, Theorem \ref{thm:impact_independence} tells us that, if we already have a \acrshort*{dag} $\mathcal{G}^{\mathrm{Op}}$ that describes the inter-operation dependencies implied by a schedule, then, the impact variables generated by running the \textsc{ImpactPropagation} algorithm over this \acrshort*{dag} will satisfy the conditional independence semantics implied by the \acrshort*{dag} structure of $\mathcal{G}^{\mathrm{Op}}$ (that is, the conditional independence statements implied by a \acrshort*{bn} structure). This guarantee has two implications in practice. First, it ensures that the posterior distribution that we inferred from the \acrshort*{bn} of impact variables is consistent with the actual disturbance distributions in the system. This implies that our Bayesian dynamic scheduling algorithm is probabilistically correct. On the other hand, for each node $\mathcal{O}$ in the \acrshort*{dag} $\mathcal{G}^{\mathrm{Op}}$, if we replace it with its corresponding impact variable $Z_{\mathcal{O}}^{\ell}$, then the resulting \acrshort*{dag} (that is, the \acrshort*{bn} $\mathcal{G}^{\mathrm{BN}}$) will correctly encode the joint distribution of impact variables involved in the schedule. Therefore, Theorem \ref{thm:impact_independence} actually unifies the operation \acrshort*{dag} $\mathcal{G}^{\mathrm{Op}}$ and the \acrshort*{bn} structure $\mathcal{G}^{\mathrm{BN}}$. This unification enables us to conveniently switch between these two \acrshort*{dag}s in the Bayesian dynamic scheduling algorithm, as we presented in Algorithm \ref{algo:bds}.

\subsection{Design of Bayesian Network structure} \label{subsec:design_of_bn_structure}

According to Theorem \ref{thm:impact_independence}, in principle, the correctness of our Bayesian dynamic scheduling method is not affected by the \acrshort*{bn} structure $\mathcal{G}^{\mathrm{BN}}$, which also describes the inter-operation dependencies. In other words, the structure of $\mathcal{G}^{\mathrm{BN}}$ only affects the semantics of impact variables, but does not affect the correctness of the probabilistic inference. Therefore, when applying our Bayesian dynamic scheduling method to a practical instance, even if there exist multiple possible approaches of constructing the \acrshort*{dag} $\mathcal{G}^{\mathrm{BN}}$, it is hard to assert which structure is ``correct'' while others are not. However, typically, constructing \acrshort*{dag}s that can be interpreted from an engineering perspective is beneficial, because the resulting impact variables will have physically interpretable semantics. As a result, the posterior distribution inferred from the \acrshort*{bn} will have practical semantics. To the best of our knowledge, constructing such physically meaningful structures is not too difficult in most applications. For example, for parallel machine environments, the dependencies only exist in consecutive operations on the same machine; for job shop environments, the dependencies exist in either consecutive operations on the same machine, or consecutive operations within a job sequence; for batch processes (our case), the dependencies are related to processing routes, as we formalised in the \textsc{StructureLearning} algorithm. As a short conclusion, in practice, despite the fact that the system features in different dynamic scheduling problems may be versatile, the principle of designing the \acrshort*{dag} structure is to reflect the inter-operation dependencies implied by the system as much as possible.

%% file: sections/section7.tex
\section{Experimental design} \label{sec:experimental_design}

We evaluate our Bayesian dynamic scheduling framework by solving four benchmark examples adopted from Li et al.\ \cite{li_novel_2022}, including two small-sized examples and two medium-sized examples. The detailed parameters for these examples are presented in Section S3 of the Supplementary Information, including the \acrshort*{stn} structure, the task-machine configurations, the material configurations, and the demand profiles. The system dynamics follow the formulation summarised in Section \ref{sec:problem_formulation} and detailed in Appendix \ref{appendix:transition_functions}, while schedules are generated using the \acrshort*{milp} formulation summarised in Section \ref{sec:problem_formulation} and detailed in Appendix \ref{appendix:milp_formulation}.

Throughout the four examples, we discretise the time axis into one-hour intervals (that is, each time step represents an hour), while consider a dynamic scheduling timespan of 240 hours (that is, $t_N = t_{240}$). The certainty horizon $h^{\mathrm{C}}$ is set to $12$ hours for all types of disturbances. The minimum required scheduling horizon $h^{\min}$ is set to $48$ hours.

\subsubsection*{Randomness and replicability}
From an implementation perspective, three sources of randomness may affect the final performance of our framework: realisations of disturbance scenarios, sampling processes within the Monte Carlo simulation, and branching processes inside \acrshort*{milp} solvers. Ideally, to achieve fully replicable results, these three sources of randomness should be controlled through a unified random seed. However, this is surprisingly tricky in practice. On the one hand, to reduce computational time, we often need to parallelise the sampling process through multiprocessing. While each worker process maintains its own Random Number Generator state, different processes may draw samples in different orders due to the fork-vs-spawn semantics and nondeterministic programming task scheduling. Therefore, it is nontrivial to replicate the sampling processes within the Monte Carlo simulation. On the other hand, random seeds inside the \acrshort*{milp} solvers may behave differently across different operating systems, such as Windows and MacOS.

Therefore, to obtain statistically reliable results, we perform the experiments as follows. For each of the four benchmark problems, we generate five disturbance scenarios. For each disturbance scenario, we run our framework ten times with the same algorithm parameters. Also, to set up a baseline comparison, we also implement the periodically-completely rescheduling strategy (that is, calling the \acrshort*{milp} solver at fixed intervals with no variables fixed in advance), where the scheduling frequency ranges from every $1$ hour to every $12$ hours. Note that this range is exhaustive, because the certainty horizon equals $12$ hours, and we cannot ensure the feasibility of action variables if the rescheduling frequency is longer than every 12 hours.

\subsubsection*{Performance metrics}
As presented in Equations \ref{eq:cost_func} and \ref{eq:nerv_func}, we consider two performance metrics: the cumulative cost $c$ and the cumulative nervousness $d$. The cost consists of three parts: the cost of initiating a batch, the cost of holding inventories, and the cost of having backlogs. The cost parameters can be found in Section S3 of the Supplementary Information.

In addition to the cumulative cost, we also report the two theoretical reference costs introduced earlier in Section \ref{sec:problem_description}: the nominal cost $c^{*}$ and the oracle cost $c^{\infty}$. Briefly, to compute $c^{*}$, we solve an \acrshort*{milp} over the entire timespan $[t_{0}, t_{N}]$ under nominal operating conditions. To compute $c^{\infty}$, we use the same \acrshort*{milp} but supply all realisations of disturbances as known parameters. Together, these two reference costs allow us to decompose the gap $c - c^{*}$ between an actual dynamic scheduling algorithm and the theoretical optimal. The gap $c^{\infty} - c^{*}$ reveals the inherent effect of disturbances on the long-term cost, that is, the additional cost that disturbances introduce even when the scheduler has infinite look-ahead information. On the other hand, the gap $c - c^{\infty}$ represents the additional cost from both the incompleteness of the look-ahead information and the suboptimal performance of a dynamic scheduling algorithm.

\subsubsection*{Algorithm parameters}
The main algorithm parameters used in this experiment are presented in Table \ref{tab:algorithm_parameters}. To ensure the practicability of our method, across all the four benchmark problems, we begin with a default setting of parameters where $\gamma_{2} = 0.5$ and $\gamma_{3} = 0.5$. When the performance of our algorithm is not competitive (specifically, in Example 4), we allow ourselves to manually adjust $\gamma_{2}$ and $\gamma_{3}$ through a few times of trial-and-errors, while keep other parameters unchanged.

\begin{table}[H]
    \caption{Main algorithm parameters}
    \small
    \setstretch{1.5}
    \centering
    \begin{tabular}{llll}
        \toprule
        & $n$ & $\gamma_{2}$ & $\gamma_{3}$ \\
        \midrule
        Example 1 & 1000 & 0.5 & 0.5 \\
        Example 2 & 1000 & 0.5 & 0.5 \\
        Example 3 & 2500 & 0.5 & 0.5 \\
        Example 4 & 4000 & 0.4 & 0.6 \\
        \bottomrule
    \end{tabular}
    \label{tab:algorithm_parameters}
\end{table}

With respect to the commercial solver, throughout the four cases, we set the time limit for solving each \acrshort*{milp} to 300 seconds. The \acrshort*{mip} gap is set to 0.01.

For the three endogenous disturbances considered in this paper, the impact of processing time variation propagates through both temporal and spatial dimensions, while the impacts of machine breakdown and yield loss propagate only through the spatial dimension.

Throughout the four cases, we set the thresholds for determining unrecoverable random states as follows. We set $\gamma_{1}^{\ell_{3}} = 100$, which means that in any cases, we do not consider that the yield loss can result in an operation being unrecoverable. We set $\gamma_{1}^{\ell_{2}} = 1$, which means that when an operation is delayed by at least one hour, it is considered as unrecoverable. We set also $\gamma_{1}^{\ell_{1}} = 1$, which means that when an operation is affected by machine breakdown, it is considered as unrecoverable.

%% file: sections/section8.tex
\section{Computational results} \label{sec:results}
The compilation of numerical results is visualised in Figure \ref{fig:performance_result}. In each subplot, the $x$-axis represents costs (the value of $c$), while the $y$-axis represents the number of operation changes (the value of $d$). The light blue vertical line represents the value of $c^{*}$. The pink vertical line represents the value of $c^{\infty}$. Each light brown scatter represents a result that is generated by the periodic, completely rescheduling policy. The integer to the top right of a light brown scatter represents the corresponding rescheduling frequency, which is measured in hours. Each light green scatter represents a result that is generated by our Bayesian dynamic scheduling method.

\begin{sidewaysfigure}
\begin{figure}[H]
    \includegraphics[width=0.85\textwidth]{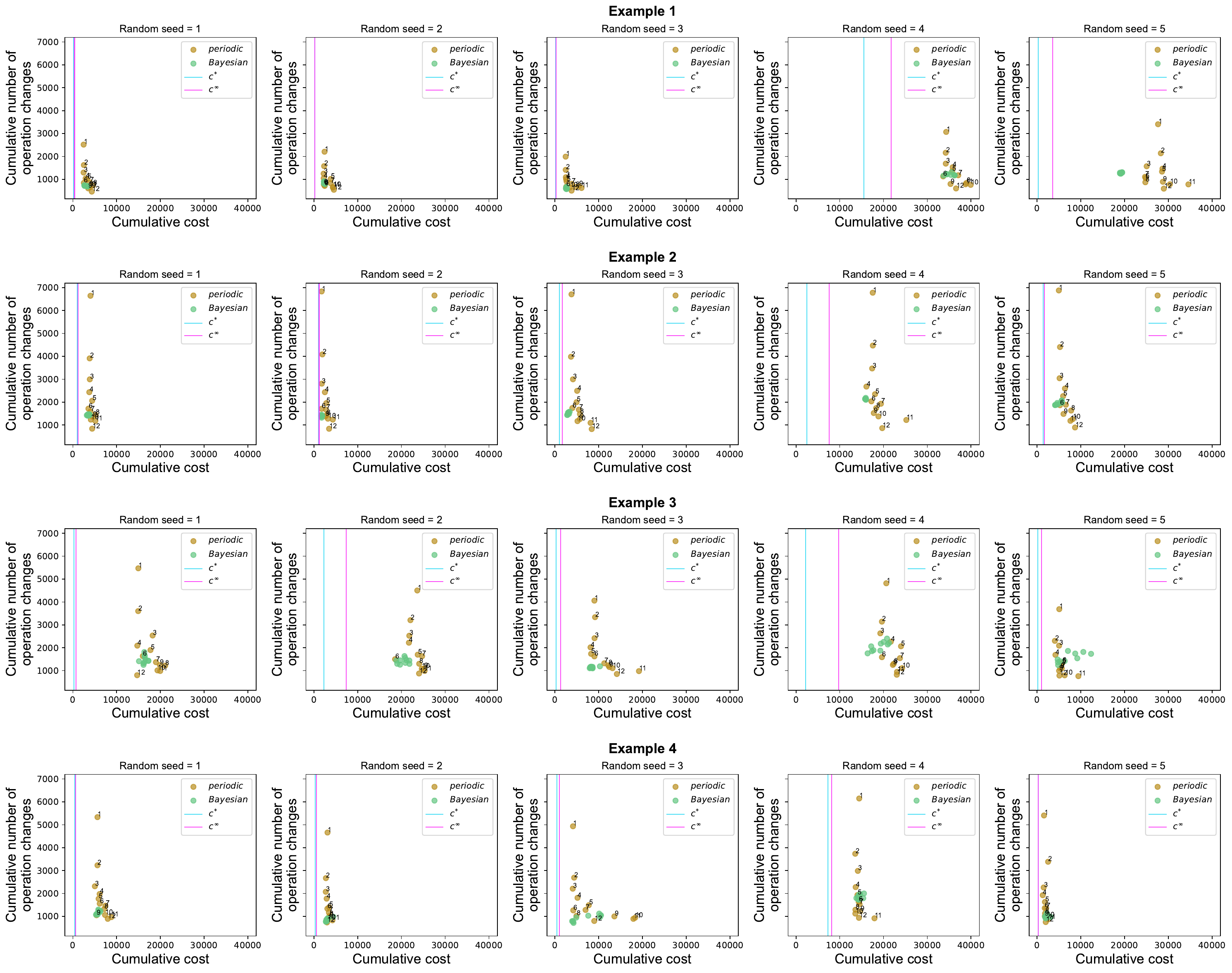}
    \centering
    \caption{
        Compilation of computational results.
        In each subplot, the x-axis represents costs (the value of $c$). The y-axis represents the number of operation changes (the value of $d$). The light blue vertical line represents the value of $c^{*}$. The pink vertical line represents the value of $c^{\infty}$. Each light brown scatter represents a result that is generated by the periodic, completely rescheduling policy. The integer to the top right of a light brown scatter represents the corresponding rescheduling frequency, which is measured in hours. Each light green scatter represents a result that is generated by our Bayesian dynamic scheduling method.
        }
    \label{fig:performance_result}
\end{figure}
\end{sidewaysfigure}

From Figure \ref{fig:performance_result}, we want to highlight several observations. First, in some scenarios, such as Example 1 with Random seed = 1 and Example 4 with Random seed = 5, the gap between $c^{*}$ and $c^{\infty}$ is negligible that the two vertical lines are almost overlapped. This means that, in these scenarios, although there exist disturbances in the system, if the scheduler has infinite look-ahead information, it is possible to absorb the effect of disturbances by rearranging the schedule over the entire dynamic scheduling timespan. However, in other scenarios, such as Example 1 with Random seed = 4 and Example 3 with Random seed = 2, the gap between $c^{*}$ and $c^{\infty}$ is significant. In these scenarios, the realisation of disturbances is ``bad'', because even if the scheduler has infinite look-ahead information, it is impossible to recover the nominal cost by rearranging the schedule. From a practical perspective, this observation implies that the effect of disturbances may depend largely on the specific realisation.

Another observation is that, there hardly exist scenarios that either the periodic rescheduling policy or our Bayesian dynamic scheduling method achieves a long-term cost that is close to $c^{\infty}$. As mentioned before, this gap consists of two parts. While the first part is due to the incompleteness of the look-ahead information, the second part is due to the suboptimality of the specific dynamic scheduling algorithm. Because we generally do not know the best cumulative cost that can be achieved by any potential dynamic scheduling algorithms, we are still unaware of how much the information incompleteness contributes to this gap. From our perspective, understanding the structure of this gap is helpful to understand the mechanism behind the performance of a dynamic scheduling algorithm. Therefore, this issue is worth exploring in the future.

With respect to the periodic, completely rescheduling policy, we have some findings that have not been mentioned by the existing literature. First, throughout all the scenarios, we can generally observe a $\frac{1}{x}$-like shape of cost-nervousness curve that is generated by the periodically rescheduling policy. Typical examples include Example 3 with Random seed = 3 and Example 4 with Random seed = 3. This finding implies that, in practice, if we were to use a periodic rescheduling policy, there will potentially be a tradeoff between the long-term cost and the system nervousness. Another finding is that, under the constraint of incomplete look-ahead information, it is not true that higher rescheduling frequency will lead to a better long-term cost. Instead, in some scenarios, the best long-term cost is generated by a slow rescheduling frequency. For example, in Example 3 with Random seed = 1, among all the rescheduling frequencies that range from every hour to every 12 hours, the every-12-hour frequency generates the best long-term cost and system nervousness. Moreover, according to the results, it is generally difficult to predict which rescheduling frequency will work the best. Therefore, from a practical perspective, if the scheduler were to adopt a periodic rescheduling policy, it would be worth exploring how to determine the best rescheduling frequency.

On the other hand, in most scenarios, except for Example 4 with Random seed = 4, our Bayesian dynamic scheduling method generates a set of long-term costs that are close to the optimal cost boundary generated by the periodically rescheduling policy. In some cases, such as Example 1 with Random seed = 5 and Example 2 with Random seed = 5, our method achieves a long-term cost that is even better than the cost generated by any rescheduling frequencies. This evidence implies that, in addition to the when-to-reschedule policy, the how-to-reschedule policy also has a significant impact on the long-term cost, especially under the constraint of incomplete look-ahead information. While the existing literature has already had a wide range of discussions over the when-to-reschedule policy in multipurpose batch processes, we point out that the how-to-reschedule policy may be equally important to the long-term cost.

Also, the performance of our Bayesian dynamic scheduling method is generally stable. See, for instance, Example 1 with all Random seeds, Example 2 with all Random seeds, Example 3 with Random seed = 1, 2, 3, and Example 4 with Random seed = 1, 2, 4, 5. However, in the remaining scenarios, such as Example 3 with Random seed = 4 and Example 4 with Random seed = 3, we observe a slight dispersion of the long-term cost. There exist two potential explanations for this instability. First, the number of episodes for \acrshort*{mc} simulation is not sufficient. As a result, the joint distribution $\hat{\mathsf{P}}$ that is learned by the \acrshort*{bn} is far from the true underlying distribution $\mathsf{P}^{*}$. Second, the commercial solver returns different schedules in different runs, which results in subsequent operations being affected. While exploring this issue is beyond the scope of this paper, we will investigate this issue in future work.

In a small number of scenarios, such as Example 3 with Random seed $= 5$ and Example 4 with Random seed $= 4$, our Bayesian dynamic scheduling method does not outperform the periodic rescheduling policy at its best frequency. Two factors may explain these cases. First, our framework learns the \acrshort*{bn} online from a finite Monte Carlo sample. By concentration inequalities, the learned \acrshort*{bn} has a non-zero probability of deviating from the true \acrshort*{bn} at any finite sample size. Although this probability decreases as the sample size grows, it is still likely that in a small fraction of runs, the deviation is large enough to misguide the rescheduling decision. On these particular realisations, a periodic strategy at a well-matched frequency may then outperform our method. Second, in all the unsatisfactory cases observed in our experiments, the periodic strategy outperforms our method only at high rescheduling frequencies, that is, short rescheduling periods. A high-frequency periodic strategy triggers a reschedule at fixed short intervals regardless of how few operations actually need to change, which lets it absorb disturbances reactively through constant re-optimisations. However, this aggressiveness comes at the cost of very high system nervousness. Our framework, by contrast, aims to balance long-term cost and the number of operation changes, as discussed earlier in Section \ref{sec:framework_overview}. In practice, plant operators often accept a small increase in long-term cost to avoid the operational disruption and computational overhead of constant rescheduling.

\begin{sidewaysfigure}
\begin{figure}[H]
    \includegraphics[width=1.0\textwidth]{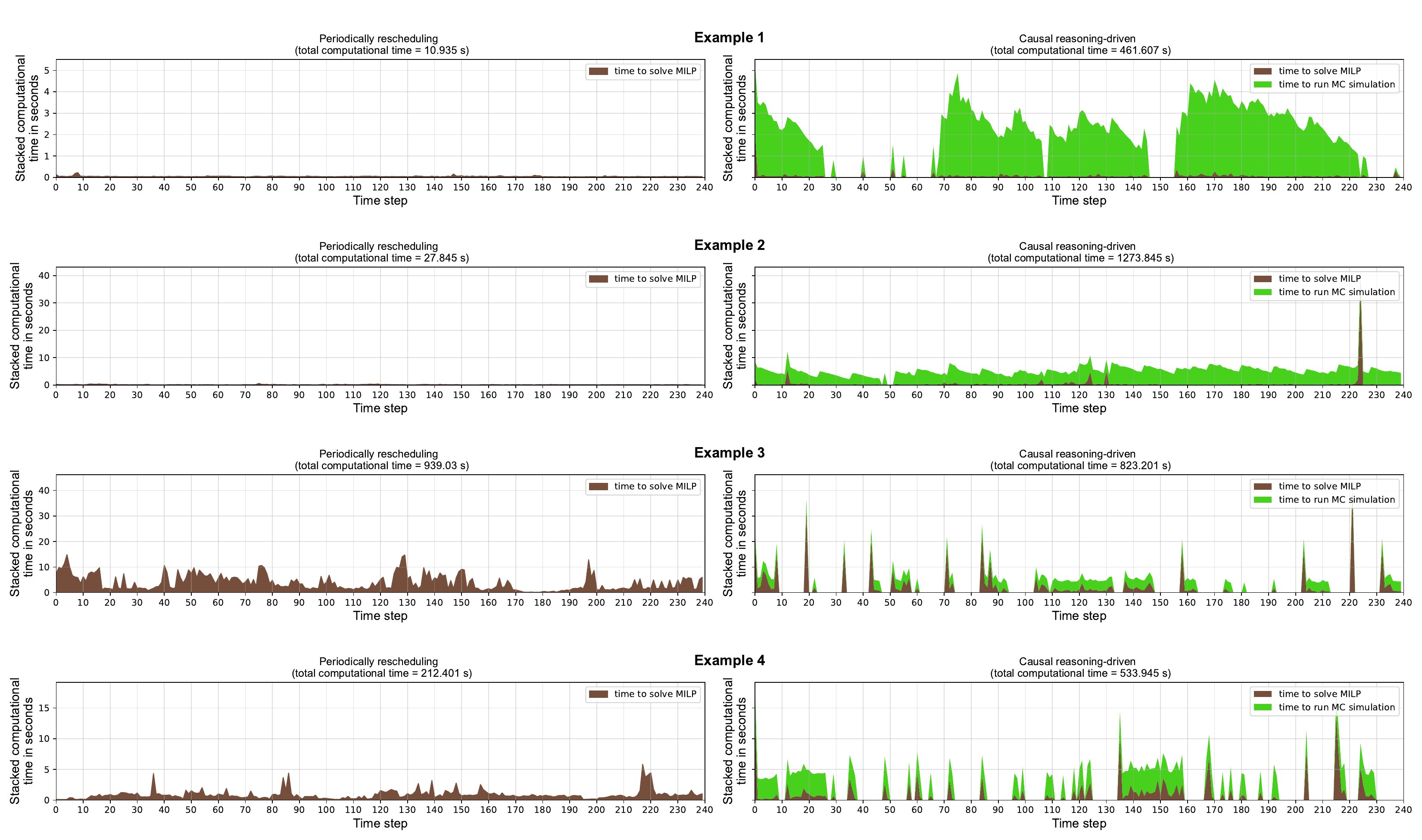}
    \centering
    \caption{
        Stack plots of computational times generated by the periodically rescheduling policy with a rescheduling frequency of every hour (left) and our Bayesian dynamic scheduling method (right) \\
        }
    \label{fig:computational_time}
\end{figure}
\end{sidewaysfigure}

To provide an intuition about computational times of our method, for each case, we take a snapshot for the computational time that is consumed at each time step. As a comparison, we also take snapshots for computational times that are consumed by the periodically rescheduling policy with a rescheduling frequency of every hour. The resulting stack plots are shown in Figure \ref{fig:computational_time}. Notably, we do not include the computational time that is consumed by the structure learning and the inference, because they are negligible in our cases. As we observe, in small-sized problems, such as Example 1 and Example 2, our method takes a significantly longer time than the periodically rescheduling policy. This is because most of the computational time is consumed on the \acrshort*{mc} simulation. However, for larger-sized problems, the computational time of our method is competitive with the periodically rescheduling policy. Especially, in the snapshot in Example 3, the cumulative computational time of our method is even shorter than the time consumed by the periodically rescheduling policy. This is because of two reasons. First, as mentioned in Section \ref{sec:methodology}, we use a warm start strategy to leverage the information from the old schedule and the probabilistic information provided by the \acrshort*{bn}. Second, we reduce the number of reschedulings. Note that rather than a comprehensive comparison, Figure \ref{fig:computational_time} is a snapshot of single runs from each case, which means that it does not imply statistical significance. While a comprehensive study on the topic of computational time is beyond the scope of this paper, we leave this topic to future work.

%% file: sections/section9.tex
\section{Conclusions} \label{sec:conclusions}
In this paper, we consider the dynamic scheduling problem of multipurpose batch processes under incomplete look-ahead information. We first describe the problem using the dynamic system formalism in Section \ref{sec:problem_formulation}. Then, we establish the detailed Bayesian dynamic scheduling method in Section \ref{sec:methodology}. Notably, we develop an algorithm, which is called impact propagation, to construct a set of random variables that can be used to describe how severe a set of operations is likely to be impacted by disturbances. The random variables that are constructed by this algorithm satisfy a set of conditional independence statements. This property allows us to develop a \acrshort*{bn} that describes the joint distribution over this set of random variables. Based on this theoretical foundation, we then develop the main body of our Bayesian dynamic scheduling method. Roughly speaking, the main body of our method is fine-controlling the rescheduling policy based on the posterior distribution that is inferred by the \acrshort*{bn}.

We test our method on four benchmark problems in multipurpose batch process scheduling. These problems are adopted from \cite{li_novel_2022}. The computational results show that, in most scenarios, our method performs well in long-term cost and system nervousness (Figure \ref{fig:performance_result}). We also discuss the theoretical aspects of the long-term cost in dynamic scheduling of multipurpose batch processes. That is, the nominal cost (the cost that is generated by assuming there is no disturbance in the system) and the oracle cost (the cost that is generated by assuming the decision maker has infinite look-ahead information). The analysis reveals that, the effect of disturbances may depend heavily on the specific realisation of disturbances. Also, for a dynamic scheduling method, under the constraint of incomplete look-ahead information, it is generally difficult to achieve the oracle cost.

From an industrial deployment perspective, we anticipate three concrete challenges. First, our framework requires a prior over the disturbance information variables. In practice, we need to estimate this prior from historical records such as machine breakdowns, processing time variations, yield losses, and demand fluctuations, which are often sparse, noisy, or inconsistently logged across plants. Constructing reliable estimates from such data is itself a non-trivial statistical problem. Second, the framework needs to integrate with existing \acrfull*{mes}, \acrfull*{erp}, and \acrshort*{iot} data pipelines. Real-time inference requires the look-ahead information to reach the scheduler with predictable latency. Also, the rescheduling output needs to flow back into these systems. We see this as primarily an engineering challenge. Third, the dominant component of the stepwise computational cost in our current algorithm is the Monte Carlo simulation used by the parameter learning module of the Bayesian Network. For industrial-scale schedules, this cost may become tight against real-time constraints. We will explore methods to reduce this cost in our future work, such as more sample-efficient parameter learning procedures and parallelisation strategies.

We also point out the following future directions. First, it is useful to understand the value of look-ahead information in dynamic scheduling of multipurpose batch processes. That is, by providing more look-ahead information, to what extent can a decision maker do better in the long-term cost? Second, our results reveal that the long-term cost depends not only on the when-to-reschedule policy, but also on the how-to-reschedule policy. Specifically, it is not always the best choice to generate an optimal schedule (in terms of the objective function in an \acrshort*{milp}) at each time step. As a result, how can we determine a better how-to-reschedule policy for multipurpose batch process scheduling? Finally, more warm start approaches are worth exploring to reduce the computational time of generating a new schedule.

%% file: sections/acknowledgments.tex
\section*{Acknowledgments}

Taicheng Zheng appreciates the financial support from UoM-CSC joint scholarship (202106440020). Dan Li appreciates the financial support from UoM-CSC joint scholarship (201908130170). Jie Li appreciates the financial support from Engineering and Physical Sciences Research Council (EP/V051008/1).

%% file: sections/appendixA.tex
\section{Dynamic system formulation} \label{appendix:transition_functions}

This appendix presents the dynamic system formulation summarised in Section \ref{sec:problem_formulation}, including the parameters, state variables, action variables, information variables, transition functions, cost function, and nervousness function. This formulation follows the dynamic system formalism, and is used to describe how the system evolves through time. Specifically, this formulation considers four types of commonly encountered disturbances: demand fluctuations, machine breakdowns, processing time variations, and yield losses. While the demand fluctuation falls into the category of exogenous disturbances, the remaining three fall into the category of endogenous disturbances. Also, we consider the operational conditions where (1) no changeover times are required to switch between different families of tasks, and (2) storage tanks have no capacity limit. Notably, although we make these specific assumptions, we intend to set a concrete basis for the later explanation of our methodology. After the entire methodological framework is presented, readers should find it not difficult to extend our method to other cases where additional types of disturbances (such as raw material shortages, operational cost variations) exist and different assumptions are required (such as different changeover policies, different storage constraints).

\subsubsection*{Parameters}
\begin{itemize}
    \item $\xi_{k,t}$. A continuous parameter that represents the supply upper bound of raw material $k$ at time point $t$. The supply is delivered periodically with constant upper bounds, and is visible to the scheduler across the next scheduling horizon.
    \item $\lambda_{k,t}$. A semi-continuous parameter that represents the \textit{baseline} demand for product $k$ at time point $t$. The baseline demand occurs periodically in constant quantities and is visible to the scheduler across the next scheduling horizon. We use baseline demand to model predictable, recurring orders from long-term contractors.
    \item $\nu_{k,t}$. A semi-continuous parameter that represents the \textit{intermittent} demand for product $k$ at time point $t$. The intermittent demand occurs non-periodically in variable quantities but is also visible to the scheduler across the next scheduling horizon. We use intermittent demand to model irregular, short-term orders from casual customers.
\end{itemize}

\subsubsection*{State variables}
\begin{itemize}
    \item $r_{i,j,t}$. A semi-integer variable that represents the elapsed time of a batch. If at the end of time interval $[t-1, t)$, a batch of task $i$ is executing or about to complete on machine $j$, then $r_{i,j,t}$ equals the number of time periods elapsed from the start time of this batch to the time point $t$. Otherwise no such batch is executing or about to complete, then $r_{i,j,t} = 0$.
    \item $b_{i,j,t}$. A semi-continuous variable that represents the size of a batch. If at the end of time interval $[t-1, t)$, a batch of task $i$ is executing or is about to complete on machine $j$, then $b_{i,j,t}$ equals the size of this batch. Otherwise no such batch is executing or is about to complete, then $b_{i,j,t} = 0$.
    \item $s_{k,t}$. A continuous variable that represents the inventory level of material $k$ during time interval $[t-1, t)$.
    \item $l_{k,t}$. A continuous variable that represents the backlog level (that is, the cumulative unmet demand) of product $k$ during time interval $[t-1, t)$.
\end{itemize}

\subsubsection*{Action variables}
\begin{itemize}
    \item $x_{i,j,t}$. A binary variable that represents the initiation of a batch. If at time point $t$, a batch of task $i$ is initiated on machine $j$, then $x_{i,j,t} = 1$. Otherwise no such batch is initiated, then $x_{i,j,t} = 0$.
    \item $y_{i,j,t}$. A semi-continuous variable that represents the input size of a batch. If at time point $t$, a batch of task $i$ is initiated on machine $j$, then $y_{i,j,t}$ equals the input size of this batch. Otherwise no such batch is initiated, then $y_{i,j,t} = 0$.
    \item $p_{k,t}$. A continuous variable that represents the quantity of raw material $k$ purchased at time point $t$.
    \item $q_{k,t}$. A continuous variable that represents the quantity of product $k$ shipped at time point $t$.
\end{itemize}

\subsubsection*{Information variables}
\begin{itemize}
    \item $f_{k,t}$. A semi-continuous variable that represents \textit{urgent} demand for product $k$ at time point $t$. The urgent demand occurs irregularly in random quantities and becomes visible to the scheduler only the length of certainty horizon prior to its shipment deadline.
    \item $u_{j,t}$. A binary variable that represents machine breakdowns. If during time interval $[t-1, t)$, machine $j$ is unavailable due to a breakdown event, then $u_{j,t} = 1$. Otherwise, $u_{j,t} = 0$.
    \item $v_{i,j,n,t}$. A positive continuous variable that represents multiplicative processing time variations. If at the end of time interval $[t-1, t)$, the elapsed time of the batch of task $i$ on machine $j$ equals $n$ time periods, then the varied processing time of this batch will equal $\lceil \tau_{i,j} \cdot v_{i,j,n,t} \rceil$ time periods (recall that the parameter $\tau_{i,j}$ represents the nominal processing time).
    \item $w_{i,j,n,t}$. A positive continuous variable that represents the multiplicative yield loss of a batch. If at the end of time interval $[t-1, t)$, the elapsed time of the batch of task $i$ on machine $j$ equals $n$ time periods, then the varied output of this batch will equal $b_{i,j,t} \cdot w_{i,j,n,t}$ units.
\end{itemize}

\subsubsection*{Transition functions}

We use a set of propositional logics to determine the value of next-step states. For convenience, we follow the convention that logical operators (such as $\land$ and $\lor$) have lower operational precedences than arithmetic operations (such as $+, -, \cdot, /$) and relational operators (such as $=, <, >, \in$). For example, the formula $a > b \land c < d$ should be interpreted as $(a > b) \land (c < d)$ rather than $a > (b \land c) < d$.

Also, the transition functions that will be presented below only represent one specific case of the operations of multipurpose batch processes. When extending to other cases (such as different operational conditions, additional disturbances), these transition functions need to be reformulated. However, reformulating these functions may not be trivial, especially when the logical interaction between variables are complexly entangled and the number of variable families is large. The challenge often comes from some corner cases (that is, some specific combinations of current-step variable values may result in a less apparent next-step state value) that may be overlooked by modellers. If such corner cases are overlooked, the reformulated transition functions will not correctly represent the real-world system dynamics. To address this challenge, we also propose an empirical procedure that can derive the transition function with the guarantee that all corner cases are inspected. This procedure relies on tracking the combination of variable values on a set of \textit{decision trees}. Since this procedure is not directly related to our methodological contributions, we present it in Appendix \ref{appendix:decision_tree}. Next, we present concretised transition functions.

\begin{subequations} \label{eq:appA_tf_rijt}
\begin{numcases}{r_{i,j,t+1} = }
    0       & $\text{if } \left( r_{i,j,t} = 0 \land x_{i,j,t} = 0 \right) \lor$ \label{eq:appA_tf_rijt_1} \\
            & $\phantom{\text{if }} \left( r_{i,j,t} > 0 \land u_{j,t} = 1 \land x_{i,j,t} = 0 \right) \lor$ \label{eq:appA_tf_rijt_2} \\
            & $\phantom{\text{if }} \left( r_{i,j,t} > 0 \land u_{j,t} = 0 \land \lceil \tau_{i,j} \cdot v_{i,j,(r_{i,j,t}),t} \rceil = r_{i,j,t} \land x_{i,j,t} = 0 \right)$ \label{eq:appA_tf_rijt_3} \\[2ex]
    1       & $\text{if } \left( r_{i,j,t} = 0 \land x_{i,j,t} = 1 \right) \lor$ \label{eq:appA_tf_rijt_4} \\
            & $\phantom{\text{if }} \left( r_{i,j,t} > 0 \land u_{j,t} = 1 \land x_{i,j,t} = 1 \right) \lor$ \label{eq:appA_tf_rijt_5} \\
            & $\phantom{\text{if }} \left( r_{i,j,t} > 0 \land u_{j,t} = 0 \land \lceil \tau_{i,j} \cdot v_{i,j,(r_{i,j,t}),t} \rceil = r_{i,j,t} \land x_{i,j,t} = 1 \right)$ \label{eq:appA_tf_rijt_6} \\[2ex]
    r_{i,j,t} + 1 & $\text{if } \left( r_{i,j,t} > 0 \land u_{j,t} = 0 \land \lceil \tau_{i,j} \cdot v_{i,j,(r_{i,j,t}),t} \rceil > r_{i,j,t} \right)$ \label{eq:appA_tf_rijt_7}
\end{numcases}
\end{subequations}

Equation \ref{eq:appA_tf_rijt} represents the transition function for the elapsed time $r_{i,j,t}$. Equation \ref{eq:appA_tf_rijt_1} represents the situation where no batch is executing during $[t-1, t)$ and no new batch is initiated at $t$. Therefore, the next-step elapsed time $r_{i,j,t+1}$ should equal $0$. Equation \ref{eq:appA_tf_rijt_2} represents the situation that an ongoing batch is disrupted by a machine breakdown during $[t-1, t)$, but no new batch is initiated at $t$. As a result, we also have $r_{i,j,t+1} = 0$. Equation \ref{eq:appA_tf_rijt_3} represents the situation where a batch is normally completing at the end of $[t-1, t)$ without being disrupted by machine breakdowns (recall that the term $\lceil \tau_{i,j} \cdot v_{i,j,(r_{i,j,t}),t} \rceil$ represents the varied processing time), but no new batch is initiated at $t$. Again, in this case we have $r_{i,j,t+1} = 0$. Equations \ref{eq:appA_tf_rijt_4}, \ref{eq:appA_tf_rijt_5}, \ref{eq:appA_tf_rijt_6} are similar to Equations \ref{eq:appA_tf_rijt_1}, \ref{eq:appA_tf_rijt_2}, \ref{eq:appA_tf_rijt_3}, respectively, with the only difference that a new batch is initiated at time point $t$. Therefore, in these cases the next-step elapsed time $r_{i,j,t+1}$ should equal one time period. Equation \ref{eq:appA_tf_rijt_7} represents the situation that a batch is normally executing during $[t-1, t)$ without being disrupted by machine breakdowns, and will remain ongoing at the end of $[t-1, t)$. As a result, the next-step elapsed time $r_{i,j,t+1}$ should equal $r_{i,j,t}$ added by one time period.

\begin{subequations} \label{eq:appA_tf_bijt}
\begin{numcases}{b_{i,j,t+1} = }
    0       & $\text{if } \left( r_{i,j,t} = 0 \land x_{i,j,t} = 0 \right) \lor$ \label{eq:appA_tf_bijt_1} \\
            & $\phantom{\text{if }} \left( r_{i,j,t} > 0 \land u_{j,t} = 1 \land x_{i,j,t} = 0 \right) \lor$ \label{eq:appA_tf_bijt_2} \\
            & $\phantom{\text{if }} \left( r_{i,j,t} > 0 \land u_{j,t} = 0 \land \lceil \tau_{i,j} \cdot v_{i,j,(r_{i,j,t}),t} \rceil = r_{i,j,t} \land x_{i,j,t} = 0 \right)$ \label{eq:appA_tf_bijt_3} \\[2ex]
    y_{i,j,t} & $\text{if } \left( r_{i,j,t} = 0 \land x_{i,j,t} = 1 \right) \lor$ \label{eq:appA_tf_bijt_4} \\
            & $\phantom{\text{if }} \left( r_{i,j,t} > 0 \land u_{j,t} = 1 \land x_{i,j,t} = 1 \right) \lor$ \label{eq:appA_tf_bijt_5} \\
            & $\phantom{\text{if }} \left( r_{i,j,t} > 0 \land u_{j,t} = 0 \land \lceil \tau_{i,j} \cdot v_{i,j,(r_{i,j,t}),t} \rceil = r_{i,j,t} \land x_{i,j,t} = 1 \right)$ \label{eq:appA_tf_bijt_6} \\[2ex]
    b_{i,j,t} & $\text{if } \left( r_{i,j,t} > 0 \land u_{j,t} = 0 \land \lceil \tau_{i,j} \cdot v_{i,j,(r_{i,j,t}),t} \rceil > r_{i,j,t} \right)$ \label{eq:appA_tf_bijt_7}
\end{numcases}
\end{subequations}

Equation \ref{eq:appA_tf_bijt} represents the transition function for the batch size $b_{i,j,t}$. The overall structure of Equation \ref{eq:appA_tf_bijt} is similar to that of Equation \ref{eq:appA_tf_rijt}. Specifically, Equations \ref{eq:appA_tf_bijt_1}--\ref{eq:appA_tf_bijt_3} correspond to Equations \ref{eq:appA_tf_rijt_1}--\ref{eq:appA_tf_rijt_3}, where no batch is executing during $[t, t+1)$. Equations \ref{eq:appA_tf_bijt_4}--\ref{eq:appA_tf_bijt_6} correspond to Equations \ref{eq:appA_tf_rijt_4}--\ref{eq:appA_tf_rijt_6}, where a new batch is initiated at time point $t$ and therefore the next-step batch size $b_{i,j,t+1}$ should equal the input size $y_{i,j,t}$. Equation \ref{eq:appA_tf_bijt_7} corresponds to Equation \ref{eq:appA_tf_rijt_7}, where $b_{i,j,t+1} = b_{i,j,t}$ because the current batch remains ongoing during $[t, t+1)$.

\begin{subequations} \label{eq:appA_tf_skt}
    \begin{align} \label{eq:appA_tf_skt_1}
    s_{k,t+1} = s_{k,t} + \sum_{j \in \mathbb{J}} \sum_{i \in \mathbb{I}_{j} \cap \mathbb{I}_{k}^{+}}  \rho_{i,k}  b_{i,j,t}^{+} + \sum_{j \in \mathbb{J}} \sum_{i \in \mathbb{I}_{j} \cap \mathbb{I}_{k}^{-}} \rho_{i,k} y_{i,j,t} + p_{k,t} - q_{k,t}
    \end{align}

    \begin{numcases}{b^{+}_{i,j,t} = }
        0   & $\text{if } \left( r_{i,j,t} = 0 \right) \lor$ \label{eq:appA_tf_skt_2} \\
            & $\phantom{\text{if }} \left( r_{i,j,t} > 0 \land u_{j,t} = 1 \right) \lor$ \label{eq:appA_tf_skt_3} \\
            & $\phantom{\text{if }} \left( r_{i,j,t} > 0 \land u_{j,t} = 0 \land \lceil \tau_{i,j} \cdot v_{i,j,(r_{i,j,t}),t} \rceil > r_{i,j,t} \right) $ \label{eq:appA_tf_skt_4} \\[2ex]
        b_{i,j,t} \cdot w_{i,j,(r_{i,j,t}),t} & $\text{if } \left( r_{i,j,t} > 0 \land u_{j,t} = 0 \land \lceil \tau_{i,j} \cdot v_{i,j,(r_{i,j,t}),t} \rceil = r_{i,j,t} \right)$ \label{eq:appA_tf_skt_5}
    \end{numcases}
\end{subequations}

Equation \ref{eq:appA_tf_skt} represents the transition function for the inventory level $s_{k,t}$. As presented in Equation \ref{eq:appA_tf_skt_1}, the next-step inventory level $s_{k,t+1}$ depends on (1) the current inventory level $s_{k,t}$, (2) the quantity of materials added from upstream batch completions $\sum_{j \in \mathbb{J}} \sum_{i \in \mathbb{I}_{j} \cap \mathbb{I}_{k}^{+}}  \rho_{i,k}  b_{i,j,t}^{+}$, (3) the quantity of material consumed by downstream batch initiations $\sum_{j \in \mathbb{J}} \sum_{i \in \mathbb{I}_{j} \cap \mathbb{I}_{k}^{-}} \rho_{i,k} y_{i,j,t}$, (4) the quantity of purchased raw materials $p_{k,t}$, and (5) the quantity of shipped products $q_{k,t}$. By default, we set $p_{k,t} = 0$ for $k \notin \mathbb{K}^{\mathrm{R}}$ and set $q_{k,t} = 0$ for $k \notin \mathbb{K}^{\mathrm{P}}$. In the term (2), the auxiliary variable $b_{i,j,t}^{+}$ represents the yield that is output by the batch of task $i$ on machine $j$ at the end of $[t-1, t)$. Equations \ref{eq:appA_tf_skt_2}--\ref{eq:appA_tf_skt_5} determine the value of $b_{i,j,t}^{+}$. That is, if the elapsed time of the batch equals $0$ (Equation \ref{eq:appA_tf_skt_2}), or the ongoing batch is disrupted by a machine breakdown (Equation \ref{eq:appA_tf_skt_3}), or the batch is not completing (Equation \ref{eq:appA_tf_skt_4}), then $b_{i,j,t}^{+} = 0$. Otherwise, if the batch is normally completing at the end of $[t-1, t)$, then $b_{i,j,t}^{+}$ equals the original batch size $b_{i,j,t}$ multiplied by the yield loss variable $w_{i,j,(r_{i,j,t}),t}$.

\begin{align} \label{eq:appA_tf_lkt}
    l_{k,t+1} = l_{k,t} + \lambda_{k,t} + \nu_{k,t} + f_{k,t} - q_{k,t}
\end{align}

Equation \ref{eq:appA_tf_lkt} represents the transition function for the backlog level $l_{k,t}$. The next-step backlog level $l_{k,t+1}$ equals the current backlog level $l_{k,t}$, added by the baseline demand $\lambda_{k,t}$, the intermittent demand $\nu_{k,t}$, and the urgent demand $f_{k,t}$, and subtracted by the quantity of shipped material $q_{k,t}$.

\subsubsection*{Cost function}

\begin{equation} \label{eq:appA_cost_func_ct}
\begin{split}
    c_{t} &= \textsc{CostFunc} \left( \boldsymbol{s}_{t}, \boldsymbol{a}_{t}, \boldsymbol{i}_{t} \right) \\
    &= \sum_{j \in \mathbb{J}} \sum_{i \in \mathbb{I}_{j}} c^{x}_{i,j} x_{i,j,t}
    + \sum_{k \in \mathbb{K}} c^{s}_{k} s_{k,t}
    + \sum_{k \in \mathbb{K}^{\mathrm{P}}} c^{l}_{k} l_{k,t}
\end{split}
\end{equation}

Equation \ref{eq:appA_cost_func_ct} represents the cost function. The total cost consists of (1) the setup cost for initiating batches at time point $t$, (2) the inventory cost, and (3) the backlog cost. The parameters $c^{x}_{i,j}$, $c^{s}_{k}$, and $c^{l}_{k}$ represent the fixed cost of initiating a batch, the cost of holding per unit of material per time period, and the cost of having per unit of backlog per time period, respectively.

\subsubsection*{Nervousness function}

\begin{equation} \label{eq:appA_nerv_func_dt}
\begin{split}
    d_{t} &= \textsc{NervFunc} \left( \boldsymbol{a}_{t: t+h_{t} \mid t}, \boldsymbol{a}_{t-1: t - 1 + h_{t-1} \mid t-1} \right) \\
    &= \sum_{t' = t }^{t + h_{t}} \sum_{j \in \mathbb{J}} \sum_{i \in \mathbb{I}_{j}} \left| x_{i,j,t' \mid t} - x_{i,j,t' \mid t - 1} \right|
\end{split}
\end{equation}

Equation \ref{eq:appA_nerv_func_dt} represents the nervousness function. This function evaluates the difference between the current-step action sequence $\boldsymbol{a}_{t: t+h_{t} \mid t}$ (which represents the ``new schedule'') and the previous-step action sequence $\boldsymbol{a}_{t-1: t - 1 + h_{t-1} \mid t-1}$ (which represents the ``original schedule''). To interpret, let us examine the formula $\left| x_{i,j,t' \mid t} - x_{i,j,t' \mid t - 1} \right|$. If either $x_{i,j,t' \mid t} = x_{i,j,t' \mid t - 1} = 0$ (which means that in either the original schedule or the new schedule, no batch of task $i$ is initiated on machine $j$ at time point $t$) or $x_{i,j,t' \mid t} = x_{i,j,t' \mid t - 1} = 1$ (which means that the batch exists in both the original schedule and the new schedule), then we have $\left| x_{i,j,t' \mid t} - x_{i,j,t' \mid t - 1} \right| = 0$. On the other hand, if exactly one of $x_{i,j,t' \mid t}$ and $x_{i,j,t' \mid t-1}$ equals $1$ while the remaining one equals $0$, then we have $\left| x_{i,j,t' \mid t} - x_{i,j,t' \mid t - 1} \right| = 1$. As a result, intuitively, Equation \ref{eq:appA_nerv_func_dt} evaluates the \textit{bidirectional} difference in batch initiations. Note that we only check the difference for the interval $[t+1, t+h_{t}]$, which is overlapped by the original schedule and the new schedule. Note also that when no rescheduling is triggered at time point $t$, the value of the nervousness function equals zero.

%% file: sections/appendixB.tex
\section{Illustration of the empirical procedure for constructing logic-based transition rules} \label{appendix:decision_tree}

This appendix presents the empirical procedure, mentioned in Appendix \ref{appendix:transition_functions}, for constructing logic-based transition rules using decision trees. We first motivate the procedure, then present the empirical procedure, and finally use a worked example to illustrate how to apply this procedure to a specific state variable.

When using the dynamic system formalism to describe a batch process, one of the main challenges exist in how to encode the feasibility logic between the action and the state variables in the transition function \textsc{TransFunc}. For example, \textsc{TransFunc} should not allow the system to initiate a new operation on a machine when another operation is ongoing. Also, when the downstream storage unit reaches full capacity, upstream machines cannot complete their operations and transfer produced materials. It becomes even more complex to incorporate these feasibility constraints into \textsc{TransFunc} when there exist endogenous disturbances in the system.

In previous works, a popular approach to address this challenge is to develop a \textit{state-space} model \cite{subramanian_state-space_2012, gupta_general_2017, risbeck_unification_2019}. In a state-space model, $\textsc{TransFunc}$ is modelled as a linear combination of the system variables. That is, the general form of the state-space model is given by
\begin{align} \label{eq:appB_general_state_space_model}
    \boldsymbol{s}_{t+1} = A \boldsymbol{s}_{t} + B \boldsymbol{a}_{t} + C \boldsymbol{i}_{t},
\end{align}
where the matrices $A$, $B$, and $C$ are state-space matrices that incorporate the system dynamics, which include these feasibility constraints. One significant advantage of the state-space model is that it is compatible with existing \acrshort*{milp} formulations. That is, if a scheduler has already developed an \acrshort*{milp} formulation for generating new schedules, the $\textsc{TransFunc}$ can be easily derived from this formulation. This is often achieved by the following two-step procedure. First, solve the \acrshort*{milp} problem at the $t$-th time step (in the context of dynamic systems) and obtain the solution. Then, use the values of decision variables at the $(t+1)$-th time step in the \acrshort*{milp} solution as the values for state variables $\boldsymbol{s}_{t+1}$. Following this procedure, essentially, the state-space matrices correspond to a set of sub-matrices in the left-hand side coefficient matrix of the \acrshort*{milp} formulation.

However, based on our experience, there exist several limitations in using a state-space model to describe the system dynamics of a batch process. On the one hand, it can become time consuming to compute \textsc{TransFunc} many times because it requires solving an \acrshort*{milp} problem from scratch. Specifically, when the complete solution is not used for subsequent executions, this may result in a waste of computational resources, which often leads to a non-negligible computational overhead. On the other hand, the matrix representation used in the state-space model often does not appear in a modular form in complex situations. That is, when a new type of disturbance needs to be considered, the entire set of state-space matrices are often required to be completely redesigned. This redesigning process can be painstaking for human practitioners and can incur additional development costs in practice.

\begin{figure}[H]
\includegraphics[width=1.0\textwidth]{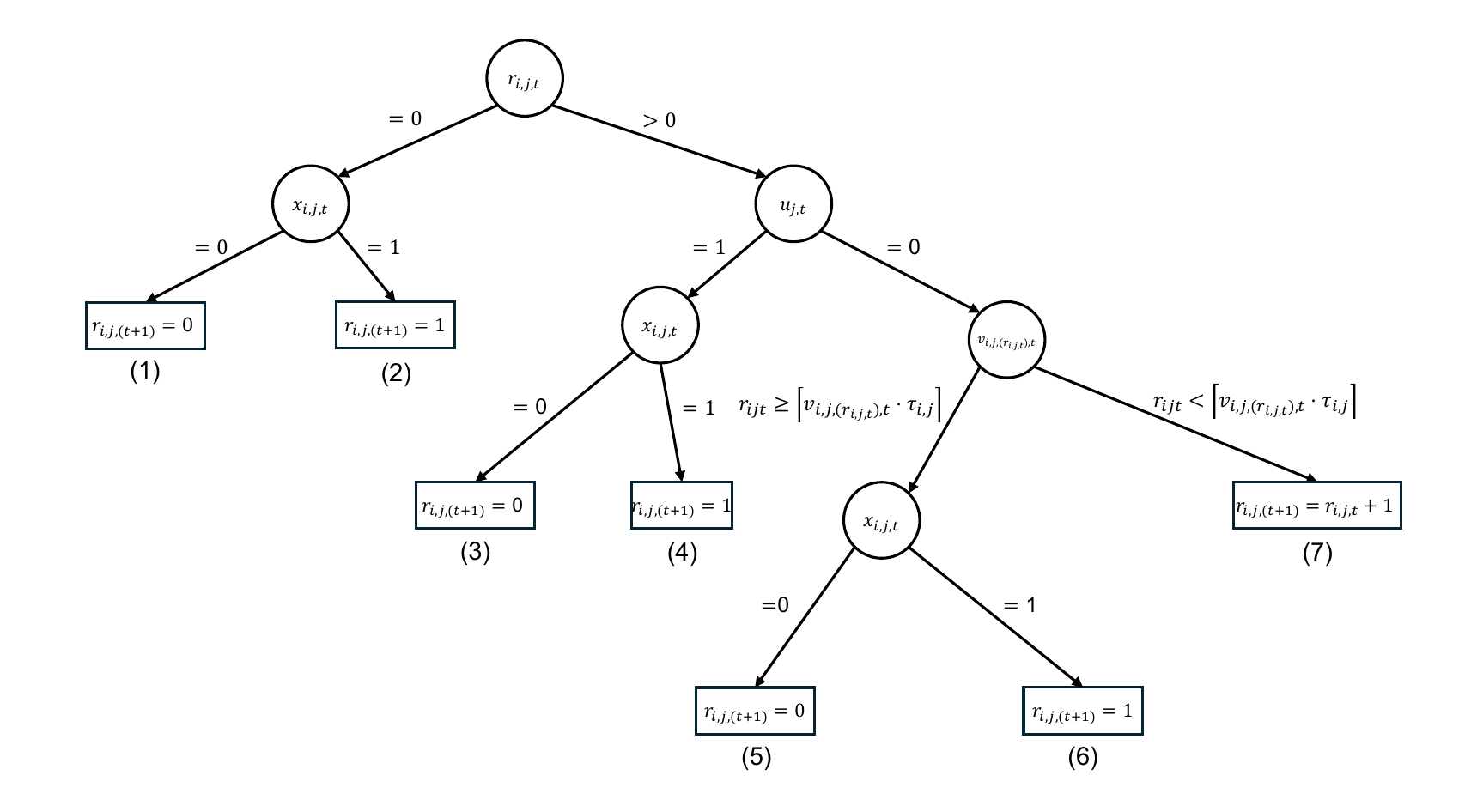}
\centering
\caption{
    Decision tree for determining the value of state variable $r_{i,j,t+1}$. \\
    In this decision tree, each non-leaf node denotes a system variable from the previous time step that may influence the value of $r_{i,j,t+1}$. Each arc outgoing from a non-leaf node represents a set of possible values for the corresponding non-leaf node variable. Each leaf node represents a set of possible values that the state variable $r_{i,j,t+1}$ can take, determined by travelling the decision tree from the root node to the leaf node.
}
\label{fig:appB_decision_tree_rijt}
\end{figure}

To address these limitations, we propose a rule-based method that uses a set of decision trees to describe $\textsc{TransFunc}$. In this method, each state variable $s_{t+1} \in \boldsymbol{s}_{t+1}$ is associated with a decision tree that describes its dedicated transition rules. Each decision tree is composed as follows. A non-leaf node represents a system variable $x_{t}$ (from the previous time step) that may influence the value of $s_{t+1}$. An arc that is outgoing from a non-leaf node represents a set of possible values for the corresponding non-leaf node variable. A leaf node represents a set of possible values that the state variable $s_{t+1}$ can take, which is determined by following the logic path from the root node to the leaf node. As an example, Figure \ref{fig:appB_decision_tree_rijt} presents the decision tree for determining the value of the state variable $r_{i,j,t+1}$, which represents the elapsed time of operations.

In the remainder of this appendix section, we first describe the empirical procedure for designing a decision tree. Next, we use the state variable $r_{i,j,t+1}$ as an example to demonstrate step by step how to design its corresponding decision tree by following the empirical procedure. After that, we explain how to encode the decision tree as a set of predicate logic formulas and how to implement these formulas in general-purpose programming languages.

\subsubsection*{Empirical procedure for designing the decision tree}
To clearly describe the empirical procedure, we need additional notations. We use $x \xrightarrow{y} *$ to represent an arc, where $x$ is the parent node, $y$ is the attribute of the arc, and the child node has not been specified but the arc is known to exist. We denote a sequence of variables by $S = \langle a_1, a_2, \cdots, a_n \rangle$. We use $\frown$ to denote the concatenation operation, that is, $S \frown \langle a_{n+1} \rangle = \langle a_1, \cdots, a_{n}, a_{n+1} \rangle$. Also, we will use $\triangleleft S$ to represent the sequence $S$ with the last element removed, that is, $\triangleleft S = \langle a_1, \cdots, a_{n-1} \rangle$.

The empirical procedure for constructing a decision tree is stated as follows.
\begin{itemize}
    \item Step 0 (Initialisation). Specify the state variable $s_{t+1}$ for which the decision tree is to be determined. Identify the set of system variables $\mathbb{X}_{t}$ at the $t$-th time step that may affect the value of $s_{t+1}$. Initialise an empty sequence $S$ to denote the variables that are currently being explored. Proceed to Step 1.
    \item Step 1 (Partitioning). Select a variable from $\mathbb{X}_{t}$ that may affect the value of $s_{t+1}$ and label it as $x_{t}$. Partition the possible values of $x_{t}$ as $\Pi(x_{t}) = \{ \Pi^{(1)}(x_{t}), \cdots, \Pi^{(n)}(x_{t}) \}$. For each subset of possible values $\Pi^{(i)}(x_t) \in \Pi(x_t)$, add the arc $x_{t} \xrightarrow{\Pi^{(i)}(x_t)} *$ to the decision tree. Update $\mathbb{X}_{t} \coloneqq \mathbb{X}_{t} \setminus \{ x_t \}$ and $S \coloneqq S \frown \langle x_{t} \rangle$. Go to Step 2.
    \item Step 2 (Branching). Choose one of the subsets $\Pi^{(i)}(x_t)$ from $\Pi(x_t)$. Update $\Pi(x_{t}) \coloneqq \Pi(x_{t}) \setminus \Pi^{(i)} (x_t)$. Check if the remaining variables in $\mathbb{X}_{t}$ may affect the value of $s_{t+1}$. If they do, then go back to Step 1; otherwise, proceed to Step 3.
    \item Step 3 (Pruning). Mark the unspecified child node $*$ in the arc $x_{t} \xrightarrow{\Pi^{(i)}(x_t)} *$ with the value of $s_{t+1}$ in this situation. Proceed to Step 4.
    \item Step 4 (Backtracking). Check if $\Pi(x_t) = \varnothing$. If this is true, which means that all subsets in $\Pi(x_t)$ have been explored, then repeat the following process until $\Pi(x_t) \neq \varnothing$: update $\mathbb{X}_{t} \coloneqq \mathbb{X}_{t} \cup \{ x_{t} \}$, update $S \coloneqq \triangleleft S$, and relabel $x_{t}$ as the last element in $S$. If this process stops with $\Pi(x_t)$ being nonempty, go back to Step 2. Otherwise, which means that all possible situations have been explored, terminate the procedure. If $\Pi(x_t) \neq \varnothing$, go back to Step 2.
\end{itemize}

\subsubsection*{Illustrative example}
Here, we take $r_{i,j,t+1}$ as an example to illustrate the empirical procedure and how the decision tree can be interpreted by a set of predicate logic formulas. Although this step-by-step illustration may look daunting at first, unfortunately, it is necessary to go through these steps to avoid missing any corner cases. Once readers have become familiar with the idea of this procedure, it will become not difficult to apply it to a new state variable. When going through the illustration, it is helpful to also refer to explanations that are presented in Figure \ref{fig:appB_decision_tree_rijt} and Figure \ref{fig:appB_machine_status_rijt}.

The empirical procedure for designing the decision tree that encodes the transition rules for the state variables $r_{i,j,t+1}$, where $j \in \mathbb{J}$, $i \in \mathbb{I}_{j}$, and $ t \in \{t_0, \cdots, t_N-1 \}$ is illustrated as follows.

\begin{figure}[h]
\includegraphics[width=1.0\textwidth]{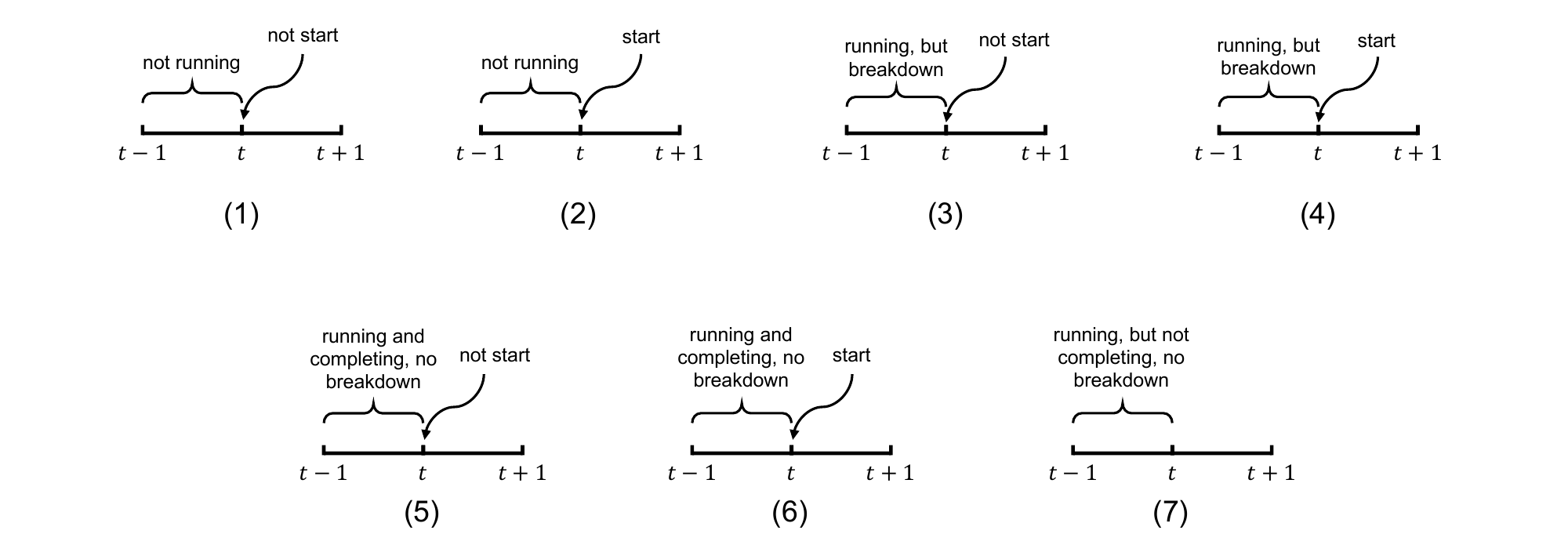}
\centering
\caption{
    Illustration of different operation statuses corresponding to the leaf nodes of the decision tree shown in Figure \ref{fig:appB_decision_tree_rijt}
}
\label{fig:appB_machine_status_rijt}
\end{figure}

\begin{enumerate}
    \item (Initialisation) We first identify the system variables that may affect the value of $r_{i,j,t+1}$, which represents the elapsed time of an operation by the end of time interval $[t, t+1)$. It is clear that it depends on $r_{i,j,t}$, the elapsed time during the last time step. Also, it may depend on whether the system initiates a new operation ($x_{i,j,t}$) at time point $t$. Regarding disturbances, if the processing time extends ($v_{i,j,(r_{i,j,t}),t}$), or if the machine breaks down ($u_{j,t}$) during $[t-1, t)$, then $r_{i,j,t+1}$ may also be affected. As a result, we initialise with $\mathbb{X}_{t} \coloneqq \{ r_{i,j,t}, x_{i,j,t}, v_{i,j,(r_{i,j,t}),t}, u_{j,t} \}$, and an empty sequence $S$. Proceed to Partitioning.

    \item (Partitioning) Without loss of generality, we select $r_{i,j,t}$ as the first variable to partition. Label $r_{i,j,t}$ as $x_t$. Because $r_{i,j,t}$ is a nonnegative integer variable, we partition it as $\Pi(r_{i,j,t}) = \bigl\{ \Pi^{(1)}(r_{i,j,t}), \Pi^{(2)}(r_{i,j,t}) \bigr\} = \bigl\{ \{ r_{i,j,t} = 0 \}, \{r_{i,j,t} > 0 \} \bigr\}$, out of which represents that the operation is not running and is running during $[t-1, t)$, respectively. According to this partition, we add arcs $r_{i,j,t} \xrightarrow{ \{ r_{i,j,t} = 0 \}} *$ and $r_{i,j,t} \xrightarrow{ \{ r_{i,j,t} > 0 \}} *$ to the decision tree. Update $\mathbb{X}_{t} \coloneqq \mathbb{X}_{t} \setminus \{ x_{t} \} = \{ x_{i,j,t}, v_{i,j,(r_{i,j,t}),t}, u_{j,t} \}$ and $S \coloneqq S \frown \langle x_{t} \rangle = \langle r_{i,j,t} \rangle$. Proceed to Branching.

    \item (Branching) We first branch on $\{ r_{i,j,t} = 0\}$. Update $\Pi(r_{i,j,t}) \coloneqq \bigl\{ \{r_{i,j,t} > 0 \} \bigr\}$. In this case, there is no operation that is running on the machine during $[t-1, t)$. Therefore, the processing time variation and machine breakdown information during this interval will not affect the elapsed time by the end of $[t, t+1)$. However, it is possible that the system initiates an operation at time point $t$. Therefore, when $\{r_{i,j,t} = 0 \}$, we still need to consider the value of $x_{i,j,t}$. Go back to Partitioning.

    \item (Partitioning) Next, we choose $x_{i,j,t}$ to partition. Label $x_{i,j,t}$ as $x_t$. Because $x_{i,j,t}$ is a binary variable that indicates whether an operation is initiated at time point $t$, we partition its possible values by $\Pi(x_{i,j,t}) = \bigl\{ \{ x_{i,j,t} = 0 \}, \{x_{i,j,t} = 1 \} \bigr\}$. Add arcs $x_{i,j,t} \xrightarrow{ \{ x_{i,j,t} = 0 \}} *$ and $x_{i,j,t} \xrightarrow{ \{ x_{i,j,t} = 1 \}} *$ to the decision tree. Update $\mathbb{X}_{t} \coloneqq \mathbb{X}_{t} \setminus \{ x_{t} \} = \{v_{i,j,(r_{i,j,t}),t}, u_{j,t} \}$ and $S \coloneqq S \frown \langle x_{t} \rangle =  \langle r_{i,j,t}, x_{i,j,t} \rangle$. Proceed to Branching.

    \item (Branching) Here, we choose to first branch on the case when $\{ x_{i,j,t} = 0 \}$. Update $\Pi(x_{i,j,t}) \coloneqq \bigl\{ \{x_{i,j,t} = 1 \} \bigr\}$. In this case, there is no operation that is running during $[t-1, t)$. Also, there is no operation that is initiated at the $t$-th time step. Therefore, there will be no operation that is running during $[t, t+1)$. As a result, there are no further variables in $\mathbb{X}_{t}$ that may affect the value of $r_{i,j,t+1}$. Proceed to Pruning.

    \item (Pruning) We mark the arc $x_{i,j,t} \xrightarrow{ \{x_{i,j,t} = 0 \} } *$ as $x_{i,j,t} \xrightarrow{ \{x_{i,j,t} = 0 \} } \{ r_{i,j,t} = 0 \}$. Proceed to Backtracking.

    \item (Backtracking) Recall that until this step, the label of $x_t$ is on $x_{i,j,t}$. Because $\Pi(x_{t}) = \bigl\{ \{x_{i,j,t} = 1 \} \bigr\} \neq \varnothing$, return to Branching.

    \item (Branching) This time, we branch on $\{ x_{i,j,t} = 1 \}$. Update $\Pi(x_{i,j,t}) \coloneqq \varnothing$. In this case, there is no operation that is running during $[t-1, t)$. However, the system initiates an operation at time point $t$. Regardless of whether this newly initiated operation will complete by the end of $[t, t+1)$ or not, we have $r_{i,j,t+1} = 1$. There are no other variables in $\mathbb{X}_{t}$ that may affect the value of $r_{i,j,t+1}$. Therefore, we proceed to Pruning.

    \item (Pruning) Again, we mark the arc $x_{i,j,t} \xrightarrow{ \{x_{i,j,t} = 0 \} } *$ as $x_{i,j,t} \xrightarrow{ \{x_{i,j,t} = 0 \} } \{ r_{i,j,t+1} = 1 \}$. Proceed to Backtracking.

    \item (Backtracking) Recall that $x_{t} = x_{i,j,t}$. Because $\Pi(x_{t}) = \varnothing$, we update $\mathbb{X}_{t} \coloneqq \mathbb{X}_{t} \cup \{ x_{t} \} = \{v_{i,j,(r_{i,j,t}),t}, u_{j,t}, x_{i,j,t} \}$, $S \coloneqq \triangleleft S = \langle r_{i,j,t} \rangle$. Assign $x_{t} \coloneqq r_{i,j,t}$. In this case, we know that $\Pi(r_{i,j,t}) = \bigl\{ \{r_{i,j,t} > 0 \} \bigr\}$, which is not empty. Therefore, we return to Branching.

    \item (Branching) At this point, the only choice is to branch on $\{ r_{i,j,t} > 0 \}$. Update $\Pi(r_{i,j,t}) \coloneqq \varnothing$. When $r_{i,j,t} > 0$, the situation becomes more complex. That is, the running operation may be disrupted by a machine breakdown, the processing time may vary, and another operation may start at time point $t$. Therefore, we proceed to Partitioning.

    \item (Partitioning) We select $u_{j,t}$ to partition. Label $u_{j,t}$ as $x_{t}$. Because $u_{j,t}$ is a binary variable that represents machine breakdown, we partition it as $\Pi(u_{j,t}) = \bigl\{ \{ u_{j,t} = 0 \}, \{u_{j,t} = 1 \} \bigr\}$. Add arcs $u_{j,t} \xrightarrow{\{u_{j,t} = 0 \}} *$ and $u_{j,t} \xrightarrow{\{u_{j,t} = 1 \}} *$ to the decision tree. Update $\mathbb{X}_{t} \coloneqq \mathbb{X}_{t} \setminus \{ x_{t} \} = \{x_{i,j,t}, v_{i,j,(r_{i,j,t}),t} \}$ and $S \coloneqq S \frown \langle x_{t} \rangle = \langle r_{i,j,t}, u_{j,t} \rangle$. Proceed to Branching.

    \item (Branching) We branch on $\{ u_{j,t} = 1 \}$. Update $\Pi(u_{j,t}) \coloneqq \bigl\{ \{ u_{j,t} = 0 \} \bigr\}$. In this case, there is a machine breakdown event during $[t-1, t)$. However, this information alone does not ensure that $r_{i,j,t+1} = 0$ because it is still possible that an operation is initiated at time point $t$. Therefore, we go back to Partitioning.

    \item (Partitioning) We partition on $x_{i,j,t}$ again. Label $x_{i,j,t}$ as $x_{t}$. Partition $\Pi(x_{i,j,t}) = \bigl\{ \{ x_{i,j,t} = 0 \}, \{x_{i,j,t} = 1 \} \bigr\}$, and add arcs $x_{i,j,t} \xrightarrow{ \{ x_{i,j,t} = 0 \}} *$ and $x_{i,j,t} \xrightarrow{ \{ x_{i,j,t} = 1 \}} *$ to the decision tree. Update $\mathbb{X}_{t} \coloneqq \mathbb{X}_{t} \setminus \{ x_{t} \} = \{ v_{i,j,(r_{i,j,t},t)} \}$, $S \coloneqq S \frown \langle x_{t} \rangle = \langle r_{i,j,t}, u_{j,t}, x_{i,j,t} \rangle $. Proceed to Branching.

    \item (Branching) We branch on $\{ x_{i,j,t} = 0 \}$. Update $\Pi(x_{i,j,t}) \coloneqq \bigl\{ \{ x_{i,j,t} = 1 \} \bigr\}$. In this case, there is an operation that is running during $[t-1, t)$, but it is disrupted by a machine breakdown, and no operation is initiated at $t$. No further variables in $\mathbb{X}_{t}$ will affect $r_{i,j,t+1}$. Proceed to pruning.

    \item (Pruning) Mark the arc $x_{i,j,t} \xrightarrow{ \{x_{i,j,t} = 0 \} } *$ as $x_{i,j,t} \xrightarrow{ \{x_{i,j,t} = 0 \} } \{r_{i,j,t} = 0 \}$. Proceed to Backtracking.

    \item (Backtracking) Because $\Pi(x_t) = \Pi(x_{i,j,t}) \neq \varnothing$, go back to Branching.

    \item (Branching) Branch on $\{ x_{i,j,t} = 1 \}$. Update $\Pi(x_{i,j,t}) \coloneqq \varnothing$. This case represents the scenario that an operation was running during $[t-1, t)$, but it was disrupted by a machine breakdown event, and a new operation is initiated at time point $t$. In this case, we have $r_{i,j,t+1} = 1$. No further variables in $\mathbb{X}_{t}$ will affect its value. Proceed to Pruning.

    \item (Pruning) Mark the arc $x_{i,j,t} \xrightarrow{ \{x_{i,j,t} = 1 \} } *$ as $x_{i,j,t} \xrightarrow{ \{x_{i,j,t} = 1 \} } \{r_{i,j,t} = 1 \}$. Proceed to Backtracking.

    \item (Backtracking) Because $\Pi(x_{t}) = \Pi(x_{i,j,t}) = \varnothing$, update $\mathbb{X}_{t} \coloneqq \mathbb{X} \cup \{ x_{t} \} = \{ x_{i,j,t}, v_{i,j,(r_{i,j,t}),t} \}$, $S \coloneqq \triangleleft S = \langle r_{i,j,t}, u_{j,t} \rangle$, relabel $x_{t}$ as $u_{j,t}$. Because $\Pi(u_{j,t}) \neq \varnothing$, we return to Branching.

    \item (Branching) We branch on $\{ u_{j,t} = 0 \}$. Update $\Pi(u_{j,t}) \coloneqq \varnothing$. In this case, the operation is running during $[t-1, t)$ without disruption. Therefore, we need to check if the operation will complete normally by the end of $[t-1, t)$. Go back to Partitioning.

    \item (Partitioning) We partition on $v_{i,j,(r_{i,j,t}),t}$. Label $v_{i,j,(r_{i,j,t}),t}$ as $x_{t}$. Because it denotes the multiplicative factor of varied processing time, we need to compare it with the elapsed time. If they are equal, the operation will complete normally. Otherwise, the operation will continue. Therefore, we partition the possible values of $v_{i,j,(r_{i,j,t}),t}$ as $\Pi(v_{i,j,(r_{i,j,t}),t}) \coloneqq \bigl\{ \{ \lceil \tau_{i, j} \cdot v_{i,j,(r_{i,j,t}),t} \rceil \leq r_{i,j,t} \}, \{ \lceil \tau_{i, j} \cdot v_{i,j,(r_{i,j,t}),t} \rceil > r_{i,j,t} \} \bigr\}$. Add arcs $v_{i,j,(r_{i,j,t}),t} \xrightarrow{ \{ \lceil \tau_{i, j} \cdot v_{i,j,(r_{i,j,t}),t} \rceil \leq r_{i,j,t} \}} *$ and $v_{i,j,(r_{i,j,t}),t} \xrightarrow{ \{ \lceil \tau_{i, j} \cdot v_{i,j,(r_{i,j,t}),t} \rceil > r_{i,j,t} \}} *$ to the decision tree. Update $\mathbb{X}_{t} \coloneqq \mathbb{X}_{t} \setminus \{ x_{t} \} = \{ x_{i,j,t} \}  $ and $S \coloneqq S \frown \langle x_{t} \rangle = \langle r_{i,j,t}, u_{j,t}, v_{i,j,(r_{i,j,t}),t} \rangle $. Proceed to Branching.

    \item (Branching) We first consider the case $\Pi^{(1)}(v_{i,j,(r_{i,j,t}),t}) = \{ \lceil \tau_{i, j} \cdot v_{i,j,(r_{i,j,t}),t} \rceil \leq r_{i,j,t} \} $. In this case, the operation completes by the end of $[t-1, t)$. The only remaining variable that may affect $r_{i,j,t+1}$ is $x_{i,j,t}$, that is, whether the system initiates another operation at time point $t$. Go back to Partitioning.

    \item (Partitioning) Similarly, we partition on $x_{i,j,t}$ and label it as $x_{t}$. Assign $\Pi(x_{i,j,t}) \coloneqq \bigl\{ \{ x_{i,j,t} = 0 \}, \{x_{i,j,t} = 1 \} \bigr\}$, and add arcs $x_{i,j,t} \xrightarrow{ \{ x_{i,j,t} = 0 \}} *$ and $x_{i,j,t} \xrightarrow{ \{ x_{i,j,t} = 1 \}} *$. Update $\mathbb{X}_{t} \coloneqq \mathbb{X}_{t} \setminus \{ x_{t} \} = \varnothing$ and $S \coloneqq S \frown \langle x_{t} \rangle = \langle r_{i,j,t}, u_{j,t}, v_{i,j,(r_{i,j,t}),t}, x_{i,j,t} \rangle$.

    \item (Branching) Branch on $\{ x_{i,j,t} = 0 \}$. Update $\Pi(x_{i,j,t}) \coloneqq \bigl\{ \{x_{i,j,t} = 1 \} \bigr\}$. No further variables in $\mathbb{X}_{t}$ may affect $r_{i,j,t}$ because it is an empty set. Proceed to Pruning.

    \item (Pruning) Mark the arc $x_{i,j,t} \xrightarrow{\{ x_{i,j,t} = 0 \}} *$ as $x_{i,j,t} \xrightarrow{\{ x_{i,j,t} = 0 \}} \{ r_{i,j,t} = 0 \}$. Proceed to Backtracking.

    \item (Backtracking) Because $\Pi(x_{t}) = \Pi(x_{i,j,t}) \neq \varnothing$, go back to Branching.

    \item (Branching) Branch on $\{ x_{i,j,t} = 1 \}$. Update $\Pi(x_{i,j,t}) \coloneqq \varnothing$. Again, no further variables in $\mathbb{X}_{t}$ may affect the value of $r_{i,j,t}$. Proceed to Pruning.

    \item (Pruning) Mark the arc $x_{i,j,t} \xrightarrow{\{ x_{i,j,t} = 1 \}} *$ as $x_{i,j,t} \xrightarrow{\{ x_{i,j,t} = 1 \}} \{ r_{i,j,t} = 1 \}$. Proceed to Backtracking.

    \item (Backtracking) Because $\Pi(x_{t}) = \Pi(x_{i,j,t}) = \varnothing$, we update $\mathbb{X}_{t} \coloneqq \mathbb{X}_{t} \cup \{ x_{t} \} = \{ x_{i,j,t} \}$, $S \coloneqq \triangleleft S = \langle r_{i,j,t}, u_{j,t}, v_{i,j,(r_{i,j,t}),t} \rangle$, and relabel $v_{i,j,(r_{i,j,t}),t}$ as $x_{t}$. Now we have $\Pi(x_{t}) = \Pi(v_{i,j,(r_{i,j,t}),t}) \neq \varnothing$, we go back to Branching.

    \item (Branching) Now, branch on $\{ \lceil \tau_{i, j} \cdot v_{i,j,(r_{i,j,t}),t} \rceil > r_{i,j,t} \}$. Update $\Pi(v_{i,j,(r_{i,j,t}),t}) \coloneqq \varnothing$. In this case, the operation is running on the machine and will not complete by the end of $[t-1, t)$. This scenario presents a corner case where different modellers may choose different approaches. While some modellers may model this situation such that a new operation preempts the previously running one, others may choose to reject the new operation initiation. In our model, we follow the latter approach, where the previously running operation takes priority and the system prevents any new operation from being initiated. Therefore, we conclude that no further variables need to be considered. Proceed to Pruning.

    \item (Pruning) Mark the arc $v_{i,j,(r_{i,j,t}),t} \xrightarrow{ \{ \lceil \tau_{i, j} \cdot v_{i,j,(r_{i,j,t}),t} \rceil > r_{i,j,t} \}} *$ as $v_{i,j,(r_{i,j,t}),t} \xrightarrow{ \{ \lceil \tau_{i, j} \cdot v_{i,j,(r_{i,j,t}),t} \rceil > r_{i,j,t} \}} \{ r_{i,j,t+1} = r_{i,j,t} + 1 \}$. Proceed to Backtracking.

    \item (Backtracking) Recall that $x_{t} = v_{i,j,(r_{i,j,t}),t}$ and $\Pi_{x_{t}} = \varnothing$. We update $\mathbb{X}_{t} \coloneqq \{ v_{ij(r_{ij(t)})t} \}$, $S \coloneqq \langle r_{i,j,t}, u_{j,t} \rangle$. Relabel $u_{j,t}$ as $x_{t}$. Because $\Pi(u_{j,t})$ is empty, continue to update as $\mathbb{X}_{t} \coloneqq \{ v_{ij(r_{ij(t)})t}, u_{j,t} \}$, $S \coloneqq \langle r_{i,j,t} \rangle$. Relabel $r_{i,j,t}$ as $x_{t}$. Because $\Pi(r_{i,j,t}) = \varnothing$, repeat this process with $\mathbb{X}_{t} \coloneqq \{ v_{i,j,(r_{i,j,t}),t}, u_{j,t}, r_{i,j,t} \}$, and $S \coloneqq \varnothing$. This process stops when $S$ is empty, which means that all possible situations have been explored. Terminate the procedure.
\end{enumerate}

Then, by encoding the resulting decision tree (Figure \ref{fig:appB_decision_tree_rijt}), we obtain the transition rules presented in Equation \ref{eq:appA_tf_rijt} of Appendix \ref{appendix:transition_functions}. The remaining transition rules are actually obtained via a similar empirical approach.

Having illustrated this empirical procedure, we obtain the decision tree for expressing the transition rules for the state variable $r_{i,j,t+1}$. As we can deduce from this procedure, there are several advantages of using a decision tree. First, by traversing the tree, we ensure that no corner cases of the transition function are ignored. Additionally, it disentangles the coupled logic into smaller components, which allows the modeller to focus on one case at a time. As a result, following this procedure reduces the possibility of negligence. Finally, this expression is highly modularised, which means that if new disturbances need to be considered, we can quickly check each leaf node in the decision tree to determine if different values of the new disturbances will affect the target state variable.

%% file: sections/appendixC.tex
\section{MILP formulation} \label{appendix:milp_formulation}

This appendix presents the \acrshort*{milp} formulation summarised in Section \ref{sec:problem_formulation}, which will be used to generate new schedules at rescheduling time points. For readers who are already familiar with static scheduling of multipurpose batch processes, our formulation can be viewed as a variant of the standard discrete-time formulation. Notably, to distinguish from variables in the dynamic system formulation, we use bar notations (such as $\overline{x}$) to denote decision variables within the \acrshort*{milp}. Also, in contrast with $t$, which is used in the dynamic system formulation to represent real-time points, we use $t'$ to index virtual-time points in the \acrshort*{milp}.

\subsubsection*{Parameters}
\begin{itemize}
    \item $\tau_{i,j}^{\max}$. An integer parameter that represents the maximum possible time of processing a batch of task $i$ on machine $j$. We set $\tau_{i,j}^{\max} = \lceil \tau_{i,j} \cdot v_{i,j,n,t}^{\max} \rceil$, where $v_{i,j,n,t}^{\max}$ represents the maximum possible value of the processing time variation variable $v_{i,j,n,t}$.
    \item $\tau_{i,j,n,t'}$. An auxiliary integer parameter that represents the adjusted processing time by incorporating the look-ahead information of $v_{i,j,n,t}$. Equation \ref{eq:appC_parameter_tau} determines the value of $\tau_{i,j,n,t'}$.
        \begin{subequations} \label{eq:appC_parameter_tau}
        \begin{numcases}{\tau_{i,j,n,t'} = }
            \lceil \tau_{i,j} \cdot v_{i,j,n,t} \rceil  & $\text{if } n = 0 \land t' \in \left[t, t+h^{\mathrm{C}} \right]$ \label{eq:appC_parameter_tau_1} \\
            \tau_{i,j}                               & $\text{if } n = 0 \land  t' \in \left[ t + h^{\mathrm{C}} + 1, t + h_{t} \right]$ \label{eq:appC_parameter_tau_2} \\
            \tau_{i,j,n-1,t'-1}                    & $\text{if } n \geq 1$ \label{eq:appC_parameter_tau_3}
        \end{numcases}
        \end{subequations}
    Here, Equation \ref{eq:appC_parameter_tau_1} represents the case where $n=0$ and the virtual-time point $t'$ falls within the next certainty horizon $\left[ t, t+h^{\mathrm{C}} \right]$. Therefore, according to the definition of $v_{i,j,n,t}$, the processing time should adjust to $\lceil \tau_{i,j} \cdot v_{i,j,n,t} \rceil$. Equation \ref{eq:appC_parameter_tau_2} represents the case that $n=0$ and $t'$ falls beyond the certainty horizon. In this case, we assume that $\tau_{i,j,n,t'}$ equals the nominal processing time $\tau_{i,j}$. Equation \ref{eq:appC_parameter_tau_3} propagates the observed value of $v_{i,j,n,t'}$ from $n=0$ to $n \geq 1$.
    \item $\rho_{i,j,k,n,t'}$. An auxiliary continuous parameter that represents the adjusted conversion coefficient by incorporating the look-ahead information of the yield loss variable $w_{i,j,n,t}$. Equation \ref{eq:appC_parameter_rho} determines the value of $\rho_{i,j,k,n,t'}$.
        \begin{subequations} \label{eq:appC_parameter_rho}
        \begin{numcases}{\rho_{i,j,k,n,t'} = }
            \rho_{i,k}                       & $\text{if } \rho_{i,k} < 0$ \label{eq:appC_parameter_rho_1} \\
            \rho_{i,k} \cdot w_{i,j,n,t'}       & $\text{if } \rho_{i,k} > 0 \land n = 0 \land t' \in \left[ t, t+h^{\mathrm{C}} \right]$ \label{eq:appC_parameter_rho_2} \\
            \rho_{i,k}                       & $\text{if } \rho_{i,k} > 0 \land n = 0 \land t' \in \left[ t+h^{\mathrm{C}}+1, t + h_{t} \right]$ \label{eq:appC_parameter_rho_3} \\
            \rho_{i,j,k,n-1,t'-1}            & $\text{if } \rho_{i,k} > 0 \land n \geq 1$ \label{eq:appC_parameter_rho_4}
        \end{numcases}
        \end{subequations}
    To interpret, recall that the parameter $\rho_{i,k}$ represents the nominal conversion coefficient. Recall also that $\rho_{i,k} > 0$ and $\rho_{i,k} < 0$ mean that the task $i$ produces and consumes material $k$, respectively. Equation \ref{eq:appC_parameter_rho_1} represents the case where task $i$ consumes material $k$. Since the batch input is not affected by the yield loss variable, in this case we simply set $\rho_{i,j,k,n,t'}$ to its nominal value. Equations \ref{eq:appC_parameter_rho_2}, \ref{eq:appC_parameter_rho_3}, and \ref{eq:appC_parameter_rho_4} are similar to Equations \ref{eq:appC_parameter_tau_1}, \ref{eq:appC_parameter_tau_2}, \ref{eq:appC_parameter_tau_3}, respectively. Specifically, Equation \ref{eq:appC_parameter_rho_2} represents the case where the value of $w_{i,j,n,t'}$ has been observed by the scheduler. Equation \ref{eq:appC_parameter_rho_3} represents the case where the value of $w_{i,j,n,t'}$ has not been observed. Equation \ref{eq:appC_parameter_rho_4} propagates the value of $\rho_{i,j,k,n,t'}$ from $n=0$ to cases where $n \geq 1$. As a result, the parameter $\rho_{i,j,k,n,t'}$ can be interpreted as, if at the end of $[t'-1, t')$, the elapsed time of a batch equals $n$ time periods, then the value of $\rho_{i,j,k,n,t'}$ represents the adjusted conversion coefficient between this batch and the material $k$.
\end{itemize}

\subsubsection*{Decision variables}
\begin{itemize}
    \item $\overline{r}_{i,j,n,t'}$. A binary variable that represents the elapsed time of a batch. If at the end of $[t'-1, t')$, the elapsed time of the batch of task $i$ on machine $j$ equals $n$ time periods, then $\overline{r}_{i,j,n,t'} = 1$. Otherwise no such batch is initiated, then $\overline{r}_{i,j,n,t'} = 0$.
    \item $\overline{b}_{i,j,n,t'}$. A semi-continuous variable that represents the size of a batch. If at the end of $[t'-1, t')$, the elapsed time of the batch of task $i$ on machine $j$ equals $n$ time periods, then $\overline{b}_{i,j,n,t'}$ equals its batch size. Otherwise no such batch is initiated, then $\overline{b}_{i,j,n,t'} = 0$.
    \item $\overline{s}_{k,t'}$. A continuous variable that represents the inventory level of material $k$ during $[t'-1, t')$.
    \item $\overline{l}_{k,t'}$. A continuous variable that represents the backlog level of product $k$ during $[t'-1, t')$.
    \item $\overline{x}_{i,j,t'}$. A binary variable that indicates the initiation of a batch. If a batch of task $i$ on machine $j$ is initiated at time point $t'$, then $\overline{x}_{i,j,t'} = 1$. Otherwise, $\overline{x}_{i,j,t'} = 0$.
    \item $\overline{y}_{i,j,t'}$. A semi-continuous variable that represents the input batch size. If a batch of task $i$ on machine $j$ is initiated at time point $t'$, then $\overline{y}_{i,j,t'}$ equals its input size. Otherwise, $\overline{y}_{i,j,t'} = 0$.
    \item $\overline{p}_{k,t'}$. A continuous variable that represents the quantity of raw material $k$ purchased at time point $t'$.
    \item $\overline{q}_{k,t'}$. A continuous variable that represents the quantity of product $k$ shipped at time point $t'$.
\end{itemize}

\subsubsection*{Constraints}

\textbf{Initial state constraints}. When developing a \acrshort*{milp} at a rescheduling time point $t$, the first step is to ensure that the decision variables with index $t' = t$ represent the current system state. We achieve this step via a set of initial state constraints that fix the value of these variables. In our \acrshort*{milp} formulation, four groups of variables require fixing: the elapsed time variable $\overline{r}_{i,j,n,t'}$, the batch size variable $\overline{b}_{i,j,n,t'}$, the inventory variable $\overline{s}_{k,t'}$, and the backlog variable $\overline{l}_{k,t'}$. Equations \ref{eq:appC_init_rijnt}--\ref{eq:appC_init_lkt} represent the initial state constraints for these four groups of variables.

\begin{align} \label{eq:appC_init_rijnt}
    M \cdot (\overline{r}_{i,j,(r_{i,j,t}),t'} - 1) \leq r_{i,j,t} \leq M \cdot \overline{r}_{i,j,(r_{i,j,t}),t'}, \quad \forall j \in \mathbb{J}, i \in \mathbb{I}_{j}, t' = t.
\end{align}

Equation \ref{eq:appC_init_rijnt} represents the initial state constraint for $\overline{r}_{i,j,n,t'}$. Here, $M$ represents the big-M constant (such as $100$). To interpret, let us first consider the case when $r_{i,j,t} > 0$. In this case, the two inequalities hold if and only if $\overline{r}_{i,j,(r_{i,j,t}),t'} = 1$. For another case where $r_{i,j,t} = 0$, the inequalities hold if and only if $\overline{r}_{i,j,(r_{i,j,t}),t'} = 0$. As a result, Equation \ref{eq:appC_init_rijnt} ensures that the value of $\overline{r}_{i,j,(r_{i,j,t}),t'}$ represents the current state of elapsed times of the system.

\begin{equation} \label{eq:appC_init_bijnt}
\begin{split}
    M \cdot (\overline{b}_{i,j,(r_{i,j,t}),t'} - 1) \leq r_{i,j,t} \leq M \cdot \overline{b}_{i,j,(r_{i,j,t}),t'}, \quad \forall j \in \mathbb{J}, i \in \mathbb{I}_{j}, t' = t.
\end{split}
\end{equation}

Similarly, Equation \ref{eq:appC_init_bijnt} represents the initial state constraint for $\overline{b}_{i,j,n,t'}$.

\begin{equation} \label{eq:appC_init_skt}
\begin{split}
    \overline{s}_{k,t'} = s_{k,t}, \quad \forall k \in \mathbb{K}, t' = t.
\end{split}
\end{equation}

Equation \ref{eq:appC_init_skt} represents the initial state constraint for $\overline{s}_{k,t}$. Because the state variable $s_{k,t}$ and the decision variable $\overline{s}_{k,t}$ share the same semantic, we simply set $\overline{s}_{k,t} = s_{k,t}$.

\begin{equation} \label{eq:appC_init_lkt}
\begin{split}
    \overline{l}_{k,t'} = l_{k,t}, \quad \forall k \in \mathbb{K}^{\mathrm{P}}, t' = t.
\end{split}
\end{equation}

Similarly, Equation \ref{eq:appC_init_lkt} represents the initial state constraint for $\overline{l}_{k,t}$.

\textbf{Lifting constraint for elapsed time}. Equation \ref{eq:appC_lifting_rijnt} represents the lifting constraint for $\overline{r}_{i,j,n,t'}$. The term ``lifting'' means that this constraint lifts the value of $\overline{r}_{i,j,n-1,t'}$ (with indices $n-1$ and $t'$) to $\overline{r}_{i,j,n,t'+1}$ (with indices $n$ and $t'+1$). Intuitively, Equation \ref{eq:appC_lifting_rijnt} states that, if at the end of $[t'-1, t')$, the elapsed time of a batch equals $n-1$ time periods, then at the end of the next time interval $[t', t'+1)$, the elapsed time of this batch will equal $n$ time periods if it remains ongoing. For time points within the certainty horizon (that is, $t' \in \left[ t, t+h^{\mathrm{C}} \right)$, as presented in Equation \ref{eq:appC_lifting_rijnt_before_hc}), we need also incorporate the look-ahead information about machine breakdowns. That is, if $u_{j,t'} = 1$, which means that machine $j$ is unavailable due to a breakdown during $[t'-1, t')$, then the batch should terminate during the next time interval $[t', t'+1)$. For time points beyond the certainty horizon (that is, $t' \in \left[t+h^{\mathrm{C}}, t + h_{t} \right)$, as presented in Equation \ref{eq:appC_lifting_rijnt_after_hc}), we assume that no forthcoming breakdowns exist on machine $j$.
\begin{subequations} \label{eq:appC_lifting_rijnt}
\begin{alignat}{2}
&\begin{aligned} \label{eq:appC_lifting_rijnt_before_hc}
    &\overline{r}_{i,j,n,t'+1} = \overline{r}_{i,j,n-1,t'} \cdot (1 - {u}_{j,t'}), \\
    &\quad \forall j \in \mathbb{J}, \forall i \in \mathbb{I}_{j}, n \in \{n \in \mathbb{Z}^{+}: n \geq 2 \land n \leq \tau_{i,j,n,t'} \}, t' \in \{t, t+1, \cdots, t + h^{\mathrm{C}} - 1 \}.
\end{aligned}
\\
&\begin{aligned} \label{eq:appC_lifting_rijnt_after_hc}
    &\overline{r}_{i,j,n,t'+1} = \overline{r}_{i,j,n-1,t'}, \\
    &\quad \forall j \in \mathbb{J}, \forall i \in \mathbb{I}_{j}, n \in \{n \in \mathbb{Z}^{+}: n \geq 2 \land n \leq \tau_{i,j,n,t'} \}, t' \in \{t + h^{\mathrm{C}}, t + h^{\mathrm{C}} + 1, \cdots, t + h_{t} - 1 \}.
\end{aligned}
\end{alignat}
\end{subequations}

\textbf{Lifting constraint for batch size}. Similar to Equation \ref{eq:appC_lifting_rijnt}, Equation \ref{eq:appC_lifting_bijnt} represents the lifting constraints for $\overline{b}_{i,j,n,t'}$.
\begin{subequations} \label{eq:appC_lifting_bijnt}
\begin{alignat}{2}
&\begin{aligned} \label{eq:appC_lifting_bijnt_before_hc}
    &\overline{b}_{i,j,n,t'+1} = \overline{b}_{i,j,n-1,t'} \cdot (1 - {u}_{j,t'}), \\
    &\quad \forall j \in \mathbb{J}, \forall i \in \mathbb{I}_{j}, n \in \{n \in \mathbb{Z}^{+}: n \geq 2 \land n \leq \tau_{i,j,n,t'} \}, t' \in \{t, t+1, \cdots, t + h^{\mathrm{C}} - 1 \}.
\end{aligned}
\\
&\begin{aligned} \label{eq:appC_lifting_bijnt_after_hc}
    &\overline{b}_{i,j,n,t'+1} = \overline{b}_{i,j,n-1,t'}, \\
    &\quad \forall j \in \mathbb{J}, \forall i \in \mathbb{I}_{j}, n \in \{n \in \mathbb{Z}^{+}: n \geq 2 \land n \leq \tau_{i,j,n,t'} \}, t' \in \{t + h^{\mathrm{C}}, t + h^{\mathrm{C}} + 1, \cdots, t + h_{t} - 1 \}.
\end{aligned}
\end{alignat}
\end{subequations}

\textbf{Inventory constraint}. Equations \ref{eq:appC_inventory_before_hc} and \ref{eq:appC_inventory_after_hc} represent the inventory constraints within and beyond the certainty horizon, respectively. In both equations, the first summation term represents the quantity of materials added from upstream batch completions, and the second summation term represents the quantity of materials consumed by downstream batch initiations. In Equation \ref{eq:appC_inventory_before_hc}, the logical expression $1- \min\{ \left| \tau_{i,j,n,t'} - n \right|, 1 \}$ ensures that only when $n = \tau_{i,j,n,t'}$ the batch will be added to the inventory. Otherwise if $n \neq \tau_{i,j,n,t'}$, the expression equals $0$ and therefore the batch will not be added to the inventory. By convention, we set $\overline{p}_{k,t'} = 0$ for $k \notin \mathbb{K}^{\mathrm{R}}$ and set $\overline{q}_{k,t'} = 0$ for $k \notin \mathbb{K}^{\mathrm{P}}$.

\begin{subequations} \label{eq:appC_inventory}
\begin{alignat}{2}
&\begin{aligned} \label{eq:appC_inventory_before_hc}
    &\overline{s}_{k,t'+1} = \overline{s}_{k,t'} + \sum_{j \in \mathbb{J}} \sum_{i \in \mathbb{I}_{j} \cap \mathbb{I}_{k}^{+} } \sum_{n = 1}^{\tau_{i,j,n,t'}} \Bigl(1- \min\left\{ \left| \tau_{i,j,n,t'} - n \right|, 1 \right\} \Bigr) \cdot (1 - u_{j,t'}) \cdot \rho_{i,j,k,n,t'} \cdot \overline{b}_{i,j,n,t'} \\
    &\quad + \sum_{j \in \mathbb{J}} \sum_{i \in \mathbb{I}_{j} \cap \mathbb{I}_{k}^{-}} \rho_{i,j,k,n,t'} \cdot \overline{y}_{i,j,t'} + \overline{p}_{k,t'} - \overline{q}_{k,t'}, \quad \forall k \in \mathbb{K}, t' \in \{t, t+1, \cdots, t+h^{\mathrm{C}} - 1 \}.
\end{aligned}
\\
&\begin{aligned} \label{eq:appC_inventory_after_hc}
    &\overline{s}_{k,t'+1} = \overline{s}_{k,t'} + \sum_{j \in \mathbb{J}} \sum_{i \in \mathbb{I}_{j} \cap \mathbb{I}_{k}^{+} } \sum_{n = 1}^{\tau_{i,j,n,t'}} \Bigl(1- \min\left\{ \left| \tau_{i,j,n,t'} - n \right|, 1 \right\} \Bigr) \cdot \rho_{i,j,k,n,t'} \cdot \overline{b}_{i,j,n,t'} \\
    &\quad + \sum_{j \in \mathbb{J}} \sum_{i \in \mathbb{I}_{j} \cap \mathbb{I}_{k}^{-}} \rho_{i,j,k,n,t'} \cdot \overline{y}_{i,j,t'} + \overline{p}_{k,t'} - \overline{q}_{k,t'}, \quad \forall k \in \mathbb{K}, t' \in \{t + h^{\mathrm{C}}, \cdots, t+h_{t} - 1 \}.
\end{aligned}
\end{alignat}
\end{subequations}

\textbf{Backlog constraint}. Equation \ref{eq:appC_backlog} represents the backlog constraint. Similar to the transition function of the backlog variable (that is, Equation \ref{eq:appA_tf_lkt}), the next-step backlog equals the current step backlog added by the baseline demand, the intermittent demand, and the urgent demand, and subtracted by the shipped product quantity.
\begin{equation} \label{eq:appC_backlog}
\begin{split}
    \overline{l}_{k,t' + 1} = \overline{l}_{k,t'} + \lambda_{k,t'} + \nu_{k,t'} + f_{k,t'} - \overline{q}_{k,t'}, \quad \forall k \in \mathbb{K}^{\mathrm{P}}, t' \in \{ t, t+1, \cdots, t+h_{t} - 1\}.
\end{split}
\end{equation}

\textbf{Logic constraint between elapsed time and batch initiation}. Equation \ref{eq:appC_logical_rijnt_xijt} represents the logical constraint between $\overline{r}_{i,j,n,t'}$ and $\overline{x}_{i,j,t'}$. Equation \ref{eq:appC_logical_rijnt_xijt} states that, if at time point $t'$, a batch initiates (that is, $\overline{x}_{i,j,t'} = 1$), then at the end of $[t', t'+1)$, the elapsed time of this batch will equal one time period. Note that even if the processing time of this batch equals only one time period, according to the definition of $\overline{r}_{i,j,n,t'}$, Equation \ref{eq:appC_logical_rijnt_xijt} still holds because the batch is completing by the end of $[t', t'+1)$.
\begin{align} \label{eq:appC_logical_rijnt_xijt}
    \overline{r}_{i,j,n,t'+1} = \overline{x}_{i,j,t'}, &\quad \forall j \in \mathbb{J}, i \in \mathbb{I}_{j}, n = 1, t' \in \{t, t+1, \cdots, t + h_{t} - 1\}.
\end{align}

\textbf{Logical constraint between batch size and batch input}. Similar to Equation \ref{eq:appC_logical_rijnt_xijt}, Equation \ref{eq:appC_logical_bijnt_yijt} represents the logical constraint between $\overline{b}_{i,j,n,t'}$ and $\overline{y}_{i,j,t'}$.
\begin{align} \label{eq:appC_logical_bijnt_yijt}
    \overline{b}_{i,j,n,t'+1} = \overline{y}_{i,j,t'}, &\quad \forall j \in \mathbb{J}, i \in \mathbb{I}_{j}, n = 1, t' \in \{t, t+1, \cdots, t + h_{t} - 1\}.
\end{align}

\textbf{Machine occupancy constraint}. Equations \ref{eq:appC_machine_occupancy_before_hc} and \ref{eq:appC_machine_occupancy_after_hc} represent the machine occupancy constraints within and beyond the certainty horizon, respectively. Intuitively, Equation \ref{eq:appC_machine_occupancy} ensures that for any time point $t'$ and machine $j$, exactly one of the following conditions hold: (1) $j$ is idle at $t'$, (2) a batch initiates at $t'$ on $j$, (3) a batch is ongoing during $[t', t'+1)$, and (4) $j$ is unavailable due to a breakdown during $[t', t'+1)$. The condition (4) only applies to time points within the certainty horizon, as presented in Equation \ref{eq:appC_machine_occupancy_before_hc}.
\begin{subequations} \label{eq:appC_machine_occupancy}
\begin{alignat}{2}
&\begin{aligned} \label{eq:appC_machine_occupancy_before_hc}
    \sum_{i \in \mathbb{I}_{j}} \overline{x}_{i,j,t'} + \sum_{i \in \mathbb{I}_{j} } \sum_{n \in \{ n \in \mathbb{Z}^{+}: n \geq 1 \land n \leq \tau_{i,j,n,t'} - 1 \} } \overline{r}_{i,j,n,t'} + u_{j,t'} \leq 1, \quad \forall j \in \mathbb{J}, t' \in \{ t, t+1, \cdots, t+h^{\mathrm{C}} \}.
\end{aligned}
\\
&\begin{aligned} \label{eq:appC_machine_occupancy_after_hc}
    \sum_{i \in \mathbb{I}_{j}} \overline{x}_{i,j,t'} + \sum_{i \in \mathbb{I}_{j} } \sum_{n \in \{ n \in \mathbb{Z}^{+}: n \geq 1 \land n \leq \tau_{i,j,n,t'} - 1 \} } \overline{r}_{i,j,n,t'} \leq 1, \quad \forall j \in \mathbb{J}, t' \in \{ t + h^{\mathrm{C}} + 1, \cdots, t+h_{t} \}.
\end{aligned}
\end{alignat}
\end{subequations}

\textbf{Bounding constraint for raw material purchases}. Equation \ref{eq:appC_bound_raw} ensures that the purchased quantity of raw materials must be greater than zero and be less than the supply upper bound.
\begin{align} \label{eq:appC_bound_raw}
    0 \leq \overline{p}_{k,t'} \leq \xi_{k,t'}, \quad \forall k \in \mathbb{K}^{\mathrm{R}}, t' \in \{t, t+1, \cdots, t+ h_{t} \}.
\end{align}

\textbf{Bounding constraint for batch input}. Equation \ref{eq:appC_bound_batch_input} ensures that, if a batch initiates, its input size should be between the minimum and maximum allowable batch inputs.
\begin{align} \label{eq:appC_bound_batch_input}
    y^{\min}_{i,j} \cdot \overline{x}_{i,j,t'} \leq \overline{y}_{i,j,t'} \leq y^{\max}_{i,j} \cdot \overline{x}_{i,j,t'}, \quad \forall j \in \mathbb{J}, i \in \mathbb{I}_{j}, t' \in \{t, t+1, \cdots, t+h_{t} \}.
\end{align}

\textbf{Bounding constraint for inventory}. Equation \ref{eq:appC_bound_inventory} ensures that the inventory level of $k$ at any time point $t'$ should be no less than zero and be no greater than the storage capacity.
\begin{align} \label{eq:appC_bound_inventory}
    0 \leq \overline{s}_{k,t'} \leq s^{\max}_{k}, \quad \forall k \in \mathbb{K}.
\end{align}

\subsubsection*{Objective functions}
In multipurpose batch scheduling, it is not rare that thousands of schedules share the same optimal cost value. However, within these schedules, some of them can be more preferable than others in practice. For example, schedules that avoid unnecessary postpones of batches or maximise consistencies with the original schedule are often preferred in industrial settings. Therefore, to prioritise schedules with desired properties, we hierarchically optimise three objective functions rather than a single one.

\textbf{Objective for minimising cost}. Equation \ref{eq:appC_obj1} presents the primary objective, which minimises the total cost over the next scheduling horizon $[t, t+h_{t}]$. Similar to the cost function used in the dynamic system formulation (that is, Equation \ref{eq:appA_cost_func_ct}), the total cost consists of the setup cost, the inventory cost, and the backlog cost.
\begin{equation} \label{eq:appC_obj1}
\begin{split}
    \min \quad \overline{f}_{1} =  \sum_{t' =t}^{t + h_{t}} \Biggl( \sum_{j \in \mathbb{J}} \sum_{i \in \mathbb{I}_{j}} c^{x}_{i,j} \overline{x}_{i,j,t'}
    + \sum_{k \in \mathbb{K}} c^{s}_{k} \overline{s}_{k,t'}
    + \sum_{k \in \mathbb{K}^{\mathrm{P}}} c^{l}_{k} \overline{l}_{k,t'} \Biggr)
\end{split}
\end{equation}

\textbf{Objective for batch initiation changes}. Equation \ref{eq:appC_obj2} presents the secondary objective, which maximises the consistency of batch initiations with the original schedule. When $x_{i,j,t'} = \overline{x}_{i,j,t'}$, which means that a batch exists in both the original schedule and the new schedule, the multiplicative reward $x_{i,j,t'} \cdot \overline{x}_{i,j,t'} = 1$. Otherwise when $x_{i,j,t'} \neq \overline{x}_{i,j,t'}$, which means that the batch is cancelled or shifted, then the reward $x_{i,j,t'} \cdot \overline{x}_{i,j,t'} = 0$. Note that Equation \ref{eq:appC_obj2} is unidirectional, which means that it only rewards consistencies for batches that already exist in the original schedule. For batches that appear in the new schedule for the first time, this objective does not enforce where they should be placed.
\begin{align} \label{eq:appC_obj2}
    \max \quad \overline{f}_{2} = \sum_{j \in \mathbb{J}} \sum_{i \in \mathbb{I}_{j}} \sum_{t' = t}^{t + h_{t}} x_{i,j,t'} \cdot \overline{x}_{i,j,t'}
\end{align}

\textbf{Objective for batch initiation lateness}. Equation \ref{eq:appC_obj3} presents the tertiary objective, which minimises the total lateness of batch initiations. The exponent $\frac{t'-t}{h_{t} - t}$ penalises delays in batch initiations. The denominator $h_{t} - t$ represents a normalisation term that ensures the penalty magnitude is only related to the length of the scheduling horizon, while is regardless of the absolute time step $t$. The reason that we use an exponential penalty rather than a linear one is to make the difference in batch initiation lateness more numerically significant. As a result, Equation \ref{eq:appC_obj3} enforces new batches to be scheduled as early as possible.
\begin{align} \label{eq:appC_obj3}
    \min \quad \overline{f}_{3} =  \sum_{j \in \mathbb{J}} \sum_{i \in \mathbb{I}_{j}} \sum_{t' = t}^{t + h_{t}} e^{\frac{t' - t}{h_{t} - t}} \cdot \overline{x}_{i,j,t'}
\end{align}

The three objectives are solved hierarchically. Specifically, we first solve $\overline{f}_{1}$ to obtain its optimal value $\overline{f}_{1}^{*}$. We then add $\overline{f}_{1} = \overline{f}_{1}^{*}$ as an equality constraint and solve $\overline{f}_{2}$ to obtain $\overline{f}_{2}^{*}$. Finally, we add both $\overline{f}_{1} = \overline{f}_{1}^{*}$ and $\overline{f}_{2} = \overline{f}_{2}^{*}$ as equality constraints and solve $\overline{f}_{3}$. The solution of the third optimisation is returned as the new schedule.

%% file: sections/appendixD.tex
\section{Pseudo codes of algorithm modules} \label{appendix:algorithms}

This appendix presents the pseudo codes of the five algorithm modules used in the main Bayesian dynamic scheduling loop (Algorithm \ref{algo:bds} in Section \ref{sec:methodology}): \textsc{ImpactPropagation}, \textsc{StructureLearning}, \textsc{ParameterLearning}, \textsc{WhenToReschedule}, and \textsc{HowToReschedule}.

\subsection*{Impact propagation}

\vspace{0.5cm}
\begin{algorithm}[H]
\setstretch{1.25}
\SetKwInput{KwInput}{Input}
\SetKwProg{Fn}{Function}{}{}
\caption{Construct impact variables over a dependence-representing \acrshort*{dag}} \label{algo:impact_propagation}
\Fn{\textup{\textsc{ImpactPropagation}(
    $\ell$, $\mathcal{G} = (\mathbb{O}, \mathbb{A})$,
    $\{\boldsymbol{I}^{\mathrm{En}, \ell}_{t_s(\mathcal{O}): t_e(\mathcal{O})} : \mathcal{O} \in \mathbb{O}\}$,
    $\textsc{IsoFunc}$, $\textsc{PropFunc}$)}}{
\Begin{
    Initialise $\mathbb{O}_{0} \coloneqq \{\mathcal{O} \in \mathbb{O}: \mathcal{O}$ is a root in $\mathcal{G}\}$\;
    \For{$\mathcal{O} \in \mathbb{O}_{0}$}{
        $\underline{Z}_{\mathcal{O}}^{\ell} \coloneqq \textsc{IsoFunc}(\boldsymbol{I}_{t_s(\mathcal{O}): t_e(\mathcal{O})}^{\mathrm{En}, \ell}; \ell)$\;
        $Z_{\mathcal{O}}^{\ell} \coloneqq \textsc{PropFunc}(\underline{Z}_{\mathcal{O}}^{\ell}, \varnothing; \ell)$\;
    }
    Initialise $\mathbb{O}_{1} \coloneqq \mathbb{O} \setminus \mathbb{O}_{0}$\;
    \While{$\mathbb{O}_{1} \neq \varnothing$}{
        $\mathbb{O}_{2} \coloneqq \{\mathcal{O} \in \mathbb{O}_{1} : \mathrm{Pa}^{\mathcal{G}}(\mathcal{O}) \subseteq \mathbb{O}_{0}\}$\;
        Select any $\mathcal{O}$ from $\mathbb{O}_{2}$\;
        $\underline{Z}_{\mathcal{O}}^{\ell} \coloneqq \textsc{IsoFunc}(\boldsymbol{I}_{t_s(\mathcal{O}): t_e(\mathcal{O})}^{\mathrm{En}, \ell}; \ell)$\;
        $Z_{\mathcal{O}}^{\ell} \coloneqq \textsc{PropFunc}(\underline{Z}_{\mathcal{O}}^{\ell}, \boldsymbol{Z}_{\mathrm{Pa}^{\mathcal{G}}(\mathcal{O})}^{\ell}; \ell)$\;
        $\mathbb{O}_{0} \coloneqq \mathbb{O}_{0} \cup \{\mathcal{O}\}$\; $\mathbb{O}_{1} \coloneqq \mathbb{O}_{1} \setminus \{\mathcal{O}\}$\;
    }
    \Return{$\boldsymbol{Z}^{\ell}_{\mathbb{O}}$}
}
}
\end{algorithm}
\vspace{0.5cm}

Note that, for any well-defined \acrshort*{dag} structure $\mathcal{G}^{\mathrm{Op}} = (\mathbb{O}, \mathbb{A})$, we can always visit the entire operation set by following the procedure of \textsc{ImpactPropagation}. Furthermore, the specific choice of the candidate operation from $\mathbb{O}_{2}$ does not affect the final output. To show these facts, assume that $\mathcal{O}_{1}, \mathcal{O}_{2}, \cdots, \mathcal{O}_{r}, \mathcal{O}_{r+1}, \cdots, \mathcal{O}_{n}$ forms a topological ordering of the operation set $\mathbb{O}$ (for the definition of a topological ordering for a \acrshort*{dag}, see Section S1 of the Supplementary Information), where the subset $\{ \mathcal{O}_{1}, \mathcal{O}_{2} \cdots, \mathcal{O}_{r} \}$ represents root operations. Then, we can apply the isolation-and-propagation procedure to the entire operation set by following this ordering. This is because we will first visit root operations $\{ \mathcal{O}_{1}, \cdots, \mathcal{O}_{r} \}$ in the initialisation block. After that, proceeding with the ordering, we will enter the main loop of the algorithm. Assume that within the main loop, we arrive at an operation $\mathcal{O}_{j}$, where $r < j \leq n$. Then, we must have $\mathrm{Pa}^{\mathcal{G}^{\mathrm{Op}}} (\mathcal{O}_{j}) \subseteq \{ \mathcal{O}_{1}, \cdots, \mathcal{O}_{r}, \mathcal{O}_{r+1}, \cdots, \mathcal{O}_{j-1} \}$. In other words, having $\mathcal{O}_{1}, \cdots, \mathcal{O}_{j-1}$ already visited, selecting $\mathcal{O}_{j}$ is always a valid choice. This is because, if $\mathrm{Pa}^{\mathcal{G}^{\mathrm{Op}}} (\mathcal{O}_{j}) \not\subseteq \{ \mathcal{O}_{1}, \cdots, \mathcal{O}_{r}, \mathcal{O}_{r+1}, \cdots, \mathcal{O}_{j-1} \}$, then there exists an operation $\mathcal{O}_{i}$ such that $\mathcal{O}_{i} \in \mathrm{Pa}^{\mathcal{G}^{\mathrm{Op}}} (\mathcal{O}_{j})$ but $\mathcal{O}_{i}$ appears after $\mathcal{O}_{j}$ in the topological ordering. However, this contradicts with the definition of topological ordering. Therefore, $\mathcal{O}_{j}$ is a valid candidate, and by induction, any topological ordering of the operation set is a valid order of traversing. As a result, we can always traverse the operation set by following a topological ordering of operations, while varying the choice of candidate operations just corresponds to following another valid topological ordering.

\subsection*{Structure learning}

\vspace{0.5cm}
\begin{algorithm}[H]
\setstretch{1.25}
\SetKwProg{Fn}{Function}{}{}
\caption{Construct \acrshort*{bn} structure for impact variables} \label{algo:structure_learning}
\Fn{\textup{\textsc{StructureLearning}(
    $\boldsymbol{Z}_{\mathbb{O}}^{\ell}$, $\chi_{\ell}^{\mathrm{Temp}}$, $\chi_{\ell}^{\mathrm{Spat}}$)}}{
\Begin{
    Initialise $\mathcal{G} \coloneqq (\boldsymbol{Z}_{\mathbb{O}}^{\ell}, \mathbb{A})$, $\mathbb{A} \coloneqq \varnothing$\;
    \For{\textup{each } $j \in \mathbb{J}$}{
        $\mathbb{O}_{j} \coloneqq \{\mathcal{O} \in \mathbb{O} : j(\mathcal{O}) = j\}$\;
        \For{\textup{each } $\mathcal{O} \in \mathbb{O}_{j}$}{
            \If{$\chi_{\ell}^{\mathrm{Temp}} = 1$}{
                $\mathbb{A} \coloneqq \mathbb{A} \cup \{Z_{\mathcal{O}'}^{\ell} \to Z_{\mathcal{O}}^{\ell}\}$ where $\mathcal{O}' \in \underset{\substack{j(\mathcal{O}') = j(\mathcal{O}) \\ t_e(\mathcal{O}') \leq t_s(\mathcal{O})}}{\arg\max} t_e(\mathcal{O}')$\;
            }
            \If{$\chi_{\ell}^{\mathrm{Spat}} = 1$}{
                $\mathbb{A} \coloneqq \mathbb{A} \cup \{Z_{\mathcal{O}'}^{\ell} \to Z_{\mathcal{O}}^{\ell}\}$ where $\mathcal{O}' \in \underset{\substack{\mathbb{K}^{+}(\mathcal{O}') \cap \mathbb{K}^{-}(\mathcal{O}) \neq \varnothing \\ t_e(\mathcal{O}') \leq t_s(\mathcal{O})}}{\arg\max} t_e(\mathcal{O}')$\;
            }
        }
    }
    \Return{$\mathcal{G}$}
}
}
\end{algorithm}
\vspace{0.5cm}

\subsection*{Parameter learning}

\vspace{0.5cm}
\begin{algorithm}[H]
\setstretch{1.25}
\SetKwProg{Fn}{Function}{}{}
\caption{\acrshort*{bn} parameter learning using Monte Carlo simulation} \label{algo:parameter_learning}
\Fn{\textup{\textsc{ParameterLearning}($n$, $\mathcal{G} = (\boldsymbol{Z}_{\mathbb{O}}^{\ell}, \mathbb{A})$)}}{
\Begin{
    Initialise $\mathbb{D} \coloneqq \varnothing$\;
    \For{$k \in \{1, \ldots, n\}$}{
        \For{\textup{each } $\mathcal{O} \in \mathbb{O}$}{
            Sample $\boldsymbol{i}_{t_s(\mathcal{O}): t_e(\mathcal{O})}^{\mathrm{En}, \ell, (k)} \sim \mathsf{P}(\boldsymbol{I}_{t_s(\mathcal{O}): t_e(\mathcal{O})}^{\mathrm{En}, \ell})$\;
        }
        $\boldsymbol{z}_{\mathbb{O}}^{\ell, (k)} \coloneqq \textsc{ImpactPropagation}(\ell, \mathcal{G}, \{\boldsymbol{i}^{\mathrm{En}, \ell, (k)}_{t_s(\mathcal{O}): t_e(\mathcal{O})} : \mathcal{O} \in \mathbb{O}\}, \textsc{IsoFunc}, \textsc{PropFunc})$\;
        $\mathbb{D} \coloneqq \mathbb{D} \cup \{\boldsymbol{z}_{\mathbb{O}}^{\ell, (k)}\}$\;
    }
    \For{\textup{each } $\mathcal{O} \in \mathbb{O}$}{
        \For{\textup{every } $(z_{\mathcal{O}}^{\ell}, \mathrm{pa}^{\mathcal{G}}(Z_{\mathcal{O}}^{\ell}))$}{
            Set $\hat{\mathsf{P}}(z_{\mathcal{O}}^{\ell} \mid \mathrm{pa}^{\mathcal{G}}(Z_{\mathcal{O}}^{\ell}))$ by the ratio in Equation \ref{eq:parameter_mle}\;
        }
    }
    \Return{$\hat{\mathsf{P}}$}
}
}
\end{algorithm}
\vspace{0.5cm}

\subsection*{When-to-reschedule}

\vspace{0.5cm}
\begin{algorithm}[H]
\setstretch{1.25}
\SetKwProg{Fn}{Function}{}{}
\caption{When-to-reschedule strategy} \label{algo:when_to_reschedule}
\Fn{\textup{\textsc{WhenToReschedule}($\mathbb{O}_{t: t+h_t}$, $\mathbb{O}^{\mathrm{EU}}_{t: t+h^{\mathrm{C}}}$, $\mathbb{O}^{\mathrm{PU}}_{t+h^{\mathrm{C}}: t+h_t}$, $\gamma_{3}$)}}{
\Begin{
    \uIf{$\dfrac{|\mathbb{O}^{\mathrm{EU}}_{t: t+h^{\mathrm{C}}}| + |\mathbb{O}^{\mathrm{PU}}_{t+h^{\mathrm{C}}: t+h_t}|}{|\mathbb{O}_{t: t+h_t}|} \geq \gamma_{3}$}{
        \Return{$\mathsf{True}$}\;
    }
    \Else{
        \Return{$\mathsf{False}$}\;
    }
}
}
\end{algorithm}
\vspace{0.5cm}

\subsection*{How-to-reschedule}

\vspace{0.5cm}
\begin{algorithm}[H]
\setstretch{1.25}
\SetKwProg{Fn}{Function}{}{}
\caption{How-to-reschedule strategy} \label{algo:how_to_reschedule}
\Fn{\textup{\textsc{HowToReschedule}(
    $(\boldsymbol{s}_t, \boldsymbol{a}_{t: t+h_t \mid t}, \boldsymbol{i}_{t: t+h^{\mathrm{C}} \mid t})$,
    $\mathbb{O}_{t: t+h_t \mid t}$,
    $\mathbb{O}^{\mathrm{EU}}_{t: t+h_t \mid t}$,
    $\mathbb{O}^{\mathrm{PU}}_{t: t+h_t \mid t}$,
    $h'_{t}$)}}{
\Begin{
    Develop the \acrshort*{milp} constraints (Equations \ref{eq:appC_init_rijnt}--\ref{eq:appC_bound_inventory}) from the triplet and the target horizon $h'_{t}$\;
    Fix $\overline{x}_{i(\mathcal{O}), j(\mathcal{O}), t_s(\mathcal{O})} = 1$ for each $\mathcal{O} \in \mathbb{O}_{t: t+h_t \mid t} \setminus (\mathbb{O}^{\mathrm{EU}}_{t: t+h_t \mid t} \cup \mathbb{O}^{\mathrm{PU}}_{t: t+h_t \mid t})$\;
    Solve $\min \bar{f}_{1}$ (Equation \ref{eq:appC_obj1}); obtain $\bar{f}_{1}^{*}$\;
    Add $\bar{f}_{1} = \bar{f}_{1}^{*}$; solve $\max \bar{f}_{2}$ (Equation \ref{eq:appC_obj2}); obtain $\bar{f}_{2}^{*}$\;
    Add $\bar{f}_{2} = \bar{f}_{2}^{*}$; solve $\min \bar{f}_{3}$ (Equation \ref{eq:appC_obj3})\;
    Decode the new action sequence $\boldsymbol{a}'_{t: t+h'_t \mid t}$ from the \acrshort*{milp} solution\;
    \Return{$\boldsymbol{a}'_{t: t+h'_t \mid t}$}
}
}
\end{algorithm}
\vspace{0.5cm}

%% file: sections/appendixE.tex
\section{Isolation and propagation functions for the three endogenous disturbance types} \label{appendix:iso_prop_funcs}

This appendix presents the designs of the isolation function \textsc{IsoFunc} and the propagation function \textsc{PropFunc} for the three endogenous disturbance types considered in this paper: machine breakdowns, processing time variations, and yield losses. We also briefly discuss how to design \textsc{IsoFunc} and \textsc{PropFunc} for disturbances beyond our settings.

\textbf{Machine breakdown ($\ell_{1}$)}. Let $\ell_{1}$ denote the machine breakdown disturbance. Let also $\boldsymbol{I}^{\mathrm{En}, \ell_{1}}_{t_s (\mathcal{O}): t_e (\mathcal{O})}$ denote the vector whose set of elements equals $\{ U_{j, t} \in \boldsymbol{I}_{t_{0}: t_{N}}: j = j(\mathcal{O}) \land t_{s} (\mathcal{O}) \leq t \leq t_{e} (\mathcal{O}) \}$. Then, to set up the isolated impact level for operation $\mathcal{O}$ regarding $\ell_{1}$, we design \textsc{IsoFunc} as Equation \ref{eq:iso_func_machine_breakdown}.
\begin{subequations}
\begin{align} \label{eq:iso_func_machine_breakdown}
\begin{split}
    \underline{Z}_{\mathcal{O}}^{\ell_{1}} &= \textsc{IsoFunc} (\boldsymbol{I}^{\mathrm{En}, \ell_{1}}_{t_{s}(\mathcal{O}): t_{e}(\mathcal{O})}; \ell_{1}) \\
    &= \min \biggl\{1,  \sum_{t = t_{s} (\mathcal{O}) + 1}^{t_{e} (\mathcal{O})} U_{j(\mathcal{O}), t} \biggr\}
\end{split}
\end{align}
Recall that $U_{j, t} = 1$ if machine $j$ is unavailable during $[t-1, t)$ due to a breakdown. Then, if over the duration of $\mathcal{O}$, the machine $j$ is unavailable for at least one time interval (that is, $\sum_{t = t_{s} (\mathcal{O}) + 1}^{t_{e} (\mathcal{O})} U_{j(\mathcal{O}), t} \geq 1$), we have $\underline{Z}_{\mathcal{O}}^{\ell_{1}} = 1$. Otherwise $\sum_{t = t_{s} (\mathcal{O}) + 1}^{t_{e} (\mathcal{O})} U_{j(\mathcal{O}), t} = 0$, which means that the machine $j$ is not impacted by breakdowns across the duration of $\mathcal{O}$, we have $\underline{Z}_{\mathcal{O}}^{\ell_{1}} = 0$. As a result, the isolation variable $\underline{Z}_{\mathcal{O}}^{\ell}$ indicates whether the operation $\mathcal{O}$ is impacted by breakdowns over the duration $[t_s (\mathcal{O}), t_e (\mathcal{O})]$.

The propagation function for machine breakdowns is presented in Equation \ref{eq:prop_func_machine_breakdown}.
\begin{align} \label{eq:prop_func_machine_breakdown}
\begin{split}
    {Z}_{\mathcal{O}}^{\ell_{1}} &= \textsc{PropFunc}(\underline{Z}_{\mathcal{O}}^{\ell_{1}}, \boldsymbol{Z}_{\mathrm{Pa}^{\mathcal{G}^{\mathrm{Op}}} (\mathcal{O})}^{\ell_{1}}; \ell_{1}) \\
    &= \max \left\{ \underline{Z}_{\mathcal{O}}^{\ell_{1}}, \max_{\substack{\mathcal{O}' \in \mathrm{Pa}^{\mathcal{G}^{\mathrm{Op}}}(\mathcal{O}) \\ \mathbb{K}^{+} (\mathcal{O}') \cap \mathbb{K}^{-} (\mathcal{O}) \neq \varnothing } } Z_{\mathcal{O}'}^{\ell_{1}}  \right\}
\end{split}
\end{align}
As we mentioned before, within a local operation family, the impact on the child $\mathcal{O}$ consists of two parts. The first part, which is described by the isolated impact $\underline{Z}_{\mathcal{O}}^{\ell_{1}}$, represents the impact originating from the child $\mathcal{O}$ itself. The second part, which is described by the second term within the curly bracket, represents the impact propagated from parents. Specifically, consider a parent operation $\mathcal{O}' \in \mathrm{Pa}^{\mathcal{G}^{\mathrm{Op}}} (\mathcal{O})$ such that at least one input material of $\mathcal{O}$ is produced by $\mathcal{O}'$ (that is, $\mathbb{K}^{+} (\mathcal{O}') \cap \mathbb{K}^{-} (\mathcal{O}) \neq \varnothing$). If the parent $\mathcal{O}'$ is impacted by machine breakdowns (that is, $Z_{\mathcal{O}'}^{\ell_{1}} = 1$), the input material for $\mathcal{O}$ will not be produced by $\mathcal{O}'$ as scheduled. Therefore, we must have $Z_{\mathcal{O}}^{\ell_{1}} = 1$. The final impact on $\mathcal{O}$ equals the maximum value over its isolated impact and the impacts propagated from parents.
\end{subequations}

\textbf{Processing time variation ($\ell_{2}$)}. Let $\ell_{2}$ denote the processing time variation. Let also $\boldsymbol{I}^{\mathrm{En}, \ell_{2}}_{t_{s} (\mathcal{O}): t_{e} (\mathcal{O})}$ denote the vector whose set of elements equals $\{ V_{i, j, n, t} \in \boldsymbol{I}_{t_{0}: t_{N}}: i = i (\mathcal{O}) \land j = j(\mathcal{O}) \land t_{s}(\mathcal{O}) \leq t \leq t_{e}(\mathcal{O}) \}$. For this type of disturbances, we design the isolation function as,
\begin{subequations}
\begin{align} \label{eq:iso_func_processing_time_variation}
\begin{split}
    \underline{Z}_{\mathcal{O}}^{\ell_{2}} &= \textsc{IsoFunc} (\boldsymbol{I}_{t_{s} (\mathcal{O}): t_{e} (\mathcal{O})}^{\mathrm{En}, \ell_{2}}; \ell_{2} )  \\
    &= \lceil \tau_{i (\mathcal{O}), j(\mathcal{O}) } \cdot v_{i(\mathcal{O}), j(\mathcal{O}), n_{0}, t_s(\mathcal{O})} \rceil - \tau_{i (\mathcal{O}), j (\mathcal{O})}
\end{split}
\end{align}
where the first term $\lceil \tau_{i (\mathcal{O}), j(\mathcal{O}) } \cdot v_{i(\mathcal{O}), j(\mathcal{O}), n_{0}, t_s(\mathcal{O})} \rceil$ represents the adjusted processing time, and the second term $\tau_{i (\mathcal{O}), j (\mathcal{O})}$ represents the nominal processing time. As a result, the isolated impact $\underline{Z}_{\mathcal{O}}^{\ell_{2}}$ represents the number of time periods that the operation $\mathcal{O}$ delays due to the processing time variation disturbance.
\begin{align} \label{eq:prop_func_processing_time_variation}
\begin{split}
    {Z}_{\mathcal{O}}^{\ell_{2}} &= \textsc{PropFunc}(\underline{Z}_{\mathcal{O}}^{\ell_{2}}, \boldsymbol{Z}_{\mathrm{Pa}^{\mathcal{G}^{\mathrm{Op}}} (\mathcal{O})}^{\ell_{2}}; \ell_{2}) \\
    &= \max \left\{ \underline{Z}_{\mathcal{O}}^{\ell_{2}}, \max_{\mathcal{O}' \in \mathrm{Pa}^{\mathcal{G}^{\mathrm{Op}}} (\mathcal{O}) } \max \left\{ 0, \tau_{i(\mathcal{O}'), j(\mathcal{O}')} + Z_{\mathcal{O'}}^{\ell_{2}} - \bigl(t_{s}(\mathcal{O}) - t_{s}(\mathcal{O}') \bigr) \right\}  \right\}
\end{split}
\end{align}
\end{subequations}
Based on Equation \ref{eq:iso_func_processing_time_variation}, we design the \textsc{PropFunc} for $\ell_{2}$ as Equation \ref{eq:prop_func_processing_time_variation}. Although Equation \ref{eq:prop_func_processing_time_variation} seems cumbersome, it is intuitive to explain with the illustration of Figure \ref{fig:illustrative_example_propfunc_proc_time_variation}. In the case (a), where the delay in the parent $\mathcal{O}'$ does not suffice to affect the start time of the child $\mathcal{O}$ (that is, $t_{s} (\mathcal{O}) - t_{s} (\mathcal{O}') \geq \tau_{i (\mathcal{O}'), j (\mathcal{O}')} + Z_{\mathcal{O}'}^{\ell_{2}}$), the innermost $\max$ operator enforces the impact from this parent to $0$. Otherwise, in the case (b), where a delay in the parent $\mathcal{O}'$ forces the child $\mathcal{O}$ to be postponed (that is, $t_{s} (\mathcal{O}) - t_{s} (\mathcal{O}') < \tau_{i (\mathcal{O}'), j (\mathcal{O}')} + Z_{\mathcal{O}'}^{\ell_{2}}$), the impact propagated from this parent equals the delayed time. The middle $\max$ operator means that the delay in the child $\mathcal{O}$ depends on all of its parents. The outermost $\max$ operator determines the final impact by combining the isolated impact originating from $\mathcal{O}$ itself and the maximum impact from its parents.

\begin{figure}[H]
\includegraphics[width=0.85\textwidth]{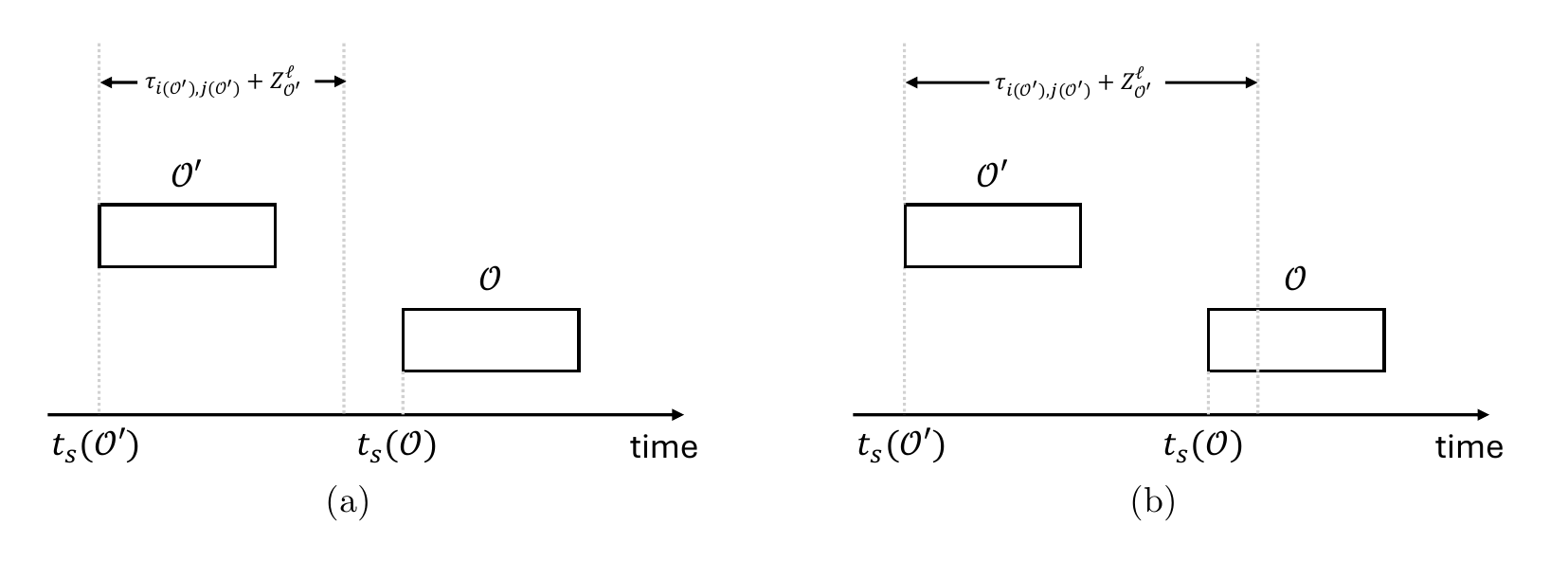}
\centering
\caption{
    Illustration of the propagation function for processing time variations
}
\label{fig:illustrative_example_propfunc_proc_time_variation}
\end{figure}

\textbf{Yield loss ($\ell_{3}$)}. Let $\ell_{3}$ denote the yield loss disturbance, and let $\boldsymbol{I}_{t_s(\mathcal{O}): t_e(\mathcal{O})}^{\mathrm{En}, \ell_{3}}$ denote the vector whose set of elements equals $\{ W_{i, j, n, t} \in \boldsymbol{I}_{t_0: t_{N}}: i = i(\mathcal{O}) \land j = j(\mathcal{O}) \land t_{s} (\mathcal{O}) \leq t \leq t_{e} (\mathcal{O}) \}$. Then, we design the \textsc{IsoFunc} for $\ell_{3}$ as
\begin{subequations}
\begin{align} \label{eq:iso_func_yield_loss}
\begin{split}
    \underline{Z}_{\mathcal{O}}^{\ell_{3}} &= \textsc{IsoFunc} (\boldsymbol{I}_{t_s(\mathcal{O}): t_e(\mathcal{O})}^{\mathrm{En}, \ell_{3}}; \ell_{3} )  \\
    &= \lceil \left( 1 - W_{i(\mathcal{O}), j(\mathcal{O}), n_{0}, t_{s} (\mathcal{O})} \right) \times 100 \rceil
\end{split}
\end{align}
where the realised value of $W_{i, j, n, t}$ represents the remaining percentage of yield after loss. As a result, the isolated impact can be interpreted as the rounded-up percentage of yield loss.
\begin{align} \label{eq:prop_func_yield_loss}
\begin{split}
    {Z}_{\mathcal{O}}^{\ell_{3}} &= \textsc{PropFunc}(\underline{Z}_{\mathcal{O}}^{\ell_{3}}, \boldsymbol{Z}_{\mathrm{Pa}^{\mathcal{G}^{\mathrm{Op}}} (\mathcal{O})}^{\ell_{3}}; \ell_{3}) \\
    &= \max \left\{ \underline{Z}_{\mathcal{O}}^{\ell_{3}}, \max_{\substack{\mathcal{O}' \in \mathrm{Pa}^{\mathcal{G}^{\mathrm{Op}}}(\mathcal{O}) \\ \mathbb{K}^{+} (\mathcal{O}') \cap \mathbb{K}^{-} (\mathcal{O}) \neq \varnothing } } Z_{\mathcal{O}'}^{\ell_{3}}  \right\}
\end{split}
\end{align}
\end{subequations}
The propagation function for yield losses is presented in Equation \ref{eq:prop_func_yield_loss}, which is similar to Equation \ref{eq:prop_func_machine_breakdown}. The final impact on the child $\mathcal{O}$ equals the maximum value over (1) the isolated yield loss from $\mathcal{O}$ itself, and (2) the yield loss from the parent operation $\mathcal{O}'$ such that $\mathcal{O}'$ produces at least one input material for $\mathcal{O}$.

\textbf{Extension to other disturbances}. As demonstrated by the above three examples, we can flexibly design \textsc{IsoFunc} and \textsc{PropFunc} to reflect the characteristics of disturbances. However, the following rules must be followed.
\begin{itemize}
    \item For \textsc{IsoFunc}, which evaluates the isolated impact on an operation $\mathcal{O}$, its argument can only involve disturbance variables related to $\mathcal{O}$ itself.
    \item For \textsc{PropFunc}, which quantifies the propagation effect within the local family $\{ \mathcal{O} \} \cup \mathrm{Pa}^{\mathcal{G}^{\mathrm{Op}}} (\mathcal{O})$, its arguments can only involve the isolated impact of $\mathcal{O}$ and the impact variables of parents of $\mathcal{O}$.
\end{itemize}
With these rules, we can extend the \textsc{ImpactPropagation} algorithm to other disturbance settings. To motivate readers, below we summarise some commonly seen disturbances and briefly introduce how to design \textsc{IsoFunc} and \textsc{PropFunc} for them.
\begin{itemize}
    \item Machine-related disturbances. In industrial settings, machine-related disturbances may involve unexpected machine breakdowns, performance degradations, planned maintenance, and so on. Typically, during the time window of these disturbances, no operations can be processed. Therefore, the isolation function \textsc{IsoFunc} can be designed as a binary function, which returns $1$ if the operation overlaps with the disturbance window and returns $0$ otherwise. Regarding the propagation effect, typically, a machine-related disturbance only affects operations subsequent in processing sequences, but does not affect operations scheduled on the same machine and starting after the disturbance window. To illustrate why, let $j$ denote the the machine on which the disturbance occurs, and let $[t_a, t_b]$ denote the disturbance window. Then, the operations starting later than $t_b$ on $j$ (that is, $\{\mathcal{O}: j(\mathcal{O}) = j \land t_{s}(\mathcal{O}) \geq t_b \}$) will not be affected. This is because, even if we cancel all the operations within the disturbance window $[t_a, t_b]$, the operations initiating later than $t_b$ on $j$ can still be executed once $j$ has been repaired. In contrast, the operations subsequent in a processing sequence may be affected. For example, consider the job shop scheduling, where we denote the $k$-th operation of job $i$ by $\mathcal{O}_{i, k}$. Let $j$ be the machine compatible with $\mathcal{O}_{i, k}$, and let us say $j$ is unavailable during the scheduled duration of $\mathcal{O}_{i, k}$. Then, subsequent operations of job $i$ (that is, $\mathcal{O}_{i, k+1}, \mathcal{O}_{i, k+2}, \cdots$) must be postponed because they rely on the completion of $\mathcal{O}_{i, k}$. As a result, the propagation function \textsc{PropFunc} need only consider upstream operations within the same processing sequence, but need not consider precedent operations on the same machine.
    \item Operation-related disturbances. Operation-related disturbances may include processing delays, out-of-specs, quality defects, and so on. While some of them may fail an operation, others may delay the completion time. For situations where an operation fails, the isolated impact and propagation effect can be designed similarly to those presented in machine-related disturbances, because the fail in an operation essentially equals to remove the operation from the schedule. For operation delays, the isolated impact can be designed as a quantity related to the actual time that an operation delays (as we did in Equation \ref{eq:iso_func_processing_time_variation}), while the propagation effect should consider both processing sequences and timings on the same machine. This is because, the extended duration of $\mathcal{O}$ may compete with the start times of later operations on $j(\mathcal{O})$. In other situations, the propagation effect of operation-related disturbances through the processing sequence also need to consider the inventory level of materials between operations. For example, assume an extreme case that in batch scheduling, the inventory levels of all materials are sufficient for initiating a large number of operations. Then, even if some midstream operations fail, downstream operations can nonetheless use the inventory of intermediate materials to initiate.
    \item Resource-related disturbances. In practice, some production processes also involve the use of shared resources. For example, a chemical plant may use steams to heat reactants or in an air unit, the process requires air compressors to increase the pressure. In principle, a shared resource can be viewed as a special type of machine.
    \item Inventory-related disturbances. In practice, inventory-related disturbances may involve mis-recording of inventory levels, inventory leakage, and so on. As an intermediate material, inventory leakage can affect operations whose input materials depends on that inventory. Therefore, the propagation effect of inventory-related disturbances can be viewed as operation-related disturbances, with the difference that the affected operations are the unions of operations depending on the leaked material.
\end{itemize}

%% file: sections/appendixF.tex
\section{Proofs} \label{appendix:proofs}

This appendix contains the proof of Theorem \ref{thm:impact_independence} from Section \ref{sec:theoretical_foundation}, together with the supporting lemmas and corollaries. Readers are referred to Section S1 of the Supplementary Information for the definitions of probabilistic independence, DAGs, and \acrshort*{bn}s that are used below.

\begin{lemma} \label{lemma:func_independent}
    Let $\boldsymbol{X}$ be a random vector whose elements are mutually independent under probability distribution $\mathsf{P}$. Let $\boldsymbol{X}_{1}, \cdots, \boldsymbol{X}_{m}$ be random vectors whose elements are disjoint subsets of $\boldsymbol{X}$, respectively. Let also $f_{1}, \cdots, f_{m}$ be a set of Borel measurable functions\footnotemark. Then, the random variables $Y_{1} = f_{1} (\boldsymbol{X}_{1})$, $\cdots$, $Y_{m} = f_{m} (\boldsymbol{X}_{m})$ are mutually independent under $\mathsf{P}$.
\end{lemma}
\footnotetext{A function $f: \mathbb{R}^{m} \rightarrow \mathbb{R}$ is Borel measurable if the preimage of every Borel set in $\mathbb{R}$ is also a Borel set in $\mathbb{R}^{m}$. Intuitively, this property ensures that $f(\boldsymbol{X})$ (where $\boldsymbol{X} \in \mathbb{R}^{m}$) is a well-defined random variable. Almost every function used in engineering practice (including continuous, piecewise continuous, and so on) is Borel measurable.}
\begin{proof}
    We prove this lemma using the set-based definition of probabilistic independence (recall the definitions in Section S1 of the Supplementary Information). Let $\mathbb{B}_{1}, \cdots, \mathbb{B}_{m} \subseteq \mathbb{R}$ be Borel sets. Then, we have
    \begin{align*}
        &\mathsf{P} \Bigl( \bigcap_{i=1}^{m} \{ f_{i} (\boldsymbol{X}_{i}) \in \mathbb{B}_{i} \} \Bigr) \\
        \overset{(1)}{=} \; &\mathsf{P} \Bigl( \bigcap_{i=1}^{m} \{ \boldsymbol{X}_{i} \in f_{i}^{-1} (\mathbb{B}_{i}) \} \Bigr) \\
        \overset{(2)}{=} \; &\prod_{i=1}^{m} \mathsf{P} \Bigl( \boldsymbol{X}_{i} \in f_{i}^{-1} (\mathbb{B}_{i}) \Bigr) \\
        \overset{(3)}{=} \; &\prod_{i=1}^{m} \mathsf{P} \Bigl( f_{i} (\boldsymbol{X}_{i}) \in \mathbb{B}_{i} \Bigr)
    \end{align*}
    Here, Equals sign $(1)$ follows from the definition of preimage. Note that because the functions $f_{1}, \cdots, f_{m}$ are Borel measurable, the sets $f_{1}^{-1}(\mathbb{B}_{1}), \cdots, f_{m}^{-1}(\mathbb{B}_{m})$ are also Borel sets. Equals sign $(2)$ follows from the definition of mutual independence of random variables. Equals sign $(3)$ follows again from the definition of preimage. Therefore, by the definition of mutual independence of random variables, the random variables $Y_{1}, \cdots, Y_{m}$ are mutually independent under distribution $\mathsf{P}$.
\end{proof}

Intuitively, Lemma \ref{lemma:func_independent} tells us that it remains mutually independent if we take functions over a set of mutually independent random variables. Having this lemma, we can immediately obtain a corollary about isolation variables.

\begin{corollary} \label{cor:isolation_independent}
    Let $\mathbb{O}$ denote a collection of non-overlapping operations in a schedule. Let $\mathbb{L}^{\mathrm{En}}$ denote the set of endogenous disturbance types. Assume that elements in the random vector $\boldsymbol{I}_{t_{0}: t_{N}}$ are mutually independent under probability distribution $\mathsf{P}$. For each $\mathcal{O} \in \mathbb{O}$ and each $\ell \in \mathbb{L}^{\mathrm{En}}$, define the isolation variable as $\underline{Z}_{\mathcal{O}}^{\ell} = \textsc{IsoFunc} (\boldsymbol{I}_{t_{s} (\mathcal{O}) : t_{e} (\mathcal{O})}; \ell )$, where the isolation function $\textsc{IsoFunc}$ is a Borel measurable function. Then, for any two distinct operations $\mathcal{O}, \mathcal{O}' \in \mathbb{O}$, we have
    \begin{align*}
        \mathsf{P} \models \underline{Z}_{\mathcal{O}}^{\ell} \independent \underline{Z}_{\mathcal{O}'}^{\ell}.
    \end{align*}
\end{corollary}
\begin{proof}
    Since the collection $\mathbb{O}$ contains only non-overlapping operations, the random variables contained in $\boldsymbol{I}^{\mathrm{En}, \ell}_{t_s(\mathcal{O}): t_e(\mathcal{O})}$ and $\boldsymbol{I}^{\mathrm{En}, \ell}_{t_s(\mathcal{O}'): t_e(\mathcal{O}')}$ are disjoint. Since the elements in $\boldsymbol{I}_{t_0: t_N}$ are mutually independent by assumption, these disjoint subsets of random variables are also independent. Therefore, by applying Lemma \ref{lemma:func_independent}, we obtain $\mathsf{P} \models \underline{Z}_{\mathcal{O}}^{\ell} \independent \underline{Z}_{\mathcal{O}'}^{\ell}$.
\end{proof}

Corollary \ref{cor:isolation_independent} ensures that the isolation variables that we constructed from the isolation function are mutually independent under the probability distribution $\mathsf{P}$. This property will form the basis for our final result. Next, we need some extra lemmas related to either conditional independence and \acrshort*{dag}s.

\begin{lemma} \label{lemma:func_conditional_independent}
    Let $\boldsymbol{X}_{1}$, $\boldsymbol{X}_{2}$, $\boldsymbol{X}_{3}$ be random vectors that satisfy $\mathsf{P} \models \boldsymbol{X}_{1} \independent \boldsymbol{X}_{2} \mid \boldsymbol{X}_{3}$. Let also $f_{1}$, $f_{2}$, and $f_{3}$ be Borel measurable functions. Assume $Y_{1} = f_{1} (\boldsymbol{X}_{1})$, $Y_{2} = f_{2} (\boldsymbol{X}_{2})$, $Y_{3} = f_{3} (\boldsymbol{X}_{3})$. Then, we have
    \begin{align*}
        \mathsf{P} \models Y_{1} \independent Y_{2} \mid Y_{3}.
    \end{align*}
\end{lemma}
\begin{proof}
    We prove this lemma also by the set-based definition of conditional independence. Let $\mathbb{B}_{1}, \mathbb{B}_{2}, \mathbb{B}_{3} \subseteq \mathbb{R}$ be Borel sets. Then, we have
    \begin{align*}
        &\mathsf{P} \Bigl( Y_{1} \in \mathbb{B}_{1}, Y_{2} \in \mathbb{B}_{2} \mid Y_{3} \in \mathbb{B}_{3} \Bigr) \\
        \overset{(1)}{=} \; &\mathsf{P} \Bigl( \boldsymbol{X}_{1} \in f_{1}^{-1} (\mathbb{B}_{1}), \boldsymbol{X}_{2} \in f_{2}^{-1} (\mathbb{B}_{2}) \mid \boldsymbol{X}_{3} \in f_{3}^{-1} (\mathbb{B}_{3}) \Bigr) \\
        \overset{(2)}{=} \; &\mathsf{P} \Bigl( \boldsymbol{X}_{1} \in f_{1}^{-1} (\mathbb{B}_{1}) \mid \boldsymbol{X}_{3} \in f_{3}^{-1} (\mathbb{B}_{3}) \Bigr) \cdot \mathsf{P} \Bigl( \boldsymbol{X}_{2} \in f_{2}^{-1} (\mathbb{B}_{2}) \mid \boldsymbol{X}_{3} \in f_{3}^{-1} (\mathbb{B}_{3}) \Bigr) \\
        \overset{(3)}{=} \; &\mathsf{P} \Bigl( Y_{1} \in \mathbb{B}_{1} \mid Y_{3} \in \mathbb{B}_{3} \Bigr) \cdot \mathsf{P} \Bigl( Y_{2} \in \mathbb{B}_{2} \mid Y_{3} \in \mathbb{B}_{3} \Bigr)
    \end{align*}
    Here, Equals sign $(1)$ follows because the functions $f_{1}, f_{2}, f_{3}$ are Borel measurable. Equals sign $(2)$ follows from the definition of conditional independence. Equals sign $(3)$ follows from the definition of preimage.
\end{proof}

Intuitively, Lemma \ref{lemma:func_conditional_independent} ensures that functions of conditionally independent random variables remain conditionally independent.

\begin{lemma} \label{lemma:disjoint_cond_independent}
    Let $\boldsymbol{X}$ be a random vector whose elements are mutually independent under $\mathsf{P}$. Let also $\boldsymbol{X}_{1}$, $\boldsymbol{X}_{2}$, and $\boldsymbol{X}_{3}$ denote disjoint sets of random variables in $\boldsymbol{X}$. Then, we have
    \begin{align*}
        \mathsf{P} \models \boldsymbol{X}_{1} \independent \boldsymbol{X}_{2} \mid \boldsymbol{X}_{3}.
    \end{align*}
\end{lemma}
\begin{proof}
    We prove this lemma using the set-based definition of conditional independence as follows.
    {\allowdisplaybreaks
    \begin{align*}
        &\mathsf{P} \Bigl( \boldsymbol{X}_{1} \in \mathbb{B}_{1} \mid \boldsymbol{X}_{2} \in \mathbb{B}_{2}, \boldsymbol{X}_{3} \in \mathbb{B}_{3} \Bigr) \\
        \overset{(1)}{=} \; &\frac{ \mathsf{P}\Bigl( \boldsymbol{X}_{1} \in \mathbb{B}_{1}, \boldsymbol{X}_{2} \in \mathbb{B}_{2}, \boldsymbol{X}_{3} \in \mathbb{B}_{3} \Bigr) }{ \mathsf{P}\Bigl( \boldsymbol{X}_{1} \in \mathbb{B}_{1}, \boldsymbol{X}_{2} \in \mathbb{B}_{2} \Bigr) } \\
        \overset{(2)}{=} \; &\frac{ \mathsf{P}\Bigl( \boldsymbol{X}_{1} \in \mathbb{B}_{1} \Bigr) \cdot \mathsf{P}\Bigl( \boldsymbol{X}_{2} \in \mathbb{B}_{2} \Bigr) \cdot \mathsf{P}\Bigl( \boldsymbol{X}_{3} \in \mathbb{B}_{3} \Bigr)  }{ \mathsf{P}\Bigl( \boldsymbol{X}_{2} \in \mathbb{B}_{2} \Bigr) \cdot \mathsf{P}\Bigl( \boldsymbol{X}_{3} \in \mathbb{B}_{3} \Bigr) } \\
        = \; &\mathsf{P} \Bigl( \boldsymbol{X}_{1} \in \mathbb{B}_{1} \Bigr) \\
        = \; &\frac{ \mathsf{P}\Bigl( \boldsymbol{X}_{1} \in \mathbb{B}_{1} \Bigr) \cdot \mathsf{P}\Bigl( \boldsymbol{X}_{3} \in \mathbb{B}_{3} \Bigr) }{ \mathsf{P} (\boldsymbol{X}_{3} \in \mathbb{B}_{3}) } \\
        \overset{(3)}{=} \; &\frac{ \mathsf{P}\Bigl( \boldsymbol{X}_{1} \in \mathbb{B}_{1}, \boldsymbol{X}_{3} \in \mathbb{B}_{3} \Bigr) }{ \mathsf{P} \Bigl(\boldsymbol{X}_{3} \in \mathbb{B}_{3} \Bigr) }  \\
        \overset{(4)}{=} \; &\mathsf{P} \Bigl( \boldsymbol{X}_{1} \in \mathbb{B}_{1} \mid \boldsymbol{X}_{3} \in \mathbb{B}_{3} \Bigr)
    \end{align*}
    }
    Here, Equals sign $(1)$ follows from the definition of conditional probability. Equals sign $(2)$ follows because random vectors $X_{1}$, $X_{2}$, $X_{3}$ are disjoint subsets of mutually independent random variables $\boldsymbol{X}$. Equals sign $(3)$ follows again because $\boldsymbol{X}_{1}$ and $\boldsymbol{X}_{3}$ are disjoint in $\boldsymbol{X}$. Equals sign $(4)$ follows from the definition of conditional probability.
\end{proof}

\begin{lemma} \label{lemma:dag_inclusion}
    Let $\mathcal{G} = (\mathbb{X}, \mathbb{A})$ be a directed acyclic graph. Let $\mathrm{H}(\mathcal{G})$ denote the height of $\mathcal{G}$. Let also $\mathbb{H}^{\mathcal{G}}_{j}$ denote the collection of nodes with height $j$ in $\mathcal{G}$. Then, for every $k \in \{0, 1, \cdots, \mathrm{H}(\mathcal{G}) - 1 \}$ and every node $X \in \mathbb{H}^{\mathcal{G}}_{k+1}$, we have
    \begin{align*}
        \mathrm{Pa}^{\mathcal{G}}(X)
        \subseteq \bigcup_{i=1}^{k+1} (\mathrm{Pa}^{\mathcal{G}})^{(i)} (X)
        \subseteq \bigcup_{j=0}^{k} \mathbb{H}_{j}^{\mathcal{G}}
        \subseteq \overline{\mathrm{De}}^{\mathcal{G}}(X).
    \end{align*}
\end{lemma}
\begin{proof}
    Intuitively, this lemma states the following chain of inclusions: the parent nodes of $X$, the recursive parent nodes of $X$, the collection of nodes whose height is less than $k+1$, and the non-descendant nodes of $X$. We prove these inclusions one by one.

    The first inclusion $\mathrm{Pa}^{\mathcal{G}}(X) \subseteq \bigcup_{i=1}^{k+1} (\mathrm{Pa}^{\mathcal{G}})^{(i)} (X)$ follows directly from the definition of $n$-fold parents (see Section S1 of the Supplementary Information).

    We prove the second inclusion $\bigcup_{i=1}^{k+1} (\mathrm{Pa}^{\mathcal{G}})^{(i)} (X)  \subseteq \bigcup_{j=0}^{k} \mathbb{H}_{j}^{\mathcal{G}}$ by contradiction. Assume that $\bigcup_{i=1}^{k+1} (\mathrm{Pa}^{\mathcal{G}})^{(i)} (X)  \nsubseteq \bigcup_{j=0}^{k} \mathbb{H}_{j}^{\mathcal{G}}$. Then, this implies that there exists a node $X' \in \mathbb{X}$ such that $X' \in \bigcup_{i=1}^{k+1} (\mathrm{Pa}^{\mathcal{G}})^{(i)} (X)$ but $X' \notin \bigcup_{j=0}^{k} \mathbb{H}_{j}^{\mathcal{G}}$. The latter implies that $\mathrm{H}^{\mathcal{G}} (X') \geq k+1$. However, since we assume that  $X' \in \bigcup_{i=1}^{k+1} (\mathrm{Pa}^{\mathcal{G}})^{(i)} (X)$, it follows from $X \in \mathbb{H}_{k+1}^{\mathcal{G}}$ that we must have $\mathrm{H}^{\mathcal{G}}(X') \leq k$. Therefore, we arrive at a contradiction, which proves the second inclusion.

    We prove the third inclusion $\bigcup_{j=0}^{k} \mathbb{H}_{j}^{\mathcal{G}} \subseteq \overline{\mathrm{De}}^{\mathcal{G}}(X)$ also by contradiction. Assume that $\bigcup_{j=0}^{k} \mathbb{H}_{j}^{\mathcal{G}} \nsubseteq \overline{\mathrm{De}}^{\mathcal{G}}(X)$, which implies that there exists a node $X' \in \bigcup_{j=0}^{k} \mathbb{H}^{\mathcal{G}}_{j}$ such that $X' \notin \overline{\mathrm{De}}^{\mathcal{G}}(X)$. In other words, the node $X'$ is a descendant of $X$ in $\mathcal{G}$. Therefore, a path $X' \rightarrow \cdots \rightarrow X$ should exist in $\mathcal{G}$. However, this contradicts with $X' \in \bigcup_{j=0}^{k} \mathbb{H}^{\mathcal{G}}_{j}$. Therefore, we prove the third inclusion.
\end{proof}

\begin{lemma} \label{lemma:dag_disjoint}
    Let $\mathcal{G} = (\mathbb{X}, \mathbb{A})$ be a directed acyclic graph. Let also $X \in \mathbb{X}$ and $X' \in \overline{\mathrm{De}}^{\mathcal{G}} (X) \cap \mathbb{H}_{j}^{\mathcal{G}}$ be two nodes in $\mathcal{G}$, where $j$ is an integer such that $0 \leq j \leq \mathrm{H}(\mathcal{G})$. Then we have
    \begin{align*}
        \{ X \} \cap \bigcup_{i=0}^{j} (\mathrm{Pa}^{\mathcal{G}})^{(i)} (X') = \varnothing.
    \end{align*}
\end{lemma}
\begin{proof}
    When $i=0$, the case is trivial because a non-descendant node of $X$ cannot be $X$ itself. For the case when $i = \{1, \cdots, j \}$, we prove by contradiction. Assume that $\{ X \} \cap \bigcup_{i=1}^{j} (\mathrm{Pa}^{\mathcal{G}})^{(i)} (X') \neq \varnothing$. This implies that $X \in \bigcup_{i=1}^{j} (\mathrm{Pa}^{\mathcal{G}})^{(i)} (X')$. Therefore, a path $X \rightarrow \cdots \rightarrow X'$ should exist in $\mathcal{G}$. However, since $X' \in \overline{\mathrm{De}} (X)$, which by definition implies that the path $X \rightarrow \cdots \rightarrow X'$ does not exist in $\mathcal{G}$, we arrive at a contradiction.
\end{proof}

Having the above lemmas prepared, we are now ready to prove an important lemma for impact variables.
\begin{lemma} \label{lemma:impact_as_function}
    Let $\mathcal{G} = (\mathbb{O}_{t+h^{\mathrm{C}}: t+h_{t}}, \mathbb{A})$ be a \acrshort*{dag} that represents the dependencies between operations in $\mathbb{O}_{t+h^{\mathrm{C}}: t+h_{t}}$. For each operation $\mathcal{O} \in \mathbb{O}_{t + h^{\mathrm{C}}: t + h_{t} \mid t}$ and endogenous disturbance type $\ell \in \mathbb{L}^{\mathrm{En}}$, construct the random variable $Z_{\mathcal{O}}^{\ell} = \textsc{PropFunc} (\underline{Z}_{\mathcal{O}}^{\ell}, \boldsymbol{Z}_{\mathrm{Pa}^{\mathcal{G}} (\mathcal{O})}^{\ell} ; \ell )$ according to the \textsc{ImpactPropagation} algorithm (that is, Algorithm \ref{algo:impact_propagation} in Appendix \ref{appendix:algorithms}). Then, for each operation $\mathcal{O} \in \mathbb{O}_{t + h^{\mathrm{C}}: t + h_{t}}$, its corresponding impact variable $Z_{\mathcal{O}}^{\ell}$ is a Borel measurable function of the following set of isolation variables.
    \begin{align*}
        \Bigl\{ \underline{Z}_{\mathcal{O}}^{\ell}: \mathcal{O} \in \bigcup_{i=0}^{\mathrm{H}^{\mathcal{G}} (\mathcal{O})} (\mathrm{Pa}^{\mathcal{G}})^{(i)} (\mathcal{O}) \Bigr\}.
    \end{align*}
\end{lemma}
\begin{proof}
    We prove this lemma by induction. Let $j$ denote the case when $\mathrm{H}^{\mathcal{G}} (\mathcal{O}) = j$. For the base case $j = 0$, since there exist no parent operations of $\mathcal{O}$ in $\mathcal{G}$, the impact variable will be defined as $Z_{\mathcal{O}}^{\ell} = \textsc{PropFunc} (\underline{Z}_{\mathcal{O}}^{\ell}, \varnothing; \ell)$. That is, the impact variable $Z_{\mathcal{O}}^{\ell}$ is only a function of its corresponding isolation variable $\underline{Z}^{\ell}_{\mathcal{O}}$. Recall by convention, we define $(\mathrm{Pa}^{\mathcal{G}})^{(0)} (\mathcal{O}) = \{ \mathcal{O} \}$. Therefore, we prove the base case.

    Assume that the statement holds for cases when $j = 1, \cdots, k$, where $k < \mathrm{H}({\mathcal{G}})$. Now we consider the case $j = k+1$, that is, $\mathrm{H}^{\mathcal{G}} (\mathcal{O}) = k+1$. According to the definition of the impact variables, we have $Z_{\mathcal{O}}^{\ell} = \textsc{PropFunc} (\underline{Z}^{\ell}_{\mathcal{O}}, \boldsymbol{Z}^{\ell}_{\mathrm{Pa}^{\mathcal{G}} (\mathcal{O}) } ; \ell)$. For the second argument, from Lemma \ref{lemma:dag_inclusion} we know that $\mathrm{Pa}^{\mathcal{G}} (\mathcal{O}) \subseteq \bigcup_{i=0}^{k} \mathbb{H}^{\mathcal{G}}_{k}$. Therefore, applying the inductive assumption, we have that the second argument $\boldsymbol{Z}_{\mathrm{Pa}^{\mathcal{G}} (\mathcal{O}) }^{\ell}$ is a function of the following set of isolation variables,
    \begin{align*}
        \Bigl\{ \underline{Z}^{\ell}_{\mathcal{O}''}: \mathcal{O}'' \in \bigcup_{\mathcal{O}' \in \mathrm{Pa}^{\mathcal{G}} (\mathcal{O}) } \bigcup_{i=0}^{\mathrm{H}^{\mathcal{G}} (\mathcal{O}') } (\mathrm{Pa}^{\mathcal{G}})^{(i)}  (\mathcal{O}') \Bigr\},
    \end{align*}
    which, by the definition of recursive parents (note that $\mathrm{H}^{\mathcal{G}} (\mathcal{O}) = k+1$ and $\mathcal{O}' \in \mathrm{Pa}^{\mathcal{G}} (\mathcal{O})$), equals
    \begin{align*}
        \Bigl\{ \underline{Z}^{\ell}_{\mathcal{O}''}: \mathcal{O}'' \in \bigcup_{i=1}^{\mathrm{H}^{\mathcal{G}} (\mathcal{O}) } (\mathrm{Pa}^{\mathcal{G}})^{(i)}  (\mathcal{O}) \Bigr\}.
    \end{align*}
    Therefore, in the case $j = k+1$, the impact variable $Z_{\mathcal{O}}^{\ell}$ is a function of
    \begin{align*}
        \{ \underline{Z}_{\mathcal{O}}^{\ell} \} \cup \Bigl\{ \underline{Z}^{\ell}_{\mathcal{O}''}: \mathcal{O}'' \in \bigcup_{i=1}^{\mathrm{H}^{\mathcal{G}} (\mathcal{O}) } (\mathrm{Pa}^{\mathcal{G}})^{(i)}  (\mathcal{O}) \Bigr\},
    \end{align*}
    which again, by the convention $(\mathrm{Pa}^{\mathcal{G}})^{(0)} (\mathcal{O}) = \{ \mathcal{O} \}$, equals
    \begin{align*}
        \Bigl\{ \underline{Z}_{\mathcal{O}}^{\ell}: \mathcal{O} \in \bigcup_{i=0}^{\mathrm{H}^{\mathcal{G}} (\mathcal{O})} (\mathrm{Pa}^{\mathcal{G}})^{(i)} (\mathcal{O}) \Bigr\}.
    \end{align*}
    Therefore, we prove the inductive case and therefore the lemma.
\end{proof}

Intuitively, Lemma \ref{lemma:impact_as_function} states that within the scope of isolation variables, an impact variable only depends on those isolation variables corresponding to its recursive parents. Having this property, we can immediately obtain the following conditional independence statement.

\begin{corollary} \label{cor:impact_and_isolation}
    Let $\mathcal{G} = (\mathbb{O}_{t+h^{\mathrm{C}}: t+h_{t}}, \mathbb{A})$ be a \acrshort*{dag} that represents the dependencies between operations in $\mathbb{O}_{t+h^{\mathrm{C}}: t+h_{t}}$. For each operation $\mathcal{O} \in \mathbb{O}_{t + h^{\mathrm{C}}: t + h_{t} \mid t}$ and endogenous disturbance type $\ell \in \mathbb{L}^{\mathrm{En}}$, construct the random variable $Z_{\mathcal{O}}^{\ell} = \textsc{PropFunc} (\underline{Z}_{\mathcal{O}}^{\ell}, \boldsymbol{Z}_{\mathrm{Pa}^{\mathcal{G}} (\mathcal{O})}^{\ell} ; \ell )$ according to the \textsc{ImpactPropagation} algorithm. Then, for each operation $\mathcal{O} \in \mathbb{O}_{t + h^{\mathrm{C}}: t + h_{t}}$, we have
    \begin{align*}
        \mathsf{P} \models Z_{\mathcal{O}}^{\ell} \independent \underbrace{\Bigl\{ \underline{Z}^{\ell}_{\mathcal{O}'}: \mathcal{O}' \in \mathbb{O}_{t + h^{\mathrm{C}}: t + h_{t}} \setminus \bigcup_{i=0}^{\mathrm{H}^{\mathcal{G}} (\mathcal{O})  } (\mathrm{Pa}^{\mathcal{G}})^{(i)} (\mathcal{O}) \Bigr\}}_{(a)} \mid \underbrace{\Bigl\{ \underline{Z}^{\ell}_{\mathcal{O}'}: \mathcal{O}' \in \bigcup_{i=1}^{\mathrm{H}^{\mathcal{G}} (\mathcal{O}) } (\mathrm{Pa}^{\mathcal{G}})^{(i)} (\mathcal{O}) \Bigr\}}_{(b)},
    \end{align*}
    where the conditioning term $(b)$ represents the isolation variables corresponding to the exclusive recursive parents\footnotemark of $\mathcal{O}$ in $\mathcal{G}$, and the second term $(a)$ represents the isolation variables corresponding to the complement of inclusive recursive parents of $\mathcal{O}$ in $\mathcal{G}$.
\end{corollary}
\footnotetext{For convenience, we call the set $\bigcup_{i=0}^{\mathrm{H}^{\mathcal{G}} (\mathcal{O}) } (\mathcal{O}) $ the \textit{inclusive recursive parents} of $\mathcal{O}$ (because it includes the operation $\mathcal{O}$ itself), while call the set $\bigcup_{i=1}^{\mathrm{H}^{\mathcal{G}} (\mathcal{O}) } (\mathcal{O}) $ the \textit{exclusive recursive parents} of $\mathcal{O}$ (because it excludes the operation $\mathcal{O}$ itself).}
\begin{proof}
    From Lemma \ref{lemma:impact_as_function}, we know that the impact variable $Z_{\mathcal{O}}^{\ell}$ is a function of the following isolation variables
    \begin{align} \label{eq:set_of_recursive_parent_isolation_variable}
        \{ \underline{Z}_{\mathcal{O}}^{\ell}: \mathcal{O} \in \bigcup_{i=0}^{\mathrm{H}^{\mathcal{G}} (\mathcal{O})} (\mathrm{Pa}^{\mathcal{G}})^{(i)} (\mathcal{O}) \}.
    \end{align}
    By excluding the conditioning term $(b)$ from Equation \ref{eq:set_of_recursive_parent_isolation_variable}, we remain to show that $\{ \underline{Z}_{\mathcal{O}}^{\ell} \}$ and the second term $(a)$ are disjoint sets in $\{ \underline{Z}_{\mathcal{O}}^{\ell}: \mathcal{O} \in \mathbb{O}_{t + h^{\mathrm{C}}: t + h_{t}} \}$. This statement follows immediately from Lemma \ref{lemma:dag_disjoint}. Therefore, applying Lemmas \ref{lemma:disjoint_cond_independent} and \ref{lemma:func_conditional_independent}, we complete the proof.
\end{proof}

Now we are ready to present the proof of Theorem \ref{thm:impact_independence}.

\begin{proof}[Proof of Theorem \ref{thm:impact_independence}]
    \allowdisplaybreaks
    \begin{subequations} \label{eq:derivation_conditional_independence_theorem}
    \begin{align}
        & \; Z_{\mathcal{O}}^{\ell} \independent \Bigl\{ \underline{Z}^{\ell}_{\mathcal{O}'}: \mathcal{O}' \in \mathbb{O}_{t + h^{\mathrm{C}}: t + h_{t}} \setminus \bigcup_{i=0}^{\mathrm{H}^{\mathcal{G}} (\mathcal{O})  } (\mathrm{Pa}^{\mathcal{G}})^{(i)} (\mathcal{O}) \Bigr\} \mid  \Bigl\{ \underline{Z}^{\ell}_{\mathcal{O}'}: \mathcal{O}' \in \bigcup_{i=1}^{\mathrm{H}^{\mathcal{G}} (\mathcal{O}) } (\mathrm{Pa}^{\mathcal{G}})^{(i)} (\mathcal{O}) \Bigr\} \label{eq:derivation_conditional_independence_theorem_1} \\
        \implies & \;  Z_{\mathcal{O}}^{\ell} \independent \Bigl\{ \underline{Z}^{\ell}_{\mathcal{O}'}: \mathcal{O}' \in \Bigl( \mathbb{O}_{t + h^{\mathrm{C}}: t + h_{t}} \setminus \bigcup_{i=0}^{\mathrm{H}^{\mathcal{G}} (\mathcal{O})  } (\mathrm{Pa}^{\mathcal{G}})^{(i)} (\mathcal{O}) \Bigr) \cap \overline{\mathrm{De}}^{\mathcal{G}} (\mathcal{O})  \Bigr\} \mid \Bigl\{ \underline{Z}^{\ell}_{\mathcal{O}'}: \mathcal{O}' \in \bigcup_{i=1}^{\mathrm{H}^{\mathcal{G}} (\mathcal{O}) } (\mathrm{Pa}^{\mathcal{G}})^{(i)} (\mathcal{O}) \Bigr\} \label{eq:derivation_conditional_independence_theorem_2} \\
        \implies & \;  Z_{\mathcal{O}}^{\ell} \independent \Bigl\{ \underline{Z}^{\ell}_{\mathcal{O}'}: \mathcal{O}' \in \overline{\mathrm{De}}^{\mathcal{G}} (\mathcal{O}) \setminus \Bigl( \bigcup_{i=0}^{\mathrm{H}^{\mathcal{G}} (\mathcal{O})  } (\mathrm{Pa}^{\mathcal{G}})^{(i)} (\mathcal{O}) \Bigr) \Bigr\} \mid \Bigl\{ \underline{Z}^{\ell}_{\mathcal{O}'}: \mathcal{O}' \in \bigcup_{i=1}^{\mathrm{H}^{\mathcal{G}} (\mathcal{O}) } (\mathrm{Pa}^{\mathcal{G}})^{(i)} (\mathcal{O}) \Bigr\} \label{eq:derivation_conditional_independence_theorem_3} \\
        \implies & \;  Z_{\mathcal{O}}^{\ell} \independent \Bigl\{ \underline{Z}^{\ell}_{\mathcal{O}'}: \mathcal{O}' \in \overline{\mathrm{De}}^{\mathcal{G}} (\mathcal{O}) \setminus \Bigl( \bigcup_{i=1}^{\mathrm{H}^{\mathcal{G}} (\mathcal{O})  } (\mathrm{Pa}^{\mathcal{G}})^{(i)} (\mathcal{O}) \Bigr) \Bigr\} \mid \Bigl\{ \underline{Z}^{\ell}_{\mathcal{O}'}: \mathcal{O}' \in \bigcup_{i=1}^{\mathrm{H}^{\mathcal{G}} (\mathcal{O}) } (\mathrm{Pa}^{\mathcal{G}})^{(i)} (\mathcal{O}) \Bigr\} \label{eq:derivation_conditional_independence_theorem_4} \\
        \implies & \;  Z_{\mathcal{O}}^{\ell} \independent \Bigl\{ \underline{Z}^{\ell}_{\mathcal{O}'}: \mathcal{O}' \in  \overline{\mathrm{De}}^{\mathcal{G}} (\mathcal{O}) \Bigr\} \mid \Bigl\{ \underline{Z}^{\ell}_{\mathcal{O}'}: \mathcal{O}' \in \bigcup_{i=1}^{\mathrm{H}^{\mathcal{G}} (\mathcal{O}) } (\mathrm{Pa}^{\mathcal{G}})^{(i)} (\mathcal{O}) \Bigr\}  \label{eq:derivation_conditional_independence_theorem_5}
    \end{align}
    \end{subequations}
    Here, Equation \ref{eq:derivation_conditional_independence_theorem_1} follows directly from Corollary \ref{cor:impact_and_isolation}. Equation \ref{eq:derivation_conditional_independence_theorem_2} follows from the decomposition rule of conditional independence (that is, $\boldsymbol{X}_1 \independent \boldsymbol{X}_{2}, \boldsymbol{X}_{3} \mid \boldsymbol{X}_{4} \implies \boldsymbol{X}_1 \independent \boldsymbol{X}_{2} \mid \boldsymbol{X}_{4}$). Note that the intersection operation always yields a smaller set than the original. Equation \ref{eq:derivation_conditional_independence_theorem_3} follows from the inclusions proved in Lemma \ref{lemma:dag_inclusion} and De Morgan's law. Equation \ref{eq:derivation_conditional_independence_theorem_4} follows from Lemma \ref{lemma:dag_disjoint} (which implies $\overline{\mathrm{De}}^{\mathcal{G}} (\mathcal{O}) \cap \{ \mathcal{O} \} = \varnothing$). Equation \ref{eq:derivation_conditional_independence_theorem_5} follows from the convention that we assume $\boldsymbol{X}_{1} \independent \boldsymbol{X}_{2} \mid \boldsymbol{X}_{3}$ actually means $(\boldsymbol{X}_{1} \setminus \boldsymbol{X}_{3}) \independent (\boldsymbol{X}_{2} \setminus \boldsymbol{X}_{3}) \mid \boldsymbol{X}_{3}$.

    Next, we show that (1) the impact variables $\boldsymbol{Z}_{\overline{\mathrm{De}}^{\mathcal{G}} (\mathcal{O})}^{\ell}$ are a function of $\{ \underline{Z}^{\ell}_{\mathcal{O}'}: \mathcal{O}' \in  \overline{\mathrm{De}}^{\mathcal{G}} (\mathcal{O}) \}$ (that is, the second term in Equation \ref{eq:derivation_conditional_independence_theorem_5}), and (2) the impact variables $\boldsymbol{Z}_{\mathrm{Pa}^{\mathcal{G}} (\mathcal{O})}^{\ell}$ are a function of $\{ \underline{Z}^{\ell}_{\mathcal{O}'}: \mathcal{O}' \in \bigcup_{i=1}^{\mathrm{H}^{\mathcal{G}} (\mathcal{O}) } (\mathrm{Pa}^{\mathcal{G}})^{(i)} (\mathcal{O}) \}$ (that is, the conditioning term in Equation \ref{eq:derivation_conditional_independence_theorem_5}). For the statement (1), from Lemma \ref{lemma:impact_as_function} we know that for any $\mathcal{O}' \in \overline{\mathrm{De}}^{\mathcal{G}} (\mathcal{O})$, its corresponding impact variable $Z_{\mathcal{O}'}^{\ell}$ is a function of $\{ \underline{Z}^{\ell}_{\mathcal{O}''}: \mathcal{O}'' \in \bigcup_{i=0}^{\mathrm{H}^{\mathcal{G}} (\mathcal{O}')} (\mathrm{Pa}^{\mathcal{G}})^{(i)} (\mathcal{O}') \}$. Therefore, taking the union of $\mathcal{O}'$ over $\overline{\mathrm{De}}^{\mathcal{G}} (\mathcal{O})$, we have that the impact variables $\boldsymbol{Z}^{\ell}_{\overline{\mathrm{De}}^{\mathcal{G} } (\mathcal{O}) }$ are a function of
    \begin{align*}
        \Bigl\{ \underline{Z}_{\mathcal{O}'}^{\ell}: \mathcal{O}' \in \bigcup_{\mathcal{O}'' \in \overline{\mathrm{De}}^{\mathcal{G}} (\mathcal{O})} \bigcup_{i=0}^{\mathrm{H}^{\mathcal{G}} (\mathcal{O}'') } (\mathrm{Pa}^{\mathcal{G}})^{(i)} (\mathcal{O}'') \Bigr\},
    \end{align*}
    which, by the inclusions implied by Lemma \ref{lemma:dag_inclusion}, is a subset of
    \begin{align*}
        \Bigl\{ \underline{Z}^{\ell}_{\mathcal{O}'} : \mathcal{O}' \in \bigcup_{\mathcal{O}'' \in \overline{\mathrm{De}}^{\mathcal{G}} (\mathcal{O})} \overline{\mathrm{De}}^{\mathcal{G}} (\mathcal{O}'') \Bigr\},
    \end{align*}
    which is also a subset of
    \begin{align*}
        \Bigl\{ \underline{Z}^{\ell}_{\mathcal{O}'} : \mathcal{O}' \in \overline{\mathrm{De}}^{\mathcal{G}} (\mathcal{O}') \Bigr\},
    \end{align*}
    because a non-descendant of non-descendant of $\mathcal{O}'$ is also a non-descendant of $\mathcal{O}'$. Therefore, we prove the statement (1).

    For the statement (2), following the similar ideas used in proving statement (1), the impact variables $\boldsymbol{Z}^{\ell}_{\mathrm{Pa}^{\mathcal{G}} (\mathcal{O}) }$ are a function of
    \begin{align*}
        \Bigl\{ \underline{Z}_{\mathcal{O}'}^{\ell}: \mathcal{O}' \in \bigcup_{\mathcal{O}'' \in {\mathrm{Pa}}^{\mathcal{G}} (\mathcal{O})} \bigcup_{i=0}^{\mathrm{H}^{\mathcal{G}} (\mathcal{O}'') } (\mathrm{Pa}^{\mathcal{G}})^{(i)} (\mathcal{O}'') \Bigr\},
    \end{align*}
    which, by the definition of recursive parents, equals exactly the conditioning term in Equation \ref{eq:derivation_conditional_independence_theorem_5}. Therefore, by applying Lemma \ref{lemma:func_conditional_independent} (recall that this lemma states functions of conditionally independent random variables are still conditionally independent) to Equation \ref{eq:derivation_conditional_independence_theorem_5}, we complete the proof.
\end{proof}